\documentclass[twoside,11pt]{article}

\usepackage{blindtext}
\usepackage{amsthm}
\usepackage{amsmath}
\usepackage{algorithmic}
\usepackage{algorithm}
\usepackage{multirow}
\usepackage{subcaption}
\usepackage{xcolor} % 或 \usepackage{color}

\newcommand{\cG}{\mathcal{G}}
\newcommand{\cH}{\mathcal{H}}
\newcommand{\cI}{\mathcal{I}}
\newcommand{\cM}{\mathcal{M}}
\newcommand{\cP}{\mathcal{P}}
\newcommand{\cT}{\mathcal{T}}
\newcommand{\cS}{\mathcal{S}}
\newcommand{\cC}{\mathcal{C}}
\newcommand{\bU}{\mathbf{U}}
\newcommand{\bF}{\mathbf{F}}
\newcommand{\bX}{\mathbf{X}}
\newcommand{\bY}{\mathbf{Y}}
\newcommand{\bZ}{\mathbf{Z}}
\newcommand{\bS}{\mathbf{S}}
\newcommand{\bV}{\mathbf{V}}
\newcommand{\bD}{\mathbf{D}}
\newcommand{\bB}{\mathbf{B}}
\newcommand{\bE}{\mathbf{E}}
\newcommand{\bT}{\mathbf{T}}
\newcommand{\bN}{\mathbf{N}}
\newcommand{\bC}{\mathbf{C}}
\newcommand{\bI}{\mathbf{I}}
\newcommand{\bbT}{\mathbb{T}}
\newcommand{\rmT}{\mathrm{T}}
\newcommand{\cRB}{\mathcal{RB}}
\newcommand{\cRD}{\mathcal{RD}}
\newcommand{\cRA}{\mathcal{RA}}
\newcommand{\cperp}{\perp\!\!\!\perp}
\newcommand{\ncperp}{\not\!\perp\!\!\!\perp}

\usepackage{jmlr2e}

\usepackage{lastpage}
\jmlrheading{27}{2026}{1-\pageref{LastPage}}{x/xx; Revised x/xx}{x/xx}{21-0000}{Haijie Xu and Chen Zhang}

\ShortHeadings{DoE for Discovering DMG}{Haijie Xu and Chen Zhang}
\firstpageno{1}

\begin{document}

\title{Design of Experiment for Discovering Directed Mixed Graph}

\author{\name Haijie Xu \email xu-hj22@mails.tsinghua.edu.cn \\
       \addr Department of Industrial Engineering\\
       Tsinghua University
       \AND
       \name Chen Zhang\thanks{Corresponding author.} \email zhangchen01@tsinghua.edu.cn \\
       \addr Department of Industrial Engineering\\
        Tsinghua University}

\editor{My editor}

\maketitle

\begin{abstract}
We study the design of interventions for causal discovery in simple structural causal models whose causal graphs are directed mixed graphs (DMGs) that may contain directed cycles and bidirected edges representing latent confounding. In such case, observational conditional-independence (CI) information may not identify even the graph skeleton, while CI alone cannot generally detect a bidirected edge coexisting with a directed edge. To this end, we propose a stage-wise framework based on tailored separating systems. Separating-system interventions first recover descendant relations and strongly connected components (SCCs). The SCCs are then ordered by ancestry, and an SCC-Anc separating system recovers the directed subgraph. Given this subgraph, further systems use CI tests interpreted through \(d\)- or \(\sigma\)-separation to recover non-adjacent bidirected edges, whose endpoints share no directed edge, and do-see comparisons to recover those coexisting with exactly one directed edge. Under our assumptions, the framework recovers the directed subgraph and every bidirected edge except double-adjacent ones, whose endpoints are connected by a directed edge in each direction. We develop algorithms for unrestricted and \(M\)-bounded settings, with each experiment targeting at most \(M\) variables in the latter. For recovering the directed subgraph and non-adjacent bidirected edges, our upper bounds on the number and maximum size of experiments match corresponding worst-case lower bounds up to logarithmic factors.

\end{abstract}


\section{Introduction}
\label{sec:introduction}
Learning causal relationships among multiple variables is a fundamental problem in causal inference \citep{pearl2009causality}. Such relationships are commonly represented by graphs, in which nodes correspond to variables and edges encode causal relations. Most existing work on causal structure learning assumes a directed acyclic graph (DAG), where an edge \(X\to Y\) indicates that \(X\) is a direct cause of \(Y\) \citep{nMsystem2015,kocaoglu2017cost,squires2020active,choo2023subset}. Although DAGs provide a powerful modeling framework, their acyclicity assumption precludes feedback loops, while their directed-edge structure cannot directly represent latent confounding among observed variables. Both phenomena arise naturally in real-world systems. For example, greenhouse gas emissions contribute to global warming, which accelerates permafrost thaw; thawing permafrost, in turn, releases additional greenhouse gases, creating a causal feedback loop \citep{hodgkins2014changes}. Moreover, latent environmental factors may jointly influence global warming and permafrost thaw, inducing confounding between them. Feedback is represented by directed cycles, whereas latent confounding is represented by bidirected edges. Modeling these phenomena simultaneously therefore requires directed mixed graphs (DMGs), which contain both directed and bidirected edges and whose directed components need not be acyclic. Such structures arise in diverse domains, including the social sciences, electrical feedback systems, and gene regulatory networks, motivating the development of causal structure learning methods for DMGs \citep{mason2014feedback,goldberger1972structural,rohbeck2024bicycle,nilsson2022artificial}.

Learning a DMG, however, is substantially more challenging than learning a DAG. Cycles and latent confounding can create inducing paths between variables that are nonadjacent in the underlying graph, making them inseparable by any conditioning set of observed variables \citep{mooij2020constraint}. Consequently, even in the absence of latent confounding, the skeleton of a directed cyclic graph may not be identifiable from observational conditional-independence (CI) information. Exact recovery can therefore require interventions even when all variables are observed \citep{mokhtarian2023unified,mokhtarian2021recursive}. Cycles also complicate the graphical characterization of conditional independence. The \(d\)-separation criterion underlying many DAG-based discovery methods is no longer adequate. Instead cyclic models require the more general \(\sigma\)-separation criterion and its associated \(\sigma\)-Markov property. Latent confounding introduces a further ambiguity. When a directed edge and a bidirected edge coexist between the same pair of variables, the presence of the bidirected edge generally cannot be inferred from CI relations alone. Do-see tests are required to address this ambiguity by comparing the consequences of observing a variable with those of intervening on it \citep{NIPS2017_291d43c6}. Taken together, these limitations make interventional data generally necessary for resolving structural ambiguities created by cycles and enabling do-see comparisons for latent confounding. Consequently, this also gives rise to the experimental design problem of constructing a collection of interventions sufficient for graph identification while minimizing experimental cost. We evaluate this cost along two complementary dimensions: the total number of experiments and the maximum experiment size, defined as the largest number of variables targeted in any single intervention.

Prior work on causal experiment design can be organized into three restricted settings. For causally sufficient DAGs, which contain neither directed cycles nor latent confounding among the observed variables, interventions are generally required to distinguish DAGs within the same observational Markov equivalence class \citep{pearl2009causality}. Methods for this setting include designs based on singleton interventions \citep{he2008active}, bounded-size interventions \citep{eberhardt2012number,nMsystem2015}, and more general subset-based or adaptive interventions \citep{squires2020active,choo2023subset}. A second line of work permits latent confounding while retaining acyclicity. \citet{NIPS2017_291d43c6} propose a multistage procedure that combines CI and do-see tests and establishes bounds on the total number of experiments, but allows interventions on arbitrary subsets of variables and therefore does not control the maximum experiment size. \citet{addanki2020efficient} instead minimize the total intervention cost under variable-specific costs, rather than imposing a cardinality constraint on each experiment. A third line of work permits cycles but assumes the absence of latent confounding. \citet{mokhtarian2023unified} analyze both bounded- and unbounded-size intervention settings and provide guarantees on both the number of experiments and the maximum experiment size. In summary, existing frameworks address cycles and latent confounding only in isolation, leaving experiment design for DMGs in which both are present unresolved.

To address this gap, we develop a unified experimental design framework for DMGs containing both directed cycles and bidirected edges. The framework combines \(\sigma\)-separation-based CI tests for learning cyclic structure with do-see comparisons for detecting latent confounding. We provide algorithms for two intervention settings: an unrestricted setting and an \(M\)-bounded setting, in which each experiment targets at most \(M\) variables. Under our stated assumptions, both algorithms recover the entire directed subgraph and all bidirected edges except double-adjacent ones—configurations in which \(X\leftrightarrow Y\) coexists with both \(X\to Y\) and \(Y\to X\).

We begin by establishing worst-case lower bounds on the maximum experiment size and the total number of experiments required to recover the directed subgraph and, separately, all non-adjacent bidirected edges. We then develop a four-stage recovery framework. In the first stage, we combine observational information with interventions generated by a colored separating system to recover the descendant relations and strongly connected components (SCCs). In the second stage, we organize the SCCs according to their ancestral relations and use an SCC-Anc separating system to recover the entire directed subgraph. Given this directed structure, the third stage applies a non-adjacent separating system and CI tests to identify all non-adjacent bidirected edges. The final stage combines an adjacent separating system with do-see tests to identify all single-adjacent bidirected edges. For the bounded setting, we replace these designs with size-constrained counterparts: an \((n,M)\)-separating system in the first stage and bounded versions of the SCC-Anc, non-adjacent, and adjacent separating systems in the subsequent stages. Table~\ref{tab: main contribution} summarizes these stages and reports the resulting upper bounds on both complexity measures. For the recovery tasks covered by our lower bounds, the achievable upper bounds match the lower bounds up to logarithmic factors. 
\begin{table}[htbp]
\centering

\begin{tabular}{l|l|c|c}

\hline
 \textbf{Edge class}& \textbf{Result type} & \textbf{Max experiment size} & \textbf{Number of experiments} \\
\hline
\multirow{3}{*}{\textbf{Directed edges}}& Unbounded alg. &$n-1$ & Corollary \ref{coro: Upper unbounded D cI}\\
~ & Bounded alg. &$M$ in Remark \ref{remark: M setting} & Corollary \ref{coro: Upper bounded D cI}\\
~&Lower bound & Theorem \ref{thm:lower D bI} & Theorem \ref{thm:lower D cI}\\
\hline
\textbf{Non-adjacent}& Unbounded alg. &$n-1$ & Propositions \ref{prop: non-ad separating cc} and \ref{prop: non-ad separating random}\\
\multirow{2}{*}{\textbf{bidirected edges}} & Bounded alg. &$M$ in Remark \ref{remark: M setting} & Theorem \ref{thm: boun 2.1}\\
~&Lower bound & Theorem \ref{thm:lower B bI} & Theorem \ref{thm:lower B cI}\\
\hline

\textbf{Adjacent}& Unbounded alg. &$n-1$ & Propositions \ref{prop: step 2.2 min strong EC} and \ref{prop: step 2.2 strong EC}\\
\textbf{bidirected edges} & Bounded alg. &$M$ in Remark \ref{remark: M setting} & Theorem \ref{thm: boun 2.2}\\ 

\hline
\multirow{2}{*}{$\cRD(\cG)$}& Unbounded alg. &$n-1$ & Corollary \ref{coro: Upper unbounded DB cI}\\
~ & Bounded alg. &$M$ in Remark \ref{remark: M setting} & Corollary \ref{coro: Upper bounded DB cI}\\ 
\hline
\end{tabular}

\caption{Main contributions of this paper. It provides the lower bound and the upper bound of the bounded or unbounded algorithm of the maximum size experiment and number of experiments for discovering each part of the DMG $\cG$. Here $n$ is the number of nodes in $\cG$, $M$ is a certain upper bound of experiment size in Remark \ref{remark: M setting} and $\cRD(\cG)$ is the $\cG$ without the double adjacent bidirected edges which is defined in Section \ref{sec: problem description}. Since the CI test fails when identifying adjacent bidirected edges, the lower bound is not well-defined in this case and thus is not discussed further.}
\label{tab: main contribution}
\end{table}

To the best of our knowledge, this is the first nonparametric experimental design framework that simultaneously accommodates directed cycles and latent confounding while analyzing both the total number of experiments and the maximum experiment size. Our main contributions are as follows:

\begin{itemize}
    \item \textbf{Worst-case lower bounds.}
    We establish worst-case lower bounds on both the maximum experiment size and the total number of experiments required to identify all directed edges (Theorems~\ref{thm:lower D bI} and~\ref{thm:lower D cI}) and all non-adjacent bidirected edges (Theorems~\ref{thm:lower B bI} and~\ref{thm:lower B cI}) when cycles and latent confounders coexist. These bounds show that the coexistence of cycles and latent confounding can require strictly larger experiments and strictly more experiments than the corresponding cyclic setting without latent confounding.

\item \textbf{Unrestricted intervention design.} We develop a multistage algorithm for the setting in which intervention size is unrestricted. The algorithm uses a colored separating system to recover descendant relations and SCCs, followed by an SCC-Anc separating system to recover the directed subgraph. It then combines non-adjacent and adjacent separating systems with CI and do-see tests, respectively, to recover the bidirected edges. Overall, the algorithm recovers \(\cRD(\cG)\), consisting of the entire directed subgraph and all bidirected edges except double-adjacent ones, and admits explicit upper bounds on the total number of experiments (Corollaries~\ref{coro: Upper unbounded D cI} and~\ref{coro: Upper unbounded DB cI}).

\item \textbf{\(M\)-bounded intervention design.}
We extend the framework to the setting where each experiment intervenes on at most $M$ variables. By constructing an $(n,M)$-separating system and bounded-size versions of the SCC-Anc, non-adjacent, and adjacent separating systems (Theorems~\ref{thm: boun 1.2}, \ref{thm: boun 2.1}, and~\ref{thm: boun 2.2}), the resulting algorithm recovers the same identifiable structure while respecting the prescribed size constraint. We further derive explicit upper bounds on the required number of experiments (Corollaries~\ref{coro: Upper bounded D cI} and~\ref{coro: Upper bounded DB cI}), which match the corresponding lower bounds up to logarithmic factors.
\end{itemize}

 The remainder of the paper is organized as follows. Section \ref{sec:literature} reviews related work. In Section \ref{sec: formulation}, we introduce the necessary preliminaries, including notation, basic assumptions, and a formal statement of the problem. Section \ref{sec: lower bound} presents worst-case lower bounds on the maximum experiment size and the total number of experiments required to identify directed edges and non-adjacent bidirected edges. In Section \ref{sec: unb algorithm}, we propose our algorithm for learning the structure of DMGs, and in Section \ref{sec: boun algorithm}, we extend this algorithm to the bounded intervention setting. Finally, Section \ref{sec:conclusion} concludes the paper and outlines directions for future research.

\section{Related Work}
\label{sec:literature}

\subsection{Causal structure learning without experimental design}
Causal structure learning from observational data concerns inferring the existence and orientation of edges in a causal graph without manipulating the underlying system. Under standard Markov and faithfulness assumptions, many observational methods identify the causal graph only up to its Markov equivalence class rather than recovering its exact structure \citep{spirtes2000causation,chickering2002optimal,pearl2009causality}.

\textbf{DAGs and acyclic causal graphs with latent confounders.}
Under causal sufficiency—that is, in the absence of latent confounding among the observed variables—two major algorithmic paradigms for DAG structure learning are constraint-based and score-based methods. Constraint-based methods, such as the PC algorithm, infer adjacencies and edge orientations from conditional-independence relations \citep{spirtes2000causation}, whereas score-based methods, such as Greedy Equivalence Search (GES) and its generalized variants, search over graphs or equivalence classes using a statistical score \citep{chickering2002optimal,huang2018generalized}. Under the standard Markov and faithfulness assumptions, both paradigms generally identify the underlying DAG only up to its observational Markov equivalence class, commonly represented by a completed partially directed acyclic graph (CPDAG), rather than recovering a unique DAG.

A complementary line of work achieves stronger identifiability by imposing additional restrictions on the data-generating mechanism through specific classes of functional causal models. Examples include LiNGAM, which assumes acyclic linear structural equations with mutually independent non-Gaussian disturbances \citep{shimizu2006linear}, and nonlinear additive-noise or post-nonlinear models \citep{hoyer2008nonlinear,zhang2012identifiability}. Under suitable model-specific assumptions, these restrictions can identify causal orientations that remain unresolved within the observational Markov equivalence class. However, these stronger conclusions depend on the validity of the imposed functional and distributional assumptions and do not extend directly to general nonparametric causal models.

When causal sufficiency is relaxed, methods such as Fast Causal Inference (FCI) \citep{spirtes2013causal} and Greedy FCI \citep{ogarrio2016hybrid} accommodate latent confounding among the observed variables. Rather than explicitly recovering the latent variables, these methods return a partial ancestral graph, which represents a Markov equivalence class of maximal ancestral graphs and encodes the ancestral relations and latent-confounding features identifiable from the observed distribution. Thus, under standard nonparametric assumptions, observational methods for acyclic models generally recover only an equivalence class—a CPDAG under causal sufficiency and a PAG in the presence of latent confounding. Instead, suitably chosen interventions can provide additional information for distinguishing structures that remain observationally equivalent.

\textbf{Cyclic causal models.}
Observational causal discovery has also been extended beyond the acyclic setting, although the resulting guarantees depend strongly on the assumed model class and notion of equivalence. Within the constraint-based paradigm, \citet{mooij2020constraint} show that FCI, originally developed for acyclic causal models, can be applied without modification to observational data generated by simple, \(\sigma\)-faithful cyclic SCMs. The resulting directed partial ancestral graph represents the corresponding \(\sigma\)-Markov equivalence class, even when latent confounding is present. For causally sufficient directed graphs, \citet{claassen2023establishing} provide an efficient characterization of \(d\)-Markov equivalence in the presence of cycles. Using conditional-independence information, \citet{semnani2025causal} propose a hybrid procedure that greedily searches over partially ordered partitions of the variables and subsequently constructs a representative graph from the resulting \(d\)-Markov equivalence class, with the partition encoding the graph's strongly connected components and their partial order.

Under linear Gaussian structural equation models, \citet{ghassami2020characterizing} introduce the notions of distribution equivalence and quasi-equivalence for cyclic and acyclic directed graphs and develop a score-based local-search method for learning cyclic structures up to quasi-equivalence. As a complementary functional-model approach, \citet{lacerda2012discovering} extend the ICA-based LiNGAM framework to causally sufficient linear cyclic models with independent non-Gaussian disturbances. Their method generally recovers a class of distribution-equivalent structural equation models, with unique identification requiring additional assumptions.

Taken together, these observational approaches either identify the causal structure only up to an appropriate equivalence class or achieve stronger identifiability under model-specific functional and distributional assumptions. In particular, they do not provide general nonparametric guarantees for uniquely recovering cyclic causal structure in the simultaneous presence of latent confounding. Resolving these observational ambiguities generally requires interventional data. This motivates causal experimental design, in which intervention targets are deliberately selected to recover otherwise unidentifiable graph components while controlling the number and size of the required experiments.

\subsection{Causal structure learning with experimental design}
Causal experimental design actively selects intervention targets—and, in some formulations, intervention values or sample allocations—to acquire structural information that may be unavailable from observational data alone. Existing work has pursued two prominent, partly overlapping objectives. The first seeks designs that guarantee identification of the full causal graph or specified graph features while minimizing experimental resources. The second assumes a fixed resource budget and selects experiments to maximize a measure of structural information, such as the number of oriented edges or expected information gain, without necessarily guaranteeing complete identification. Our work follows the first perspective and evaluates experimental complexity in terms of the total number of experiments and the maximum experiment size.

\textbf{Guaranteed identification and intervention complexity.}
Early work in this direction focused primarily on causally sufficient DAGs. \citet{eberhardt2012number} derived a tight worst-case bound on the number of experiments under unrestricted simultaneous interventions and provided a constructive upper bound when the number of variables targeted in each experiment is bounded. \citet{he2008active} studied single-variable interventions, proposing a non-adaptive batch design that finds a minimum sufficient set of intervention targets, together with adaptive sequential designs based on minimax and maximum-entropy criteria. Bridging guaranteed identification and uncertainty reduction, \citet{hauser2014two} developed a sequential minimax strategy based on singleton interventions and a minimum-cardinality design with potentially multivariate intervention targets that guarantees full identifiability. For interventions of size at most (M), \citet{nMsystem2015} established lower bounds and developed constructions based on ((n,M))-separating systems, together with an adaptive algorithm for chordal components of the observational essential graph.

Subsequent work developed more instance-dependent and adaptive guarantees. \citet{greenewald2019sample} proposed an adaptive Bayesian strategy for causal forests and provided finite-sample extensions under additional assumptions. \citet{squires2020active} considered general causally sufficient DAGs under noiseless singleton interventions, deriving a clique-based lower bound and a two-stage policy with a logarithmic approximation guarantee under an additional structural condition. More recently, \citet{choo2022verification} distinguished verification, which tests whether a given DAG within the observational MEC is the true graph, from search, which seeks to identify the unknown true DAG. \citet{choo2023subset} extended these questions to a prescribed subset of target edges, providing exact subset-verification results, bounded-size and cost extensions, and lower bounds for general subset search. \citet{choo2023adaptivity} subsequently formalized (r)-adaptive design, in which interventions are selected in at most (r) sequential batches, and characterized the trade-off between adaptivity and the total number of experiments, including extensions to bounded-size interventions.

Comparatively few works provide identification guarantees beyond causally sufficient DAGs. For acyclic semi-Markovian causal models with latent confounders, \citet{kocaoglu2017experimental} developed a stage-wise procedure that first recovers the directed graph among the observed variables and then detects latent confounding between both adjacent and non-adjacent pairs using CI and do-see tests. Their guarantees concern the total number of experiments, while allowing intervention targets of unrestricted size. \citet{addanki2020efficient} studied alternative intervention-cost objectives in the presence of latent confounding. Under a linear cost model, they approximately minimize the total variable-specific intervention cost required to recover ancestral relations among the observed variables; under an identity-cost model, they bound the number of interventions required to recover the full observed-variable causal structure, including latent-confounding relations. Neither formulation imposes a cardinality constraint on each intervention. For causally sufficient simple SCMs whose directed graphs may contain cycles, \citet{mokhtarian2023unified} developed a two-stage procedure that first recovers descendant relations and strongly connected components and then identifies all directed edges. They analyzed both unrestricted and (M)-bounded interventions and established worst-case guarantees on both the number of experiments and the maximum experiment size.

\textbf{Information gain under a fixed budget.}
A second line of work asks how much structural information can be obtained under a prescribed experimental budget, rather than requiring complete identification. For causally sufficient DAGs, \citet{ghassami2018budgeted} considered the non-adaptive selection of at most \(k\) singleton interventions to maximize the expected number of oriented edges, where the expectation is taken uniformly over the DAGs in the observational MEC. They showed that this objective is monotone and submodular, yielding a greedy approximation algorithm. Because evaluating such objectives requires reasoning over the DAGs in an MEC, subsequent work developed more efficient computational primitives. \citet{ghassami2019counting} introduced clique-tree-based procedures for counting and uniformly sampling DAGs in an MEC and applied them to non-adaptive intervention design, with polynomial-time counting guarantees for bounded-degree graphs. \citet{ahmaditeshnizi2020lazyiter} developed LazyIter to efficiently enumerate the possible interventional essential graphs resulting from singleton interventions and used it to support both sequential and exact non-adaptive design objectives. \citet{wienobst2023polynomial} subsequently provided polynomial-time algorithms for counting and uniformly sampling DAGs from an MEC and demonstrated their use in active intervention design. Complementing these MEC-based combinatorial formulations, \citet{agrawal2019abcd} considered a Bayesian finite-sample setting in which batches of interventions are selected to maximize expected information gain about a user-specified feature of the causal graph, subject to constraints on the total sample size and the number of experimental rounds.

In summary, the fixed-budget literature reviewed above focuses primarily on causally sufficient DAGs. More broadly, experimental-design frameworks that move beyond this setting have accommodated latent confounding while retaining acyclicity, or feedback cycles while maintaining causal sufficiency. To the best of our knowledge, neither line of work provides a unified experimental-design framework for DMGs in which latent confounding and directed cycles coexist.
\section{Preliminaries and Problem Description}
\label{sec: formulation}

In Section \ref{sec: notation}, we introduce the notation used throughout the paper. Section \ref{sec: SCM} describes the generative model underlying our analysis. In Section \ref{sec:formulation separation}, we present two types of separation rules relevant to our setting. The concept of intervention is formally defined in Section \ref{sec: intervention}. Finally, Section \ref{sec: problem description} provides a formal statement of the problem studied in this work.

\subsection{Preliminary Graph Definitions}
\label{sec: notation}

A \textit{directed mixed graph} (DMG) is a graph $\cG=(\bV,\bD,\bB)$, where $\bV$ is a set of variables, $\bD$ is a set of directed edges over $\bV$, and $\bB$ is a set of bidirected edges over $\bV$. A directed edge from $X$ to $Y$ is denoted by $(X,Y)$, while a bidirected edge between $X$ and $Y$ is denoted by $[X,Y]$ and is interpreted as representing an unobserved confounder shared by $X$ and $Y$. We partition $\bB$ into two disjoint subsets: the non-adjacent bidirected edge set $\bB^N=\{[X,Y]\in\bB\mid (X,Y)\notin\bD,\ (Y,X)\notin\bD\}$ and the adjacent bidirected edge set $\bB^A=\bB\setminus\bB^N$. Throughout this paper, we assume that the DMG contains no self-loops; that is, $(X,X)\notin\bD$ and $[X,X]\notin\bB$ for every $X\in\bV$.

Let $\cG=(\bV,\bD,\bB)$ be a DMG. A sequence $(X_1,E_1,X_2,E_2,\dots,E_{k-1},X_k)$ is called a \textit{path} in $\cG$ if $X_i\in\bV$ for every $1\leq i\leq k$ and $E_i$ is one of $(X_i,X_{i+1})$, $(X_{i+1},X_i)$, or $[X_i,X_{i+1}]$ for every $1\leq i\leq k-1$. If $E_i=(X_i,X_{i+1})$ for every $1\leq i\leq k-1$, the path is called a \textit{directed path}. A non-endpoint variable on a path is called a \textit{collider} if both edges incident to it on the path have arrowheads pointing toward it. For two variables $X,Y\in\bV$, we write $X\rightsquigarrow_{\cG}Y$ if there exists a directed path from $X$ to $Y$ in $\cG$.

\begin{definition}[Graphical relations]
\label{def:graphical-relations}
Let $\cG=(\bV,\bD,\bB)$ be a DMG and let $X\in\bV$. The set of \textit{parents} of $X$ is defined as $Pa_{\cG}(X):=\{Y\in\bV\mid (Y,X)\in\bD\}$. The set of \textit{children} of $X$ is defined as $Ch_{\cG}(X):=\{Y\in\bV\mid (X,Y)\in\bD\}$. The set of \textit{neighbors} of $X$ is defined as $Ne_{\cG}(X):=Pa_{\cG}(X)\cup Ch_{\cG}(X)$. The set of \textit{siblings} of $X$ is defined as $Sib_{\cG}(X):=\{Y\in\bV\mid [X,Y]\in\bB\}$. The set of \textit{ancestors} of $X$ is defined as $Anc_{\cG}(X):=\{Y\in\bV\mid Y=X\text{ or }Y\rightsquigarrow_{\cG}X\}$. The set of \textit{descendants} of $X$ is defined as $De_{\cG}(X):=\{Y\in\bV\mid Y=X\text{ or }X\rightsquigarrow_{\cG}Y\}$.
\end{definition}

Two variables $X,Y\in\bV$ are said to be \textit{strongly connected} if $X=Y$ or if both $X\rightsquigarrow_{\cG}Y$ and $Y\rightsquigarrow_{\cG}X$ hold. Equivalently, $X$ and $Y$ are strongly connected if each is both an ancestor and a descendant of the other.

\begin{definition}[Strongly connected component]
\label{def:SCC}
The \textit{strongly connected component} (SCC) containing $X\in\bV$ is defined as $SCC_{\cG}(X):=Anc_{\cG}(X)\cap De_{\cG}(X)$. The SCCs form a partition of $\bV$, and two variables are strongly connected if and only if they belong to the same SCC.
\end{definition}

For notational convenience, the above definitions extend to sets of variables by taking unions over their elements; for example, $Ch_{\cG}(\bX):=\bigcup_{X\in\bX}Ch_{\cG}(X)$ for $\bX\subseteq\bV$. They similarly extend to edges by taking unions over their endpoints; for example, $Pa_{\cG}([X,Y]):=Pa_{\cG}(X)\cup Pa_{\cG}(Y)$.

An \textit{undirected graph} is a graph only containing undirected edges. We denote an undirected edge between two distinct variables $X$ and $Y$ by $[X,Y]^u$. The \textit{skeleton} of a DMG $\cG$ is the undirected graph $(\bV,\bE)$ in which $[X,Y]^u\in\bE$ if and only if $X$ and $Y$ are neighbors or siblings in $\cG$. A \textit{directed graph} (DG) is a special case of a DMG containing no bidirected edges, i.e., $\cG=(\bV,\bD)$. A \textit{directed acyclic graph} (DAG) is a DG that contains no directed cycles.

\subsection{Generative Model}
\label{sec: SCM}
We use \textit{Structural causal models} (SCMs, \citet{pearl2009causality}) to describe the causal mechanisms of
a system.
\begin{definition} [SCM]
     An SCM is a tuple $\cM = (\bV,\bU,\bF,P(\bU))$, where $\bV$ is a set of endogenous variables, $\bU$ is a set of exogenous variables with the joint distribution $P(\bU)$ where the variables in $\bU$ are assumed to be jointly independent, and $\bF$ is a set of functions $\{f_X\}_{X\in\bV}$ such that $X = f_{X}(Pa(X),\bU^X)$, where $Pa(X)\subseteq \bV\setminus X$ and $\bU^{X} \subseteq \bU$.
\end{definition}
Let $\cM = (\bV,\bU,\bF,P(\bU))$ be an SCM. Note that we do not make the causal sufficiency assumption, i.e., $\bU^X\cap \bU^{Y}$ are not assumed to be $\varnothing$ for any $X,Y\in\bV$. The causal graph of $\cM$ is a DMG over
$\bV$ with directed edges from $Pa(X)$ to $X$ for each variable $X \in \bV$ and bidirected edges between $X$ and $Y$ for each pair $X,Y $ in $\bV$ with $\bU^X\cap\bU^Y\neq \varnothing$. Therefore, the endogenous parent variables appearing in the structural equations $\bF$ determine the directed edge set $\bD$, while overlaps among the exogenous-variable sets $\{\bU^X\}_{X\in\bV}$ determine the bidirected edge set $\bB$. Hence, the SCM $\cM$ uniquely induces a causal graph $\cG=(\bV,\bD,\bB)$. 

 An SCM is called acyclic if the corresponding causal graph does not contain a cycle. According to \citet{bongers2021foundations}, an acyclic SCM always induces unique observational, interventional, and counterfactual distributions, while this does not hold for an SCM with cycles. To solve this issue, \citet{bongers2021foundations} introduces simple SCMs, a subclass of SCMs (cyclic or acyclic). 

 \begin{definition}[Simple SCM]
 \label{def: simple SCM}
      An SCM is simple if any subset of its structural equations can be solved uniquely for its associated variables in terms of the other variables that appear in these equations.
 \end{definition}

 \begin{proposition}[\citet{bongers2021foundations}]
      Simple SCMs always have uniquely defined observational, interventional, and counterfactual distributions.
 \end{proposition}

 In the remainder of this paper, we only consider simple SCMs. The same assumption is also used in \citet{mokhtarian2023unified}.

\subsection{$d$-separation and $\sigma$-separation}
\label{sec:formulation separation}

We use $(\bX \cperp\bY|\bZ)_P$ to denote that $\bX$ and $\bY$ are conditional independent (CI) given $\bZ$, where $\bX,\bY,\bZ$ are  three disjoint subsets of variables with the joint distribution $P$. Then we define $d$-separation and $\sigma$-separation for DMGs.

\begin{definition}[$d$-separation]
    Suppose $\cG = (\bV,\bD,\bB)$ is a DMG, $X$ and $Y$ are two distinct variables in $\bV$, and $\bS\subseteq \bV \setminus\{X,Y\}$. $\cP = (X = Z_0, E_0,Z_1,E_1,\dots,Z_k,E_k,Z_{k+1} = Y)$ is a path between $X$ and $Y$ in $\cG$, where $Z_i\in\bV$ for $0\leq i\leq k+1$ and $E_i\in \bD\cup\bB$ for $0\leq i\leq k$. We say $\cP$ is $d$-blocked by $\bS$ if there exists $1\leq i \leq k$ such that
    \begin{itemize}
        \item $Z_i$ is a collider on $\cP$ and $Z_i \notin Anc_{\cG}(\bS\cup\{X,Y\})$, or,
        \item $Z_i$ is not a collider on $\cP$ and $Z_i \in \bS$.
    \end{itemize}
    We say $\bS$ $d$-separation $X$ and $Y$ in $\cG$ if  all the paths in $\cG$ between $X$ and $Y$ are $d$-blocked by $\bS$ and denote it as $(X\cperp_{d} Y|\bS)_{\cG}$. Furthermore, for three disjoint subsets, $\bX,\bY,\bS$ in $\bV$, we say $\bS$ $d$-separation $\bX$ and $\bY$ in $\cG$ if for any $X\in\bX$ and $Y\in\bY$, $(X\cperp_{d} Y|\bS)_{\cG}$, and we denote it as $(\bX\cperp_{d} \bY|\bS)_{\cG}$.
\end{definition}

\begin{definition}[$\sigma$-separation]
    Suppose $\cG = (\bV,\bD,\bB)$ is a DMG, $X$ and $Y$ are two distinct variables in $\bV$, and $\bS\subseteq \bV \setminus\{X,Y\}$. $\cP = (X = Z_0, E_0,Z_1,E_1,\dots,Z_k,E_k,Z_{k+1} = Y)$ is a path between $X$ and $Y$ in $\cG$, where $Z_i\in\bV$ for $0\leq i\leq k+1$ and $E_i\in \bD\cup\bB$ for $0\leq i\leq k$. We say $\cP$ is $\sigma$-blocked by $\bS$ if there exists $1\leq i \leq k$ such that
    \begin{itemize}
        \item $Z_i$ is a collider on $\cP$ and $Z_i \notin Anc_{\cG}(\bS\cup\{X,Y\})$, or,
        \item $Z_i$ is not a collider on $\cP$, $Z_i \in \bS$ and either $E_i = (Z_i,Z_{i+1})$ and $Z_{i+1}\notin SCC_{\cG}(Z_i)$, or $E_{i-1} = (Z_i, Z_{i-1})$ and $Z_{i-1} \notin SCC_{\cG}(Z_i)$.
    \end{itemize}
    We say $\bS$ $\sigma$-separation $X$ and $Y$ in $\cG$ if  all the paths in $\cG$ between $X$ and $Y$ are $\sigma$-blocked by $\bS$ and denote it as $(X\cperp_{\sigma} Y|\bS)_{\cG}$. Furthermore, for three disjoint subsets, $\bX,\bY,\bS$ in $\bV$, we say $\bS$ $\sigma$-separation $\bX$ and $\bY$ in $\cG$ if for any $X\in\bX$ and $Y\in\bY$, $(X\cperp_{\sigma} Y|\bS)_{\cG}$, and we denote it as $(\bX\cperp_{\sigma} \bY|\bS)_{\cG}$.
\end{definition}

\begin{figure}[ht]
    \centering
    % 第一行的单个子图
    
    % 第二行的五个并列子图
        \includegraphics[width=0.65\textwidth]{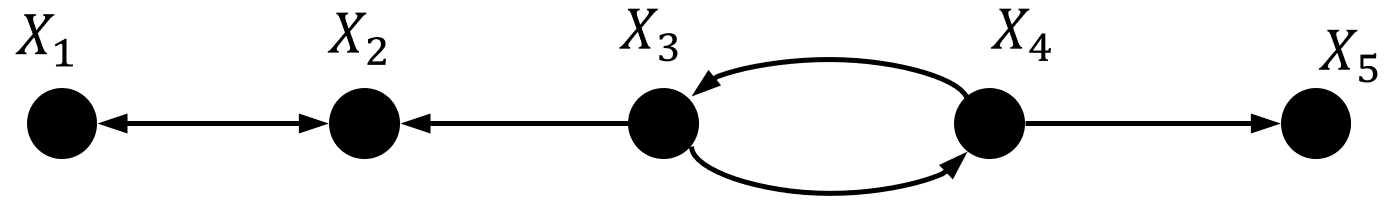}
        \caption{DMG of Example \ref{exam:example seperation}.}
        \label{fig:example seperation}
\end{figure}

\begin{example}[$d$-block and $\sigma$-block]
\label{exam:example seperation}
    In Figure \ref{fig:example seperation}, the path between $X_1$ and $X_5$ is $d$-blocked and $\sigma$-blocked by $\varnothing$ since $X_2$ is a collider and not in the condition set. The path between $X_1$ and $X_5$ is only $d$-blocked but not $\sigma$-blocked by $\{X_2,X_3\}$ since $X_3$ and $X_4$ are in a same SCC. 
\end{example}
It is easy to see that when there are no cycles in the DMG, $d$-separation and $\sigma$-separation are equivalent, since each SCC contains only one node. However, the two are not equivalent when there are cycles in the DMG. For convenience, we introduce letter $r$ to stand for either $d$ (as in $d$-separation) or $\sigma$
(as in $\sigma$-separation). Then, we formally define $r$-independence model, $r$-Markov equivalence
class, $r$-Markov property, and $r$-faithfulness. 

\begin{definition}[$IM_r(\cG)$]
For a DMG $\cG$, the $r$-independence model $IM_r(\cG)$ is defined as the set of $r$-separations of $\cG$. That is,
\begin{equation*}
    IM_r(\cG) = \{(X,Y,\bZ)|X,Y\in\bV, \bZ\subseteq \bV\setminus\{X,Y\}, (X\cperp_r Y|\bZ )_{\cG}\}
\end{equation*}
We can find that for a DMG without cycles, $IM_d(\cG) = IM_{\sigma}(\cG)$. In this case, for convenience, we omit the subscript $r$ and write it as $IM(\cG)$.
\end{definition}

\begin{definition}[$r$-MEC]
    Two DMGs with identical $r$-independence models are called $r$-Markov equivalent. We denote by $[\cG]^r$ the $r$-Markov equivalence class ($r$-MEC) of $\cG$, i.e., the set of $r$-Markov equivalent DMGs of $\cG$.
\end{definition}

\begin{definition}[$r$-Markov property, $r$-faithfulness]
    A distribution $P$ satisfies $r$-Markov
property with respect to a DMG $\cG$ if for any $r$-separation $(X\cperp_r Y|\bZ )_{\cG}$ in $\cG$, the CI $(X\cperp Y|\bZ )_{P}$ holds in $P$. Similarly, a distribution $P$ satisfies $r$-faithfulness with respect to
a DMG $\cG$ if for any CI $(X\cperp Y|\bZ )_{P}$ in $P$, the $r$-separation $(X\cperp_r Y|\bZ )_{\cG}$  holds in $\cG$.
\end{definition}

Suppose $\cM = (\bV,\bU,\bF,P(\bU))$ is a simple SCM with observational distribution $P^{\cM}(\bV)$ and causal graph $\cG$. For convenience, we usually drop the superscript $\cM$ when it is clear from the context.
It has been shown that $P$ always satisfies $\sigma$-Markov property with respect to $\cG$. However, the $d$-Markov property holds in specific settings, e.g., acyclic SCMs, SCMs with continuous variables and linear relations, or SCMs with discrete variables \citep{bongers2021foundations}.

\subsection{Intervention and Experiment}
\label{sec: intervention}
Suppose $\cM = (\bV,\bU,\bF,P(\bU))$ is an SCM. We use $do(\bI)$ to denote a full-support hard intervention on a subset $\bI\subseteq\bV$. It converts $\cM$ to a new SCM $\cM_{do(\bI)} = (\bV,\bU,\bF^{\prime}, P(\bU))$, where for each $X\in\bI$, the structural assignment of $X$ in $\bF$ is replaced by $X = \xi_X$ in $\bF^{\prime}$, and $\xi_X$ is a random variable whose support is the same as the support of $X$ and is independent of all other random variables in the system. We use $P_{do(\bI)}$ to denote the distribution of $\cM_{do(\bI)}$.

\begin{proposition}(\citet{bongers2021foundations})
    If $\cM = (\bV,\bU,\bF,P(\bU))$ is a simple SCM, then for any $\bI\subseteq\bV$, SCM $\cM_{do(\bI)}$ is also a simple SCM.
\end{proposition}

After intervening on $\bI$, because the variables in $\bI$ are no longer functions of other variables in $\bV$, the corresponding causal graph of $\cM_{do(\bI)}$ can be obtained from graph $\cG$ by removing the incoming edges (directed edges or bidirected edges) of the variables in $\bI$. We denote the corresponding graph by $\cG_{\Bar{\bI}}$. An \textit{experiment} on a target set $\bI$ is the act of conducting a full-support hard intervention on $\bI$ and obtaining the interventional distribution $P_{do(\bI)}$.

\begin{definition}[$\cI$-$r$-MEC]
Suppose $\cI$ is a collection of subsets of $\bV$ (can include the
empty set). Two DMGs $\cG$ and $\cH$ are $\cI$-$r$-Markov equivalent if $IM_r(\cG_{\Bar{\bI}}) = IM_r(\cH_{\Bar{\bI}})$ for each $\bI\in\cI$. We denote by $[\cG]^r_{\cI}$ the $\cI$-$r$-Markov equivalent class of $\cG$, i.e., the set of $\cI$-$r$-Markov equivalent DMGs of $\cG$.
    
\end{definition}

From the definition, we can see that through experiments on each element in $\cI$, we cannot distinguish two DMGs belonging to the same $\cI$-$r$-MEC using $r$-separation.

\subsection{Problem Description}
\label{sec: problem description}
Consider a simple SCM  $\cM = (\bV,\bU,\bF,P(\bU))$ with observational  distribution $P^{\cM}(V)$ and
causal graph $\cG = (\bV,\bD,\bB^N\cup\bB^A)$. Since we do not assume causal sufficiency and acyclity, $\cG$ is a DMG. In this paper, we consider the
following two scenarios which are also used in \cite{mokhtarian2023unified}.

\begin{itemize}
    \item \textbf{Scenarios 1}: $P^{\cM}$ satisfies $d$-Markov property and $d$-faithfulness w.r.t. $\cG$. In this case, CI relations are equivalent to $d$-separations. That is,$(X\cperp Y|\bZ )_{P}\Longleftrightarrow (X\cperp_d Y|\bZ )_{\cG}$.
    \item \textbf{Scenarios 2}: $P^{\cM}$ satisfies $\sigma$-Markov property and $\sigma$-faithfulness w.r.t. $\cG$. In this case, CI relations are equivalent to $\sigma$-separations. That is,$(X\cperp Y|\bZ )_{P}\Longleftrightarrow (X\cperp_\sigma Y|\bZ )_{\cG}$.
\end{itemize}

Our goal in this paper is to design a set of experiments for learning $\cG$ under Scenario 1 or Scenario 2. According to \citet{NIPS2017_291d43c6}, for any collection of experiments $\cI$, the adjacent bidirected edges $\bB^A$, whose endpoints are also connected by at least one directed edge, cannot be identified solely from $IM_r(\cG_{\Bar{\bI}})$ for $\bI\in\cI$. In other words, there always exists an SCM $\cM$ with corresponding DMG $\cG = (\bV,\bD,\bB^N\cup\bB^A)$ where $\bB^A \neq \varnothing$, for any collection of experiments $\cI$, $|[\cG]_\cI^r|>1$.
To identify $\bB^A$, we can not only use $r$-separation relations, but also use \textit{do-see} test proposed by \citet{NIPS2017_291d43c6} (with details shown in Section \ref{sec:unb 2.2}). To describe our problem more clearly, we first define the following operator and corresponding examples are shown in Figure \ref{fig: operators}.

\begin{figure}[ht]
    \centering
    % 第一行的单个子图
    
    % 第二行的五个并列子图
    
    \begin{subfigure}[b]{0.2\textwidth}
        \includegraphics[width=\textwidth]{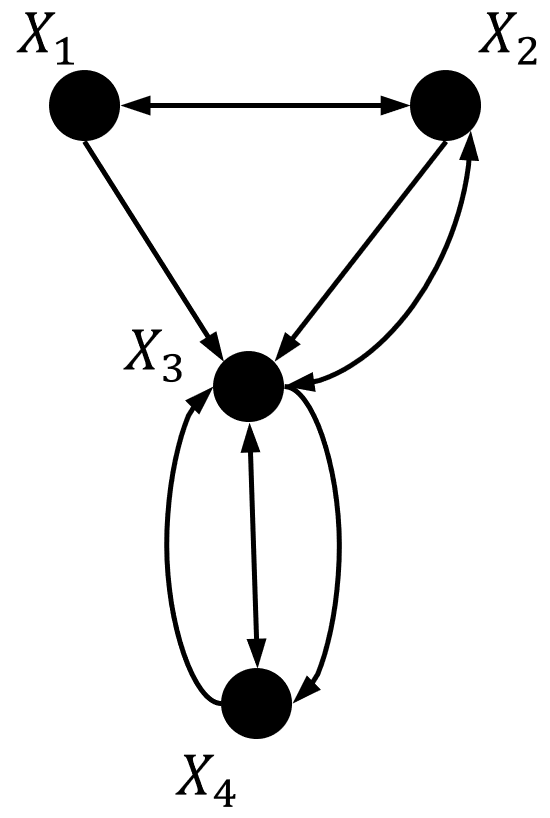}
        \caption{$\cG$}
        \label{fig:sub_2}
    \end{subfigure}
    \hfill
    \begin{subfigure}[b]{0.2\textwidth}
        \includegraphics[width=\textwidth]{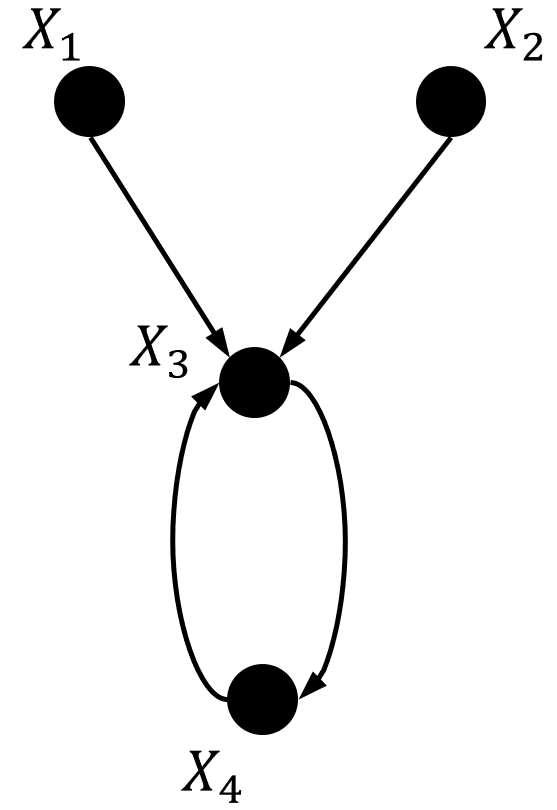}
        \caption{$\cRB(\cG)$}
        \label{fig:sub_3}
    \end{subfigure}
    \hfill
    \begin{subfigure}[b]{0.2\textwidth}
        \includegraphics[width=\textwidth]{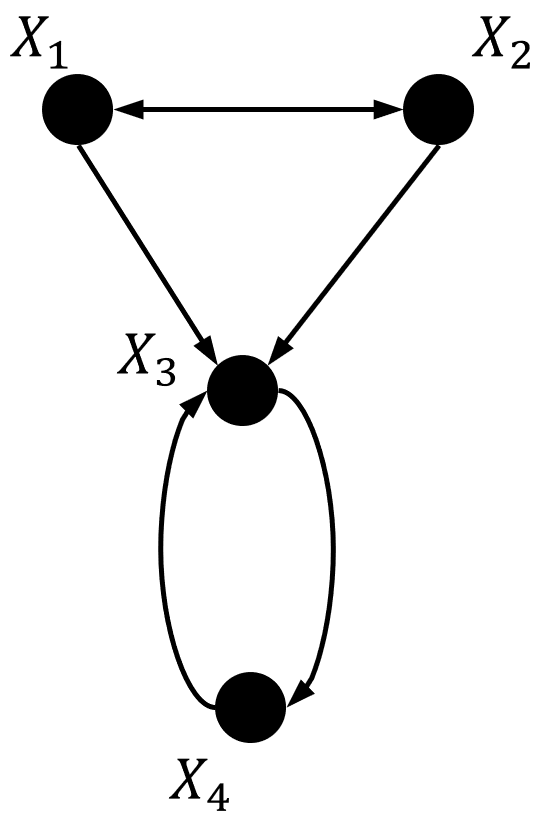}
        \caption{$\cRA(\cG)$}
        \label{fig:sub_4}
    \end{subfigure}
    \hfill
    \begin{subfigure}[b]{0.2\textwidth}
        \includegraphics[width=\textwidth]{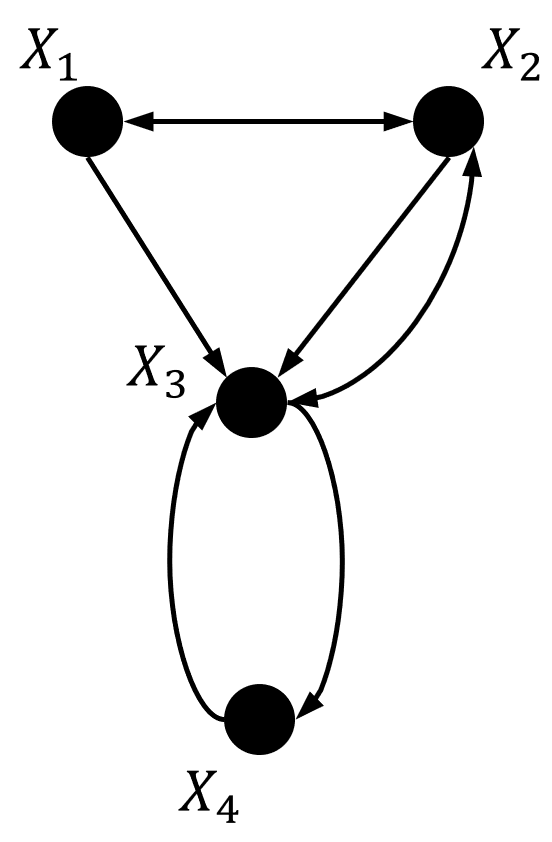}
        \caption{$\cRD(\cG)$}
        \label{fig:sub_5}
    \end{subfigure}
    
    \caption{
Examples of DMG $\cG$ and the corresponding $\cRB(\cG)$, $\cRA(\cG)$, $\cRD(\cG)$.
}
    \label{fig: operators}
\end{figure}

\begin{definition}[Removing bidirected edge operator]
     A removing bidirected edge operator $\cRB$ projects the DMG into the DG as follows: for $\cG = (\bV, \bD, \bB)$, $\cRB(\cG) = (\bV,\bD)$. We can also apply the operator to the collection of DMG as $\cRB([\cG]^r) = \{\cRB(\cG) |\cG\in[\cG]^r\}$
\end{definition}

\begin{definition}[Removing adjacent-bidirected edge operator]
    A removing adjacent-bidirected edge operator $\cRA$ projects the DMG into the DMG as follows: for $\cG = (\bV, \bD, \bB^N\cup\bB^A)$, $\cRA(\cG) = (\bV,\bD,\bB^N)$. We can also apply the operator to the collection of DMG as $\cRA([\cG]^r) = \{\cRA(\cG) |\cG\in[\cG]^r\}$
\end{definition}

 For a DMG $\cG = (\bV,\bD,\bB^N\cup\bB^A)$, we further partition the set $\bB^A$ into two disjoint subsets: the double adjacent bidirected edge set $\bB^{AD}\{[X,Y]\in\bB^A|(X,Y)\in\bD,(Y,X)\in\bD\}$ and the single adjacent bidirected edge set $\bB^{AS} = \bB^A\backslash\bB^{AD}$. 
 \begin{definition}[Removing double adjacent-bidirected edge operator]
    A removing double adjacent-bidirected edge operator $\cRD$ projects the DMG into the DMG as follows: for $\cG = (\bV, \bD, \bB^N\cup\bB^{AS}\cup\bB^{AD})$, $\cRD(\cG) = (\bV,\bD,\bB^N\cup\bB^{AS})$. We can also apply the operator to the collection of DMG as $\cRD([\cG]^r) = \{\cRD(\cG) |\cG\in[\cG]^r\}$
\end{definition}

Due to the difficulty mentioned in \citet{NIPS2017_291d43c6}, our goal is to design a collection of subsets $\cI$ such that $\cRA([\cG]^r_{\cI}) = \{\cRA(\cG)\}$ and use do-see test to identify $\bB^A$ of $\cG$ as much as possible. In fact, the method we propose can identify almost all the bidirected edges $[X,Y] \in \bB^A$, except one case, which is $\bB^{AD}$.
Additionally, as performing experiments can be costly, we aim to minimize the number of
necessary experiments.

\section{Lower Bound of $\max_{\bI\in\cI}|\bI|$ and $|\cI|$}
\label{sec: lower bound}
When confounders are not considered, as discussed in \citet{mokhtarian2023unified}, it is generally impossible to accurately recover the skeleton of a directed graph (DG) $\cG = (\bV, \bD)$ using only observational data. Moreover, even performing all size-one interventions is insufficient for full identifiability. To identify all directed edges $\bD$, they show that the intervention set $\cI$ must satisfy the following conditions: for all $\bI \in \cI$, the intervention size must be at least $\zeta_{\max} - 1$, and the total number of experiments must be at least $\zeta_{\max}$, where $\zeta_{\max}$ denotes the size of the largest SCC in $\cG$.

When confounders are present, i.e., when bidirected edges are introduced and the causal graph becomes a DMG $\cG = (\bV, \bD, \bB)$, the task of identification via $r$-separation relations becomes even more challenging. In Section \ref{sec:lower D}, we demonstrate that even if the goal is limited to recovering only the directed part of $\cG$ (denoted $\mathcal{RB}(\cG)$), the lower bounds on both $\max_{\bI \in \cI} |\bI|$ and $|\cI|$ exceed those in the DG setting.
In Section \ref{sec:lower NB}, we further establish the lower bounds on $\max_{\bI \in \cI} |\bI|$ and $|\cI|$ required to identify the non-adjacent bidirected edges $\bB^N$. As for adjacent bidirected edges $\bB^A$, since they cannot be identified through $r$-separation relations alone and instead require the use of the do-see test \citep{NIPS2017_291d43c6}, the corresponding identifiability lower bounds are not well defined. We therefore leave the study of this case for future work and do not discuss it further in this paper.

%When we do not consider confounders, as discussed in \citet{mokhtarian2023unified}, for some DG $\cG = (\bV,\bD)$, it is not possible to accurately infer the skeleton just from observational data, and even all size-one experiments cannot identify them. They also propose that to identify all directed edges $\bD$, the experiments should satisfy that  $|\bI| \geq \zeta_{max}-1$ for all $\bI\in\cI$ and $|\cI| \geq \zeta_{max}$ where $\zeta_{max}$  denotes the size of the largest SCC of $\cG$ .

%When we consider confounders, that is, by introducing bidirected edges, the resulting DMG is $\cG=(\bV,\bD,\bB)$, and identifying it through $r$-separation relations becomes even more difficult. In Section \ref{sec:lower D}, we show that even if we aim to identify only the directed part of $\cG$, i.e., $\cRB(\cG)$, the lower bound of the corresponding $\max_{\bI\in\cI}|\bI|$ and $|\cI|$ will still be higher than in the case of a DG. In Section \ref{sec:lower NB}, we give the lower bound of $\max_{\bI\in\cI}|\bI|$ and $|\cI|$ to identify non-adjacent bidirected edges $\bB^N$ when we already know the directed edges $\bD$. As to adjacent bidirected edges $\bB^A$, since it cannot be identified using only $r$-separation relationships and requires the additional introduction of the do-see test \citep{NIPS2017_291d43c6}, the corresponding identifiability lower bound is not clearly defined, and we will not discuss it further here.

\subsection{Lower Bound of Identifying Directed Edges $\bD$}
\label{sec:lower D}
To discuss the lower bound of $\max_{\bI\in\cI}|\bI|$ and $|\cI|$ to recognize the directed edges $\bD$, we start from the following lemma.

\begin{lemma}
\label{lem:lower D}
    Consider a DMG $\cG = (\bV,\bD,\bB)$ with a node $S\in\bV$ which satisfies the following constrains,
    \begin{itemize}
        \item $Pa_{\cG}(S) = Anc_\cG(S)\backslash \{S\}$ and there exists a bidirected edge between $S$ and each of its parents.
        \item There exists a node $F\in\bV$ that is a parent of all other nodes in $\cG$, and bidirected edges exist between $F$ and every other node.
    \end{itemize}
   For a set of experiment $\cI$, if there do not exist a $\bI\in\cI$ such that $S\notin\bI$ and $Pa_\cG(S)\subseteq\bI$, then we have 
   \begin{equation*}
       \cG^{\prime} = (\bV, \bD\backslash(F,S),\bB) \in [\cG]^r_{\cI}.
   \end{equation*}
\end{lemma}
\begin{proof}
    To show $\cG^{\prime}  \in [\cG]^r_{\cI}$, we only need to show $IM_r(\cG_{\Bar{\bI}}) = IM_r(\cG_{\Bar{\bI}})^\prime$ for any $\bI\in\cI$. If $S \in\bI$, then $\cG_{\Bar{\bI}} = \cG_{\Bar{\bI}}^\prime$ and we get $IM_r(\cG_{\Bar{\bI}}) = IM_r(\cG_{\Bar{\bI}})^\prime$. In the following, we only consider $\bI$, which does not contain $S$.
    Since $\cG^\prime \subseteq \cG$, we have $IM_r(\cG_{\Bar{\bI}}) \subseteq IM_r(\cG_{\Bar{\bI}}^\prime)$. To complete the proof, we need to show $IM_r(\cG_{\Bar{\bI}}^\prime) \subseteq IM_r(\cG_{\Bar{\bI}})$. Let $(X,Y,\bZ) \in IM_r(\cG_{\Bar{\bI}}^\prime)$, we next prove that $(X,Y,\bZ) \in IM_r(\cG_{\Bar{\bI}})$ holds by considering two separate scenarios. 

    \textbf{Scenario 1:} $F\in\bI$. As $Pa_\cG(S)\nsubseteq\bI$, we know that there exists a parent of $S$ that does not belong to $\bI$, denote it by $Q$, and $Q \neq F$. Since $(X,Y,\bZ) \in IM_r(\cG_{\Bar{\bI}}^\prime)$, we have every path between $X$ and $Y$ in $\cG_{\Bar{\bI}}^\prime$ is $r$-blocked by $\bZ$, we will show that these path is also $r$-blocked by $\bZ$ in $\cG_{\Bar{\bI}}$. Otherwise, by contradiction, suppose that there exists a path $\cP$ between $X$ and $Y$ is $r$-blocked by $\bZ$ in $\cG_{\Bar{\bI}}^\prime$ but not $r$-blocked in $\cG_{\Bar{\bI}}$. If $Q \in Anc_{\cG_{\Bar{\bI}}}(\bZ\cup\{X,Y\})$, we use $F\rightarrow Q \leftrightarrow S$ to replace $F\rightarrow S$ to get a new path $\Tilde{\cP}$. If $Q \notin Anc_{\cG_{\Bar{\bI}}}(\bZ\cup\{X,Y\})$, we use $F\rightarrow Q \rightarrow S$ to replace $F\rightarrow S$ to get a new path $\Tilde{\cP}$. 
    We can verify that the modified path $\Tilde{\cP}$ is not $r$-blocked by $\bZ$ in $\cG_{\Bar{\bI}}$ and does not use the edge $(F,S)$. Therefore, removing $(F,S)$ does not affect the local edge pattern of $\Tilde{\cP}$. In addition, since $Q\notin \bI$, the directed connectivity from $F$ to $S$ is preserved via $F\to Q\to S$. 
    Hence, the strongly connected component memberships of vertices on $\Tilde{\cP}$ remain unchanged, and $Anc_{\cG_{\Bar{\bI}}}(\bZ\cup\{X,Y\}) = Anc_{\cG_{\Bar{\bI}}^\prime}(\bZ\cup\{X,Y\})$, implying that $\Tilde{\cP}$ is still not $r$-blocked by $\bZ$ in $\cG_{\Bar{\bI}}^\prime$.
    Now we get a path $\Tilde{\cP}$ between $X$ and $Y$, and it is not $r$-blocked by $\bZ$ in  $\cG_{\Bar{\bI}}^\prime$. Contradiction reached. 

    \textbf{Scenario 2:} %$F\notin \bI$. Similar to the proof in scenario 1, we proceed by contradiction and assume that there exists a path $\cP$ between $X$ and $Y$ is $r$-blocked by $\bZ$ in $\cG_{\Bar{\bI}}^\prime$ but not $r$-blocked in $\cG_{\Bar{\bI}}$. If the path $\cP$ contain $(F,S)$, i.e, $\cP = (X\cdots  F\rightarrow S \cdots Y)$ (without loss of generality, assume that the right side of $S$ no longer contains the edge $(F, S)$.), we make the following replacement: 1) if $F\in Anc_{\cG_{\Bar{\bI}}}(\bZ\cup\{X,Y\})$, $\Tilde{\cP}  = (X\leftrightarrow F \leftrightarrow S \cdots Y)$ and maintain the path between $S$ and $Y$ same as it in $\cP$,; 2) if $F\notin Anc_{\cG_{\Bar{\bI}}}(\bZ\cup\{X,Y\})$, $\Tilde{\cP}  = (X\leftarrow F \leftrightarrow S \cdots Y)$ and maintain the path between $S$ and $Y$ same as it in $\cP$. We can verify that $\Tilde{\cP}$ remains not $r$-blocked by $\bZ$ in $\cG_{\Bar{\bI}}$ and does not use the edge $(F,S)$. Therefore, removing $(F,S)$ does not affect $\Tilde{\cP}$, and $\Tilde{\cP}$ remains not $r$-blocked by $\bZ$ in $\cG_{\Bar{\bI}}^\prime$, except for the specific case discussed below.The specific case need to satisfies the following three conditions: 1) $Pa_\cG(S)\backslash\{F\} \subseteq\bI$ and $S\in SCC_{\cG_{\Bar{\bI}}}(F)$; 2) $\Tilde{\cP}$ contains edge $(S, F)$ and $S\in Anc_{\cG_{\Bar{\bI}}}(\bZ\cup\{X,Y\})$ 3) only consider the $\sigma$-blocked. The core reason lies in the fact that $S$ and $F$ are the only nodes in their strongly connected component in $\cG_{\Bar{\bI}}$, since all other parents of $S$ are subject to intervention. The removal of edge $(F, S)$ modifies the SCC structure and consequently affects the $\sigma$-blocked property. If this specific case occurs, we can simply replace $(S,F)$ in $\Tilde{\cP}$ with $[S,F]$ to get $\Tilde{\cP}^\prime$  and we have $\Tilde{\cP}^\prime$ is not $\sigma$-blocked in $\cG_{\Bar{\bI}}^\prime$, since the $SCC_{\cG_{\Bar{\bI}}}(S)$ contains only two nodes and incoming edges on the left side of $S$ must point to $S$. Otherwise, $\Tilde{\cP}$ is already not $r$-blocked in $\cG_{\Bar{\bI}}^\prime$. Contradiction reached. 
    $F\notin \bI$. Similar to Scenario 1, we proceed by contradiction and assume that there exists a path $\cP$ between $X$ and $Y$ that is $r$-blocked by $\bZ$ in $\cG_{\Bar{\bI}}^\prime$ but not $r$-blocked in $\cG_{\Bar{\bI}}$. Suppose that $\cP$ contains the edge $(F,S)$, i.e., $\cP=(X\cdots F\to S\cdots Y)$.
Without loss of generality, assume that the segment of $\cP$ to the right of $S$ does not contain the edge $(F,S)$.

We construct a modified path $\Tilde{\cP}$ by considering the following cases.

\textit{Case 1:} $X\notin \bI$. In this case, we have $F\in Anc_{\cG_{\Bar{\bI}}}(\bZ\cup\{X,Y\})$. If $X\neq F$ We replace the segment $F\to S$ in $\cP$ with $F\leftrightarrow S$ and define
$\Tilde{\cP}=( X\leftrightarrow F\leftrightarrow S\cdots Y)$,
while keeping the segment from $S$ to $Y$ unchanged. If $X = F$, we set $\Tilde{\cP} = (F\leftrightarrow S \cdots Y)$, while keeping the segment from $S$ to $Y$ unchanged.

\textit{Case 2:} $X\in \bI$. Since $X$ is intervened upon, any path starting from $X$ must move along an outgoing directed edge. We further divide this case into two subcases.

\textit{Case 2a:} There exists a collider on the segment from $X$ to the rightmost occurrence of $F\to S$ in $\cP$, excluding $F$ and $S$. Let $H$ denote the first such collider. Then $\cP$ can be written as
$\cP=(X\to\cdots H\cdots F\to S\cdots Y)$,
and clearly $H\notin \bI$. If $F\in Anc_{\cG_{\Bar{\bI}}}(\bZ\cup\{X,Y\})$, we replace the segment from $H$ to $S$ by
$H\leftrightarrow F\leftrightarrow S$;
if $F\notin Anc_{\cG_{\Bar{\bI}}}(\bZ\cup\{X,Y\})$, we replace it by
$H\leftarrow F\leftrightarrow S$.
The resulting path is denoted by $\Tilde{\cP}$, and the segment from $S$ to $Y$ is kept unchanged.

\textit{Case 2b:} There is no collider on the segment from $X$ to the rightmost occurrence of $F\to S$ in $\cP$. Then $X\in Anc_{\cG_{\Bar{\bI}}}(F)$, and hence $X\in Anc_{\cG_{\Bar{\bI}}}(S)$. Since $Anc_{\cG}(S)=Pa_{\cG}(S)\cup\{S\}$ and $S\notin \bI$, we have $X\in Pa_{\cG_{\Bar{\bI}}}(S)$. In this case, we replace the segment from $X$ to $S$ in $\cP$ by the single edge $X\to S$, and define
$\Tilde{\cP}=(X\to S\cdots Y)$.

In all the above cases, $\Tilde{\cP}$ does not use the edge $(F,S)$, and one can verify that $\Tilde{\cP}$ remains not $r$-blocked by $\bZ$ in $\cG_{\Bar{\bI}}$. We next show how to modify $\Tilde{\cP}$, if necessary, so that it also remains not $r$-blocked by $\bZ$ in $\cG_{\Bar{\bI}}^\prime$.
Since $\Tilde{\cP}$ no longer contains $(F,S)$, the only two possible reasons why $\Tilde{\cP}$ may become $r$-blocked in $\cG_{\Bar{\bI}}^\prime$ are the following:
(i) the deletion of $(F,S)$ changes the ancestor set $Anc_{\cG_{\Bar{\bI}}}(\bZ\cup\{X,Y\})$; or
(ii) the deletion of $(F,S)$ changes the SCC structure involving $F$ and $S$.
We discuss these two cases separately.

First, consider the change of the ancestor set. Clearly,
$Anc_{\cG_{\Bar{\bI}}^\prime}(\bZ\cup\{X,Y\}) \subseteq Anc_{\cG_{\Bar{\bI}}}(\bZ\cup\{X,Y\})$.
If the inclusion is strict, then necessarily $S\in \bZ\cup\{X,Y\}$ and every node in $(\bZ\cup\{X,Y\})\backslash\{S\}$ belongs to $\bI$; otherwise, for some node $A\in (\bZ\cup\{X,Y\})\backslash\{S\}$ with $A\notin \bI$, the edge $F\to A$ remains in $\cG_{\Bar{\bI}}^\prime$ because $F$ is a parent of every node in $\cG$, and thus any ancestor relation lost by deleting $(F,S)$ would be recovered through $A$, a contradiction.
Moreover, if $\Tilde{\cP}$ is not $r$-blocked with respect to $Anc_{\cG_{\Bar{\bI}}}(\bZ\cup\{X,Y\})$ but becomes $r$-blocked with respect to $Anc_{\cG_{\Bar{\bI}}^\prime}(\bZ\cup\{X,Y\})$, then there must exist a collider $H$ on $\Tilde{\cP}$ such that
$H\in Anc_{\cG_{\Bar{\bI}}}(\bZ\cup\{X,Y\})$
but
$H\notin Anc_{\cG_{\Bar{\bI}}^\prime}(\bZ\cup\{X,Y\})$.
We next show that such a collider must be $F$.
Indeed, to satisfy
$H\in Anc_{\cG_{\Bar{\bI}}}(S)$
but
$H\notin Anc_{\cG_{\Bar{\bI}}^\prime}(S)$,
we must have $H\in Pa_{\cG_{\Bar{\bI}}}(S)\backslash Pa_{\cG^\prime_{\Bar{\bI}}}(S) $, since $S\notin \bI$. Therefore, the only possibility is $H=F$.
We now repair this case. Consider a collider segment of the form
$A\to(\leftrightarrow)F\leftarrow(\leftrightarrow)B$
on $\Tilde{\cP}$.  
If at least one of the two edges incident to $F$ in this collider segment is bidirected, without loss of generality assume that the left edge is bidirected. Then the collider segment has the form
$A\leftrightarrow F\leftarrow(\leftrightarrow)B$.
In particular, $A\notin \bI$.
If $A\neq S$, we rewrite this segment as
$A\leftarrow F\leftarrow(\leftrightarrow)B$.
If $A=S$, we further distinguish two cases according to the edge between $F$ and $B$.
If this edge is bidirected, then we rewrite the segment
$S\leftrightarrow F\leftrightarrow B$
as
$S\leftrightarrow F\to B$.
Otherwise, the segment is of the form
$S\leftrightarrow F\leftarrow B$,
and we rewrite it as
$S\leftarrow B$. In all cases, we obtain a modified path $\Tilde{\cP}^\prime$. Since $S\in \bZ\cup\{X,Y\}$, the modified path $\Tilde{\cP}^\prime$ remains not $r$-blocked in $\cG^\prime_{\Bar{\bI}}$.
If both sides are directed edges, then the only possible collider configuration is
$A\to F\leftarrow B$.
In this case, both $A$ and $B$ belong to $Anc_{\cG_{\Bar{\bI}}^\prime}(F)$, and hence to $Anc_{\cG_{\Bar{\bI}}}(F)$. Since $F\to S$ in $\cG_{\Bar{\bI}}$, we have $A,B\in Anc_{\cG_{\Bar{\bI}}}(S)=Pa_{\cG_{\Bar{\bI}}}(S)\cup\{S\}$. If $A,B\neq S$, it follows that
$A,B\in Pa_{\cG_{\Bar{\bI}}^\prime}(S)$.
Therefore, we can replace $F$ with $S$, rewrite this segment as
$A\to  S\leftarrow B$. If one of $A$ and $B$ is equal to $S$, without loss of generality assume $A=S$. In this case, we rewrite the segment $S\to F\leftarrow B$ as $S\leftarrow B$. Thereby, we also obtain a modified path $\Tilde{\cP}^\prime$. As we discussed before, $S\in \bZ\cup\{X,Y\}$. By the above modification, we obtain a path $\Tilde{\cP}^\prime$ that remains not $r$-blocked by $\bZ$ in $\cG_{\Bar{\bI}}^\prime$ with respect to the ancestor-set issue.

Finally, we consider the second special case, namely the change of SCC structure caused by removing $(F,S)$. This can only occur when all of the following conditions hold:
1) $Pa_{\cG}(S)\backslash\{F\}\subseteq \bI$ and $S\in SCC_{\cG_{\Bar{\bI}}}(F)$;
2) $\Tilde{\cP}$ contains the edge $(S,F)$ and $S\in Anc_{\cG_{\Bar{\bI}}}(\bZ\cup\{X,Y\})$; and
3) we are considering the $\sigma$-blocked case.
The core reason is that $S$ and $F$ are then the only nodes in their strongly connected component in $\cG_{\Bar{\bI}}$, since all other parents of $S$ are intervened upon. Hence, removing $(F,S)$ changes the SCC structure and may affect the $\sigma$-blocked status of $\Tilde{\cP}$.
If this specific case occurs, we replace the edge $(S,F)$ in $\Tilde{\cP}$ with $[S,F]$ and obtain a new path $\Tilde{\cP}^{\prime\prime}$. Then $\Tilde{\cP}^{\prime\prime}$ is not $\sigma$-blocked in $\cG_{\Bar{\bI}}^\prime$, since $SCC_{\cG_{\Bar{\bI}}}(S)$ contains only two nodes and any incoming edge on the left side of $S$ must point to $S$. Therefore, after the above correction, we always obtain a path that is not $r$-blocked by $\bZ$ in $\cG_{\Bar{\bI}}^\prime$, contradicting the assumption. This completes the proof.
    
\end{proof}

Lemma \ref{lem:lower D} tells us that, in the worst-case, identifying all parents of a node at least requires intervening on all of its parents.

%Then we give the following definitions to describe the ancestor relationship in DMG. Consider a DMG $\cG = (\bV,\bD,\bB)$ with $s$ SCCs, denoted by $\{\bS_1,\dots,\bS_s\}$. These SCCs is a partition of $\bV$. We say that an SCC $\bS_i$ is an SCC-ancestor of another SCC $\bS_j$, if there exists a node $X\in\bS_i$ that is an ancestor of some node $Y\in\bS_j$.  More specifically, if there exists a directed path from $X$ to $Y$ that passes through at most $k-1$ other SCCs, then we say that $\bS_i$ is an $k$-order SCC-ancestor of $\bS_j$. If every SCC in the DMG has at most $l$-order SCC-ancestors, we define the SCC-Anc length of the graph to be $l$.

We now introduce the following definitions to describe ancestor relationships at the level of SCCs in a DMG. Consider a DMG $\cG = (\bV, \bD, \bB)$ with $s$ SCCs, denoted by ${\bS_1, \dots, \bS_s}$. These SCCs form a partition of the node set $\bV$.
We say that an SCC $\bS_i$ is an SCC-ancestor of another SCC $\bS_j$ if there exists a node $X \in \bS_i$ and a node $Y \in \bS_j$ such that $X$ is an ancestor of $Y$ in $\cG$.
%More generally, if there exists a directed path from $X \in \bS_i$ to $Y \in \bS_j$ that passes through at most $k - 1$ other SCCs (excluding $\bS_i$ and $\bS_j$), we say that $\bS_i$ is a $k$-order SCC-ancestor of $\bS_j$.
More generally, for two distinct SCCs $\bS_i$ and $\bS_j$, we say that
$\bS_i$ is a $k$-order SCC-ancestor of $\bS_j$ if there exists a directed
path from some $X\in\bS_i$ to some $Y\in\bS_j$ that passes through exactly
$k-1$ other SCCs, and no such directed path passes through more than
$k-1$ other SCCs. The SCC-Anc length of $\cG$ is defined as the maximum
order of an SCC-ancestor relation in $\cG$, with length zero if no such
relation exists.

\begin{definition}[SCC-Anc partition]
\label{def: scc-anc partition}
For a DMG $\cG$ with SCC-Anc length $l$, we define the SCC-Anc partition as $\bbT^\cG = {\cT_1^\cG, \dots, \cT_{l+1}^\cG}$, where each $\cT_k^\cG$ is a collection of SCCs at a particular ancestor level.
\begin{itemize}
    \item The first layer $\cT_1^\cG$ is defined as:
    \begin{equation*}
        \cT_1^\cG = \{\bS| \bS \text{ is an SCC of $\cG$ and it has no SCC-ancestor}\}.
    \end{equation*}
    \item For $2 \leq k \leq l+1$, the $k$-th layer is defined as:

    \[
\cT_k^{\cG}
=
\left\{
\bS \,\middle|\,
\begin{array}{l}
\bS \text{ is an SCC of } \cG,\\ \text{ there exists }
\bS'\in\cT_1^{\cG}\text{ such that }\bS'
\text{ is a }(k-1)\text{-order SCC-ancestor of }\bS,\\
\text{and no }\widetilde{\bS}\in\cT_1^{\cG}
\text{ is an }r\text{-order SCC-ancestor of }\bS
\text{ for any }r>k-1
\end{array}
\right\}.
\]
\iffalse
    \begin{equation*}
    \begin{aligned}
         \cT_k^\cG = \{\bS|& \text{$\bS$ is a SCC of $\cG$. There exist a $\bS^\prime\in\cT_1^\cG$ is a $k-1$-order SCC-ancestor of $\bS$ }\\  &\text{and there are no $\Tilde{\bS}\in\cT_1^\cG$  is the $r$-order SCC-ancestor of $\bS$ for $r > k-1$}\}.
    \end{aligned}
    \end{equation*}
\fi
\end{itemize}
We use $\zeta^{k,\cG}_{\max}$ to denote the size of the largest SCC in the $k$-th layer, i.e., in $\cT_k^\cG$.
%For a DMG $\cG$ with SCC-Anc length $l$, define the SCC-Anc partition as $\bbT^\cG = \{\cT_1^\cG,\dots,\cT_{l+1}^\cG\}$ where $\cT_1^\cG = \{\bS| \bS \text{ is the SCC of $\cG$ and it has no SCC-ancestor}\}$ and for $2\leq k\leq l+1$, $\cT_k^\cG = \{\bS| \text{There exist a $\bS^\prime\in\cT_1^\cG$ is a $k$-order SCC-ancestor of $\bS$ }\\  \text{and there are no $\Tilde{\bS}\in\cT_1^\cG$ is the $r$-order SCC-ancestor of $\bS$ for $r > k$}\}$. We use $\zeta^{k,\cG}_{max}$ to denote the size of the largest SCC in $\cT_k^\cG$.
\end{definition}

Consider  a DMG $\cG = (\bV,\bD,\bB)$ with its SCC-Anc partition $\bbT^\cG = \{\cT_1^\cG,\dots,\cT_{l+1}^\cG\}$, we define $\rmT_k^\cG = \cup_{j=1}^{k-1}\cT_j^\cG$ for $2\leq k \leq l+1$ and $\rmT_1^\cG = \varnothing$. We use $|\rmT_k^\cG|$ to denote the number of SCCs in $\rmT_k^\cG$ and use  $|\rmT_k^\cG|_n$ to denote the number of nodes in $\rmT_k^\cG$. It is easy to find that for any node $X\in\bS\in\cT_k^\cG$, $Pa_{\cG}(X) \subseteq \bS\cup\rmT_k^{\cG}$.

\begin{figure}[ht]
    \centering
        \includegraphics[width=0.65\textwidth]{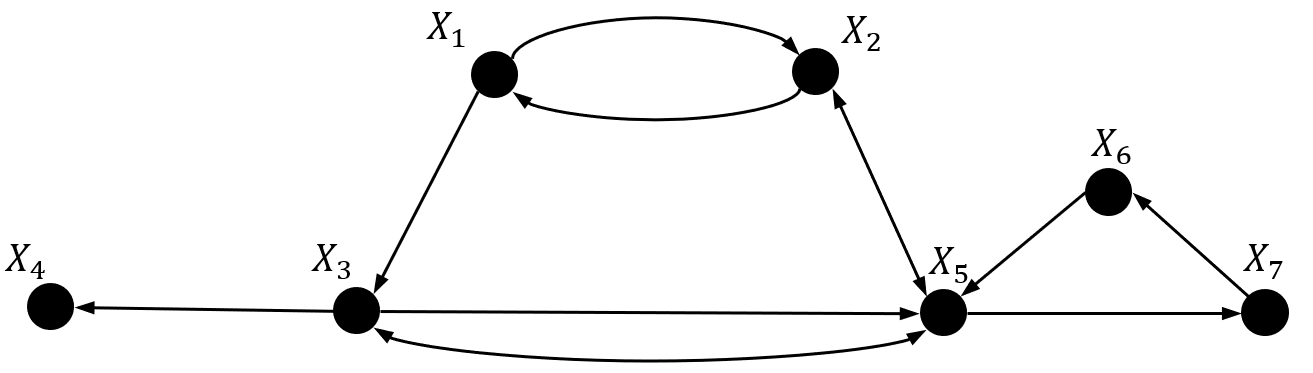}
        \caption{DMG of Examples \ref{exam:SCC-Anc partition} and \ref{exam: SCC-Anc separating system}.}
        \label{fig:example SCC-Anc patition}
\end{figure}

\begin{example}[SCC-Anc partition]
\label{exam:SCC-Anc partition}
    Consider the DMG $\cG$ shown in Figure \ref{fig:example SCC-Anc patition}, there are four SCCs in $\cG$, i.e., $\{\{X_1,X_2\},\{X_3\},\{X_4\},\{X_5,X_6,X_7\}\}$.  The SCC-Anc length is $2$ and the corresponding SCC-Anc partition is $\cT_1^\cG = \{\{X_1,X_2\}\}$, $\cT_2^\cG = \{\{X_3\}\}$ and $\cT_3^\cG = \{\{X_4\},\{X_5,X_6,X_7\} \}$. Then $\rmT_1^\cG = \varnothing$, $\rmT_2^\cG = \cT_1^\cG$ and $\rmT_3^\cG = \cT_1^\cG\cup\cT_2^\cG$. $|\rmT_3^\cG| = 2,|\rmT_3^\cG|_n = 3$.
\end{example}

We now present two important consequences of Lemma \ref{lem:lower D}.
\begin{theorem}
    \label{thm:lower D bI}
    Fix a vertex set $\bV$ with $|\bV|=n$ and any feasible integers $l$, $|\rmT_{l+1}|_n$, and $\zeta_{\max}^{l+1}$ satisfying $1\leq l\leq n-1$, $l\leq |\rmT_{l+1}|_n\leq n-1$, $\zeta_{\max}^{l+1}\geq 1$, and $|\rmT_{l+1}|_n+\zeta_{\max}^{l+1}\leq n$. Then there exists a DMG $\cG$ over $\bV$ whose SCC-Anc partition is $\bbT^\cG=\{\cT_1^\cG,\dots,\cT_{l+1}^\cG\}$ and which realizes the prescribed parameters: the SCC-Anc length of $\cG$ is $l$, the number of variables in $\rmT_{l+1}^{\cG}$ is $|\rmT_{l+1}^{\cG}|_n=|\rmT_{l+1}|_n$, and the maximum SCC size at level $l+1$ is $\zeta_{\max}^{l+1,\cG}=\zeta_{\max}^{l+1}$. For this graph, any collection of experiments $\cI$ satisfying
\begin{equation}
\label{eq:thm lower D bI}
    \max_{\bI\in\cI}|\bI|
    <
    |\rmT_{l+1}^{\cG}|_n+\zeta_{\max}^{l+1,\cG}-1
\end{equation}
fails to uniquely identify its directed structure; equivalently, $|\cRB([\cG]^r_{\cI})|>1$.
\end{theorem}

\begin{proof}
    We construct a DMG $\cG = (\bV,\bD,\bB)$ with $n$ nodes by the following steps. 
    \begin{itemize}
        \item  Partition the $n$ nodes into $\{\bS_1,\dots,\bS_s\}$, and then assign each of these $s$ subsets to one of the sets  $\{\cT_1^\cG,\dots,\cT_{l+1}^\cG\}$, such that each  $\cT_k^\cG$ for $2\leq k\leq l+1$ contains at least one subset, $\cT_{1}^{\cG}$ contains exactly one subset, $\cup_{k=1}^l\cT_k^\cG$ contain $|\rmT_{l+1}|_n$ nodes and the size of largest subset in $\cT_{l+1}^{\cG}$ is $\zeta_{max}^{l+1}$.  
        \item For any $\bS\in\cT_k^\cG, 1\leq k \leq l+1 $ and $X,Y\in\bS$, let $(X,Y) \in\bD, (Y,X)\in\bD$ and $[X,Y]\in\bB$.
        \item For any $\bS\in\cT_k^\cG$, $X\in\bS$ and $\bS^\prime \in \cup_{j=1}^{k-1}\cT_j^\cG$, $Y\in\bS^\prime$, $1\leq k \leq l+1$, let $(Y,X)\in\bD$ and $[Y,X]\in\bB$.
    \end{itemize}
    It is easy to check that $\bbT^\cG = \{\cT_1^\cG,\dots,\cT_{l+1}^\cG\}$ is an SCC-Anc partition of $\cG$ and $|\rmT_{l+1}^{\cG}|_n=|\rmT_{l+1}|_n, \zeta^{l+1,\cG}_{max} = \zeta^{l+1}_{max}$. 

    Consider the largest SCC in $\cT_{l+1}^{\cG}$, denoted by $\bS$, we can verify that satisfies the condition mentioned in Lemma \ref{lem:lower D} if we regard any node $X$ in $\bS$ as $S$ and regard any node in $\cT_1^{\cG}$ as $F$. Using Lemma \ref{lem:lower D}, we know that if there does not exist a $\bI\in\cI$ such that $X\notin \bI$ and $Pa_\cG(X) \subseteq \bI$, there will be another $\cG^\prime \in [\cG]^r_{\cI}$ and $\cRB(\cG)\neq \cRB(\cG^\prime)$.  However, by our contraction, $|Pa_{\cG}(X)| = |\rmT_{l+1}^{\cG}|_n + \zeta^{l+1,\cG}_{max} -1$. Combining Equation (\ref{eq:thm lower D bI}), we get that the condition of Lemma \ref{lem:lower D} holds, so we complete the proof.
\end{proof}

\begin{corollary}
    \label{coro: lower D bI}In the worst case, i.e., the setting in the proof of Theorem \ref{thm:lower D bI}, the directed edges of DMG $\cG$ cannot be learned by any algorithm that performs experiments with the maximum size less than $|\rmT_{l+1}^{\cG}|_n + \zeta^{l+1,\cG}_{max} -1$.
\end{corollary}

\begin{theorem}
    \label{thm:lower D cI}
     For a vertex set $\bV$ with $|\bV|=n$ and any feasible integers $l$ and $\{\zeta_{\max}^{k}\}_{k=1}^{l+1}$ satisfying $1\leq l\leq n-1$, $\zeta_{\max}^{k}\geq 1$ for every $k$, and $\sum_{k=1}^{l+1}\zeta_{\max}^{k}\leq n$. Then there exists a DMG $\cG$ over $\bV$ whose SCC-Anc partition is $\bbT^\cG=\{\cT_1^\cG,\dots,\cT_{l+1}^\cG\}$ and which realizes the prescribed parameters: the SCC-Anc length of $\cG$ is $l$, and the maximum SCC size at level $k$ satisfies $\zeta_{\max}^{k,\cG}=\zeta_{\max}^{k}$ for every $1\leq k\leq l+1$. For this graph, any collection of experiments $\cI$ satisfying
\begin{equation}
\label{eq:thm lower D cI}
    |\cI|<\sum_{k=1}^{l+1}\zeta_{\max}^{k,\cG}
\end{equation}
fails to uniquely identify its directed structure; equivalently, $|\cRB([\cG]^r_{\cI})|>1$.
\end{theorem}
\begin{proof}
    We construct a DMG $\cG = (\bV,\bD,\bB)$ with $n$ nodes by the following steps, which is similar to the construction in the proof of Theorem \ref{thm:lower D bI}. 
    \begin{itemize}
        \item  Partition the $n$ nodes into $\{\bS_1,\dots,\bS_s\}$, and then assign each of these $s$ subsets to one of the sets  $\{\cT_1^\cG,\dots,\cT_{l+1}^\cG\}$, such that each  $\cT_k^\cG$ for $2\leq k\leq l+1$ contains at least one subset, $\cT_{1}^{\cG}$ contains exactly one subset and the size of largest subset in $\cT_{k}^{\cG}$ is $\zeta_{max}^{k}$ for $1\leq k\leq l+1$. 
        \item For any $\bS\in\cT_k^\cG, 1\leq k \leq l+1 $ and $X,Y\in\bS$, let $(X,Y) \in\bD, (Y,X)\in\bD$ and $[X,Y]\in\bB$.
        \item For any $\bS\in\cT_k^\cG$, $X\in\bS$ and $\bS^\prime \in \cup_{j=1}^{k-1}\cT_j^\cG$, $Y\in\bS^\prime$, $1\leq k \leq l+1$, let $(Y,X)\in\bD$ and $[Y,X]\in\bB$.
    \end{itemize}
    It is easy to check that $\bbT^\cG = \{\cT_1^\cG,\dots,\cT_{l+1}^\cG\}$ is an SCC-Anc partition of $\cG$ and $\zeta^{k,\cG}_{max} = \zeta^{k}_{max}$  for $1\leq k\leq l+1$. 

    Now, fix $k \in \{1,\dots,l+1\}$ we denote the largest SCC in $\cT_{k}^{\cG}$ as $\bS_k$. Fixed a node $X\in\bS_k$, using Lemma \ref{lem:lower D}, to recognize its parents node, we need to have $\bI\in\cI$ such that $X\notin\bI$ and $\rmT_{k}^\cG\cup\bS\backslash\{X\}\subseteq \bI$. By enumerating all nodes in $\bS_k$, we obtain $\zeta_{max}^{k,\cG}$ distinct interventions. Then, by enumerating $k = 1,\dots,l+1$,  we obtain in a total of $\sum_{k=1}^{l+1} \zeta_{max}^{k,\cG}$ distinct interventions. 
    Since for each node $X$ in the largest SCC $\bS_k$ of layer $k$, identifying its parents requires an experiment $I \in \mathcal{I}$ such that $X \notin \bI,  Pa_\cG(X)=\rmT_k^\cG \cup (\bS_k \setminus \{X\}) \subseteq \bI$, the intervention requirements corresponding to two different target nodes cannot be satisfied simultaneously in a single experiment. Hence, these experiments cannot be merged, and at least $\sum_{k=1}^{l+1}\zeta^{k,\cG}_{\max}$
rounds of intervention are necessary.
\end{proof}

Theorems \ref{thm:lower D bI} and \ref{thm:lower D cI} establish worst-case lower bounds on the maximum number of nodes that must be intervened on per experiment, as well as on the total number of experiments required to identify the directed structure of a DMG in the presence of bidirected edges. 

\begin{corollary}
    \label{coro: lower D cI}
    At least $\sum_{k=1}^{l+1} \zeta_{max}^{k,\cG}$ experiments are required to learn the directed edges of DMG $\cG$ in the worst case.
\end{corollary}

\begin{remark}
    Comparing our results with \cite{mokhtarian2023unified}, when there are no bidirected edges, to identify the directed structure, only $|\cI|\geq \max_{1\leq k\leq l+1}\zeta_{max}^{k,\cG}$. The presence of bidirected edges inherently increases the intervention complexity, making it more difficult to identify the directed structure of the graph. Moreover, whereas \cite{mokhtarian2023unified} establishes a lower bound for each individual intervention set $\bI$, our result provides a lower bound on $\max_{\bI\in\cI}|\bI|$. This difference has no impact on the design of algorithms with bounded intervention sizes, as these algorithms impose constraints on the maximum number of nodes allowed per intervention. Consequently, any valid upper bound must be no smaller than the established lower bound on $\max_{\bI\in\cI}|\bI|$.
\end{remark}

\subsection{Lower Bound of Identifying Non-adjacent Bidirected Edges $\bB^N$}
\label{sec:lower NB}

In this section, we provide the lower bound of $\max_{\bI\in\cI}|\bI|$ and $|\cI|$ to recognize the non-adjacent bidirected edges $\bB^N$. We start from the following lemma.

\begin{lemma}
    \label{lem:lower B}
    Consider a DMG $\cG = (\bV,\bD,\bB)$ with two nodes $S,R\in \bV$ which satisfies the following constraints,
    \begin{itemize}
        \item $(S,R)\notin\bD$, $(R,S)\notin\bD$ and $[S,R]\in \bB$, i.e., $[S,R]\in\bB^N$.
        \item Both \(S\) and \(R\) are connected by bidirected edges to every node in
\(Pa_{\cG}(\{S,R\})\).
    \end{itemize}
    For a set of experiment $\cI$, if there do not exist a $\bI\in\cI$ such that $S\notin\bI,R\notin\bI$ and $Pa_{\cG}(S\cup R) \subseteq \bI$, then we have 
    \begin{equation*}
        \cG^\prime = (\bV, \bD,\bB\backslash [S,R]) \in [\cG]_{\cI}^r.
    \end{equation*}
\end{lemma}
\begin{proof}
    The proof is similar to that of Lemma \ref{lem:lower D}. To show $\cG^{\prime}  \in [\cG]^r_{\cI}$, we only need to show $IM_r(\cG_{\Bar{\bI}}) = IM_r(\cG_{\Bar{\bI}})^\prime$ for any $\bI\in\cI$. If $S \in\bI$ or $R\in\bI$, then $\cG_{\Bar{\bI}} = \cG_{\Bar{\bI}}^\prime$ and we get $IM_r(\cG_{\Bar{\bI}}) = IM_r(\cG_{\Bar{\bI}})^\prime$. In the following, we only consider $\bI$, which does not contain $S$ and $R$.
    Since $\cG^\prime \subseteq \cG$, we have $IM_r(\cG_{\Bar{\bI}}) \subseteq IM_r(\cG_{\Bar{\bI}}^\prime)$. To complete the proof, we need to show $IM_r(\cG_{\Bar{\bI}}^\prime) \subseteq IM_r(\cG_{\Bar{\bI}})$. Let $(X,Y,\bZ) \in IM_r(\cG_{\Bar{\bI}}^\prime)$, we next prove that $(X,Y,\bZ) \in IM_r(\cG_{\Bar{\bI}})$ holds. 

    Since $Pa_{\cG}(S\cup R)\nsubseteq \bI$, we may, without loss of generality, assume that $S$ has a parent $Q$ such that $Q\notin\bI$. Since $(X,Y,\bZ) \in IM_r(\cG_{\Bar{\bI}}^\prime)$, we have every path between $X$ and $Y$ in $\cG_{\Bar{\bI}}^\prime$ is $r$-blocked by $\bZ$, we will show that these path is also $r$-blocked by $\bZ$ in $\cG_{\Bar{\bI}}$. Otherwise, by contradiction, suppose that there exists a path $\cP$ between $X$ and $Y$ is $r$-blocked by $\bZ$ in $\cG_{\Bar{\bI}}^\prime$ but not $r$-blocked in $\cG_{\Bar{\bI}}$.  
    If $Q \in Anc_{\cG_{\Bar{\bI}}}(\bZ\cup\{X,Y\})$, we use $R\leftrightarrow Q \leftrightarrow S$ to replace $R\leftrightarrow S$ to get a new path $\Tilde{\cP}$. If $Q \notin Anc_{\cG_{\Bar{\bI}}}(\bZ\cup\{X,Y\})$, we use $R\leftrightarrow Q \rightarrow S$ to replace $R\leftrightarrow S$ to get a new path $\Tilde{\cP}$. 
    We can verify that the new path $\Tilde{\cP}$ remains not $r$-blocked by $\bZ$ in $\cG_{\Bar{\bI}}$ and does not contain the edge $[S,R]$.
    Moreover, removing $[S,R]$ does not change the strongly connected component memberships of vertices on $\Tilde{\cP}$, and $Anc_{\cG_{\Bar{\bI}}}(\bZ\cup\{X,Y\}) = Anc_{\cG_{\Bar{\bI}}^\prime}(\bZ\cup\{X,Y\})$. Hence, $\Tilde{\cP}$ also remains not $r$-blocked by $\bZ$ in $\cG_{\Bar{\bI}}^\prime$.
     Now we get a path $\Tilde{\cP}$ between $X$ and $Y$, and it is not $r$-blocked by $\bZ$ in  $\cG_{\Bar{\bI}}^\prime$. Contradiction reached. 
\end{proof}

Lemma \ref{lem:lower B} establishes that, in the worst case, identifying a non-adjacent bidirectional edge with certainty necessitates intervening on all parents of both of its endpoints. We now present two important consequences of Lemma \ref{lem:lower B}.

\begin{theorem}
    \label{thm:lower B bI}
    Consider a set of $n$ nodes denoted by $\bV$ and an integer $1 \leq c \leq n-2$. There exists a DMG $\cG = (\bV,\bD,\bB^N\cup\bB^A)$ over $\bV$ with $max_{[X,Y]\in\bB^N }|Pa_{\cG}(X\cup Y) | = c $. If 
    \begin{equation}
    \label{eq:thm lower B bI}
        \max_{\bI\in\cI}|\bI| < \max_{[X,Y]\in\bB^N }|Pa_{\cG}(X\cup Y) |, 
    \end{equation}
    then $|\cRA([\cG]_\cI^r)| > 1$.
\end{theorem}

\begin{proof}
    The proof is straightforward. We first construct a DMG $\cG =  (\bV,\bD,\bB^N\cup\bB^A)$ satisfies the following conditions: 1) $max_{[X,Y]\in\bB^N }|Pa_{\cG}(X\cup Y) | = c$; 2) there exists a bidirected edge between every pair of nodes. This construction is straightforward to achieve. Then we consider $[S,R] = \arg\max_{[X,Y]\in\bB^N }|Pa_{\cG}(X\cup Y) |$. According to Lemma \ref{lem:lower B}, we know that if there does not exist a $\bI\in\cI$ such that $S\notin\bI$, $R\notin\bI$ and $Pa_{\cG}(S\cup R)\subseteq \bI$, there will be another $\cG^\prime\in[\cG]_\cI^r$ and $\cRA(\cG)\neq \cRA(\cG^\prime)$. However, by our contraction, combining Equation (\ref{eq:thm lower B bI}), we get that the condition of Lemma \ref{lem:lower B} holds, so we complete the proof.
\end{proof}

\begin{corollary}
    \label{coro: lower B bI}
    In the worst case, i.e., the setting in the proof of Theorem \ref{thm:lower B bI}, the non-adjacent bidirected edges of DMG $\cG$ cannot be learned by any algorithm that performs experiments with the maximum size less than $max_{[X,Y]\in\bB^N }|Pa_{\cG}(X\cup Y) |$.
\end{corollary}

To facilitate the presentation of the following theorem, we first introduce two definitions.
\begin{definition}[Directed skeleton graph and undirected component graph]
\label{def: Guc}
    For a DMG $\cG=(\bV,\bD,\bB)$, its corresponding \textit{directed skeleton graph}, denoted by $\cG^u$, is the undirected graph obtained by taking the skeleton of its directed part $\cRB(\cG)$.  We further define the \textit{undirected complement graph} $\cG^{uc}=(\bV,\bE)$ as the complement of $\cG^u$, where $\bE=\{[X,Y]^u\mid X,Y\in\bV,\ X\neq Y,\ (X,Y)\notin\bD,\ (Y,X)\notin\bD\}$.
\end{definition}

\begin{definition}[Clique]
Let $\cG=(\bV,\bE)$ be an undirected graph. A nonempty subset $\bC\subseteq\bV$ is called a \textit{clique} of $\cG$ if $[X,Y]^u\in\bE$ for every pair of distinct variables $X,Y\in\bC$. 
\end{definition}

\begin{definition}[Edge clique covering]
    Let $\cG = (\bV,\bE)$ be an undirected graph. A clique edge cover of $\cG$ is a collection of cliques $\bC_1,\dots,\bC_k$, where each $\bC_i\subseteq\bV$, such that every edge in $\bE$ is contained in at least one of the induced subgraphs $\cG[\bC_i]$. That is, 
    \begin{equation*}
        \bE \subseteq\cup_{i=1}^k E(\cG[\bC_i]),
    \end{equation*}
    where $\cG[\bC_i]$ denotes the subgraph of $\cG$ induced by the vertex set $\bC_i$ and $E(\cG[\bC_i])$ denotes its edge set.

    The minimal edge clique covering number of $\cG$, denoted by $cc(\cG)$, is the minimum number $k$ such that there exists a clique edge cover of GG consisting of $k$ cliques.
\end{definition}

\begin{figure}[ht]
    \centering
    % 第一行的单个子图
    
    % 第二行的五个并列子图
    
    \begin{subfigure}[b]{0.18\textwidth}
        \includegraphics[width=\textwidth]{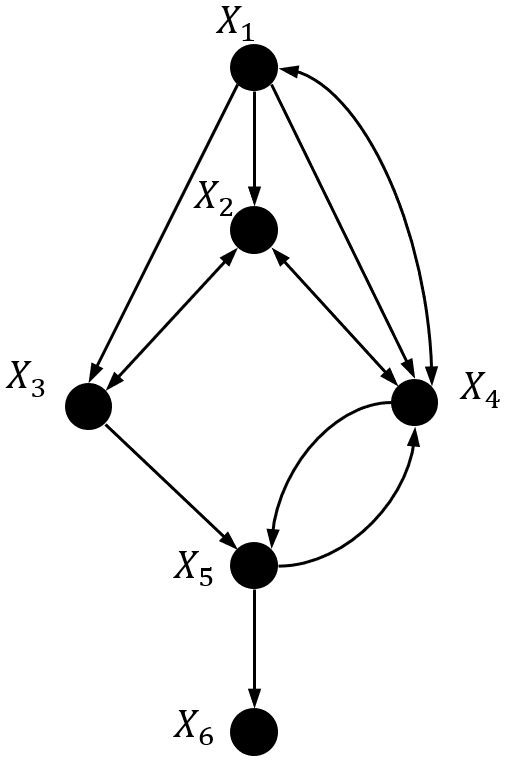}
        \caption{$\cG$}
        \label{fig:example 4.1}
    \end{subfigure}
    \hfill
    \begin{subfigure}[b]{0.25\textwidth}
        \includegraphics[width=\textwidth]{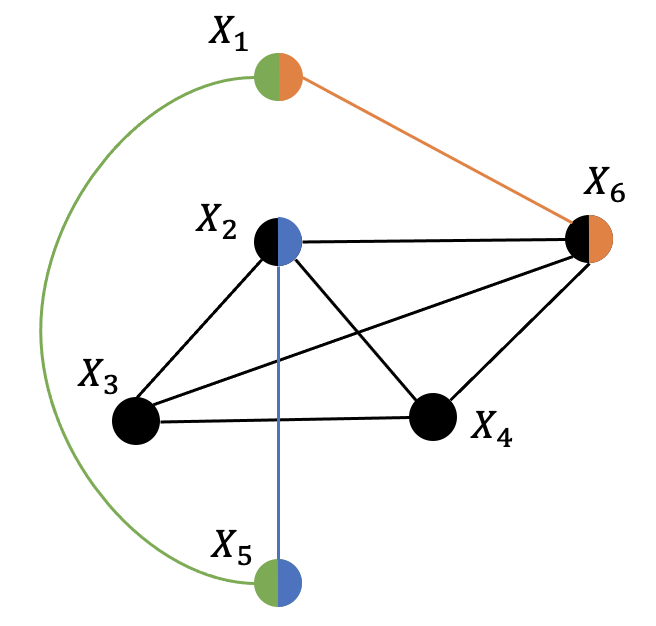}
        \caption{$\cG^{uc}$}
        \label{fig:example 4.2}
    \end{subfigure}
    \hfill
    \begin{subfigure}[b]{0.18\textwidth}
        \includegraphics[width=\textwidth]{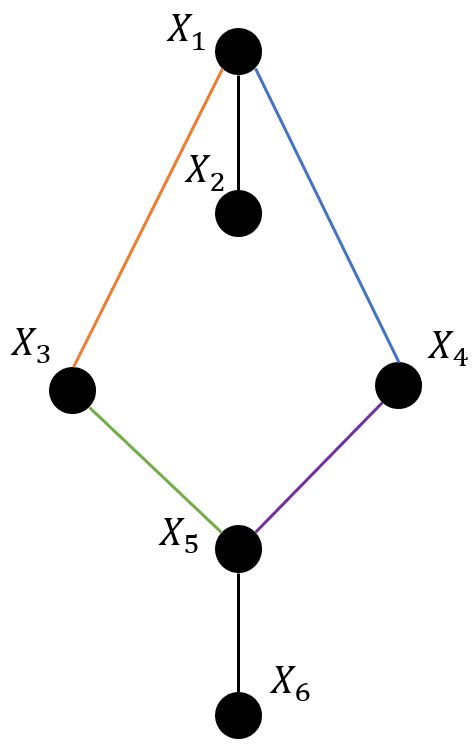}
        \caption{$\cG^{u}$}
        \label{fig:example 4.3}
    \end{subfigure}
   
    \caption{
Examples of DMG $\cG$ and its corresponding $\cG^{uc}$, $\cG^{u}$. In (b), different colors denote different cliques in the edge clique covering. Since a node may belong to multiple cliques, it can be assigned multiple colors. In (c), different colors denote different colors of edges in the strong edge coloring.
}
    \label{fig: cc}
\end{figure}

\begin{example}[Edge clique covering]
\label{exam: edge clique covering}
    Consider the DMG $\cG$ in Figure \ref{fig:example 4.1}, Figure \ref{fig:example 4.2} is its corresponding $\cG^{uc}$ and different colors denote the different cliques in the edge clique covering. In this case, the edge clique covering is $\{\{X_1,X_6\},\{X_1,X_5\},\{X_2,X_5\},\{X_2,X_3,X_4,X_6\}\}$ and we can verify that this is the minimum edge clique covering of $\cG^{uc}$ which means that $cc(\cG^{uc}) = 4$.
\end{example}

\begin{theorem}
\label{thm:lower B cI}
    Consider a set of $n$ nodes denoted by $\bV$ and an integer $1 \leq c \leq \lfloor \frac{n^2}{4}\rfloor$. There exists a DMG $\cG = (\bV,\bD,\bB^N\cup\bB^A)$ over $\bV$ with $cc(\cG^{uc}) = c $. If 
    \begin{equation}
        |\cI| < cc(\cG^{uc}),
    \end{equation}
    then $|\cRA([\cG]_\cI^r)| > 1$.
\end{theorem}
\begin{proof}
    We first construct a DMG $\cG =  (\bV,\bD,\bB^N\cup\bB^A)$ that satisfies the following conditions: 1) $cc(\cG^{uc}) = c$; 2) there exists a bidirected edge between every pair of nodes. This construction is straightforward to achieve, since the $cc(\cG^{uc})$ depends only on the directed edges of $\cG$. For every integer $1\leq c\leq \lfloor n^2/4\rfloor$, one can construct an $n$-node undirected graph $\cG^{uc}$ such that $cc(\cG^{uc})=c$. Moreover, $\lfloor n^2/4\rfloor$ is the maximum possible edge clique covering number of an $n$-node undirected graph, confirming that our prescribed range of $c$ is valid \citep{erdos1966representation}. Using Lemma \ref{lem:lower B}, we know that to recognize the $[S,R]\in\bB^N$, we at least need to construct a $\bI\in\cI$ such that $S\notin\bI,R\notin\bI$ and $Pa_{\cG}(S\cup R)\subseteq \bI$. Then we give the following three important observations and its proof.
    \begin{itemize}
        \item In a single intervention, we can identify at most the non-adjacent bidirected edges in $\cG$ among the nodes forming a single clique $\bC$ in the undirected component graph $\cG^{uc}$.
        
        \textbf{proof: }This is because a clique in $\cG^{uc}$ implies that the corresponding nodes in $\cG$ are pairwise non-adjacent via directed edges. As a result, they may form non-adjacent bidirected edges in $\cG$, and since none of them is a parent of another within the set, it is possible to intervene on their respective parents in $\cG$ without intervening on the nodes in the clique itself. In contrast, if a set of nodes does not form a clique in $\cG^{uc}$, then there must exist at least one pair of nodes within the set that are connected by a directed edge in $\cG$. Without loss of generality, suppose $X$ is a parent of $Y$ in $\cG$. To identify all non-adjacent bidirectional edges involving $Y$, we would need to intervene on $X$. However, doing so would prevent us from identifying the non-adjacent bidirectional edges involving $X$. Repeating this argument recursively eventually leads to a subset of nodes that form a clique in $\cG^{uc}$, which characterizes the maximal set of nodes whose non-adjacent bidirectional edges in $\cG$ can be identified simultaneously through a single intervention.
        \item The minimal number of experiments of identifying all non-adjacent bidirectional edges in $\cG$ can be reduced to the minimal number of edge clique covering problem in the undirected complement graph $\cG^{uc}$. 

        \textbf{proof: }Based on the above observation, this result is immediate: since each intervention can reveal at most the non-adjacent bidirected edges among nodes in a single clique of $\cG^{uc}$, and all such edges in $\cG$ are located where $\cG^{uc}$ has edges, at least $cc(\cG^{uc})$ interventions are necessary.

        \item Suppose $\{\bC_1,\dots,\bC_k\}$ is a minimal edge clique covering of $\cG^{uc}$, then $Pa_{\cG}(\bC_i)\neq Pa_{\cG}(\bC_j)$ for $i\neq j$. 

        \textbf{proof: }
        By contradiction, suppose $Pa_{\cG}(\bC_i)=Pa_{\cG}(\bC_j)$. Then, no node in $\bC_i$ can be a parent or child of any node in $\bC_j$. Otherwise, if there exists a directed edge between some $X\in \bC_i$ and $Y\in \bC_j$, then $X\in Pa_{\cG}(\bC_j)$ or $Y\in Pa_{\cG}(\bC_i)$, which contradicts the equality of the parent sets. Hence, there are no directed edges between $\bC_i$ and $\bC_j$. Since both $\bC_i$ and $\bC_j$ are cliques in $\cG^{uc}$, it follows that $\bC_i\cup \bC_j$ is also a clique in $\cG^{uc}$. This contradicts the assumption that $\{\bC_1,\dots,\bC_k\}$ is a minimum edge clique covering of $\cG^{uc}$.
    \end{itemize}
    Combining the three observations above, we conclude that accurately identifying all non-adjacent bidirected edges in $\cG$ requires at least as many experiments as $cc(\cG^{uc})$. This completes the proof.
    
\end{proof}

\begin{corollary}
    \label{coro: lower B cI}
    At least $cc(\cG^{uc})$ experiments are required to learn the non-adjacent bidirected edges of DMG $\cG$ in the worst case.
\end{corollary}

\begin{remark}
    Even with full knowledge of the directed edge structure of $\cG$, computing the minimal edge clique covering of $\cG^{uc}$ remains NP-hard. Therefore, motivated by the strategy proposed in \cite{NIPS2017_291d43c6}, we employ a probabilistic approach to approximate the covering in our subsequent method.
\end{remark}

Theorems \ref{thm:lower B bI} and \ref{thm:lower B cI} establish worst-case lower bounds on the maximum number of nodes
that must be intervened on per experiment, as well as on the total number of experiments
required to identify the non-adjacent bidirected structure of a DMG.

\section{Unbounded-size Experiment Design}
\label{sec: unb algorithm}
Having established worst-case lower bounds on the maximum experiment size and the total number of experiments in Section~\ref{sec: lower bound}, we now develop an experiment design algorithm that approaches these fundamental limits. In this section, we propose an experiment design algorithm for learning the structure of a DMG $\cG$ that may contain cycles and latent confounders, under the setting where there is no constraint on the size of each designed intervention. In Section \ref{sec:unb 0}, we introduce the observational graph estimate $\cG_r^{obs}$ and explain why learning the structure of a DMG using observational data alone is challenging in the presence of cycles and confounders. Sections \ref{sec:unb 1.1} and \ref{sec:unb 1.2} present a set of designed experiments aimed at identifying the directed edges of $\cG$. Sections \ref{sec:unb 2.1} and \ref{sec:unb 2.2} introduce an additional set of experiments for identifying the bidirected edges of $\cG$.

\subsection{Step 0: Identify $\cG^{obs}_{r}$} 
\label{sec:unb 0}
In this section, we further characterize the graphs that can be identified from observational data when confounders are considered. 

\begin{definition}[$\cG_r^{obs}$]
Suppose $\cG = (\bV,\bD,\bB)$ is a DMG. Let $\cG_r^{obs}$ denote the undirected graph
 over $\bV$ where there is an edge between $X$ and $Y$ if and only if $X$ and $Y$ are not $r$-separable
 in $\cG$, i.e., for any $\bS\subseteq\bV\backslash\{X,Y\}$ we have $(X\ncperp_rY|\bS)_{\cG}$.
\end{definition}

Find that $\cG_r^{obs}$ contains the skeleton of $\cG$, and it also contains other edges when we consider the cycle and confounder. To describe this, we give the following two definitions.

\begin{definition}[$d$-inducing path]
    Let $\cG = (\bV,\bD,\bB)$ be a DMG. An $d$-inducing path between two nodes $X,Y\in\bV$ is a path in $\cG$ between $X$ and $Y$ on which every node (except the end nodes) is a collider on the path and is in $Anc_{\cG}(\{X,Y\})$.
\end{definition}

\begin{definition}[$\sigma$-inducing path]
    Let $\cG = (\bV,\bD,\bB)$ be a DMG. An $\sigma$-inducing path between two nodes $X,Y\in\bV$ is a path in $\cG$ between $X$ and $Y$ on which every collider is in $Anc_{\cG}(\{X,Y\})$ and each non-endpoint non-collider on the path, only has outgoing directed edges to neighboring nodes on the path that lie in the same SCC of $\cG$.
\end{definition}

\begin{proposition}[\cite{mooij2020constraint}]
\label{prop:inducing path}
    Let $\cG = (\bV,\bD,\bB)$ be a DMG and $X,Y$ two distinct nodes in $\cG$. Then  $(X\ncperp_rY|\bS)_{\cG}$ for any $\bS\subseteq\bV\backslash\{X,Y\}$ if and only if there is an $r$-inducing path between $X$ and $Y$.
\end{proposition}
\begin{proof}
    The proof of $\sigma$-separation case is shown in \cite{mooij2020constraint}, and the $ d$-separation case is its straightforward extension.
\end{proof}

\begin{figure}[ht]
    \centering
    % 第一行的单个子图
    
    % 第二行的五个并列子图
    
    \begin{subfigure}[b]{0.18\textwidth}
        \includegraphics[width=\textwidth]{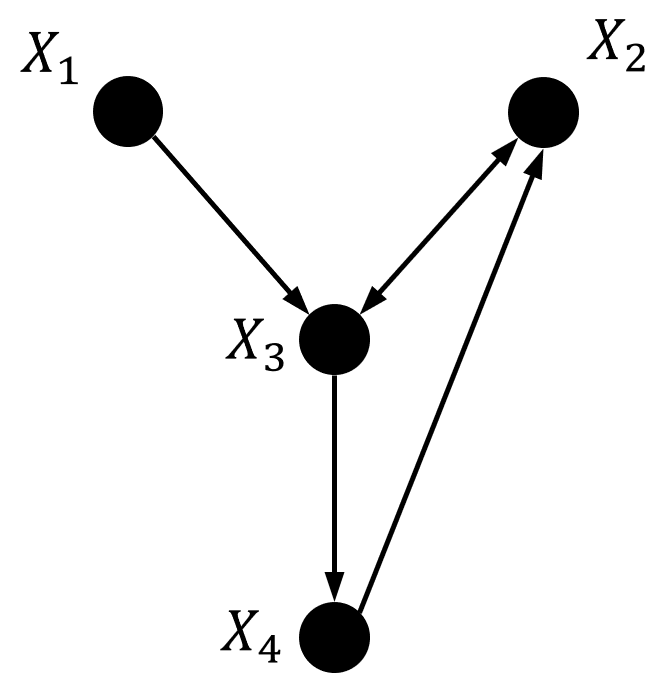}
        \caption{$\cG$}
        \label{fig:example 5.1}
    \end{subfigure}
    \hfill
    \begin{subfigure}[b]{0.18\textwidth}
        \includegraphics[width=\textwidth]{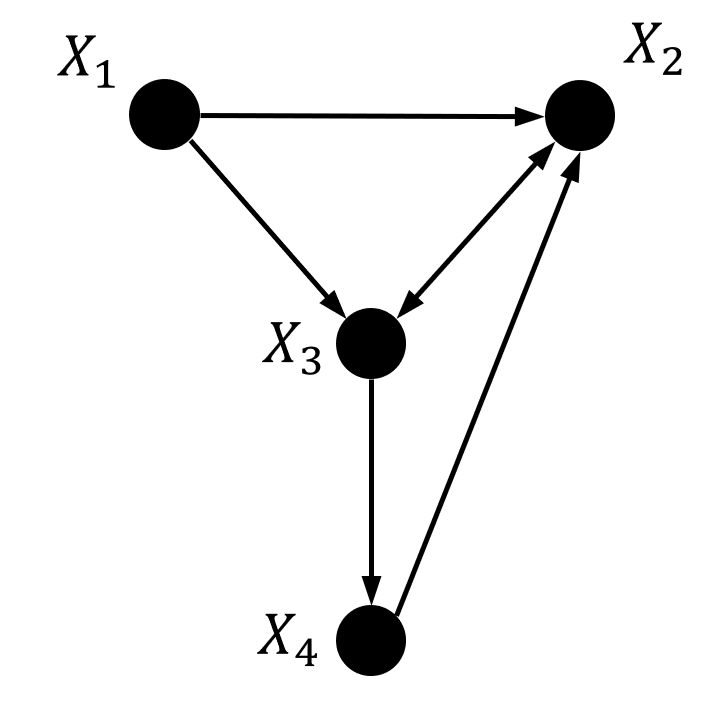}
        \caption{$\Tilde{\cG}$}
        \label{fig:example 5.2}
    \end{subfigure}
    \hfill
    \begin{subfigure}[b]{0.18\textwidth}
        \includegraphics[width=\textwidth]{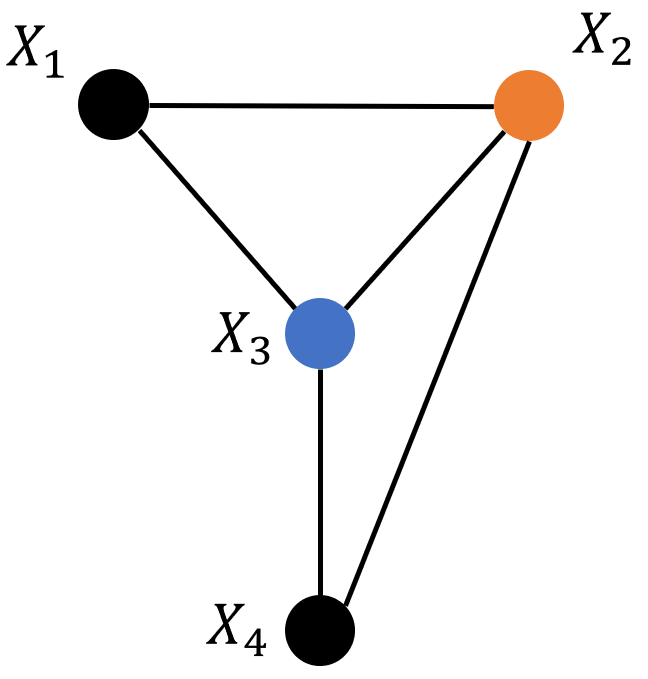}
        \caption{$\cG_r^{obs}$ and $\Tilde{\cG}_r^{obs}$}
        \label{fig:example 5.3}
    \end{subfigure}
    \hfill
    \begin{subfigure}[b]{0.25\textwidth}
        \includegraphics[width=\textwidth]{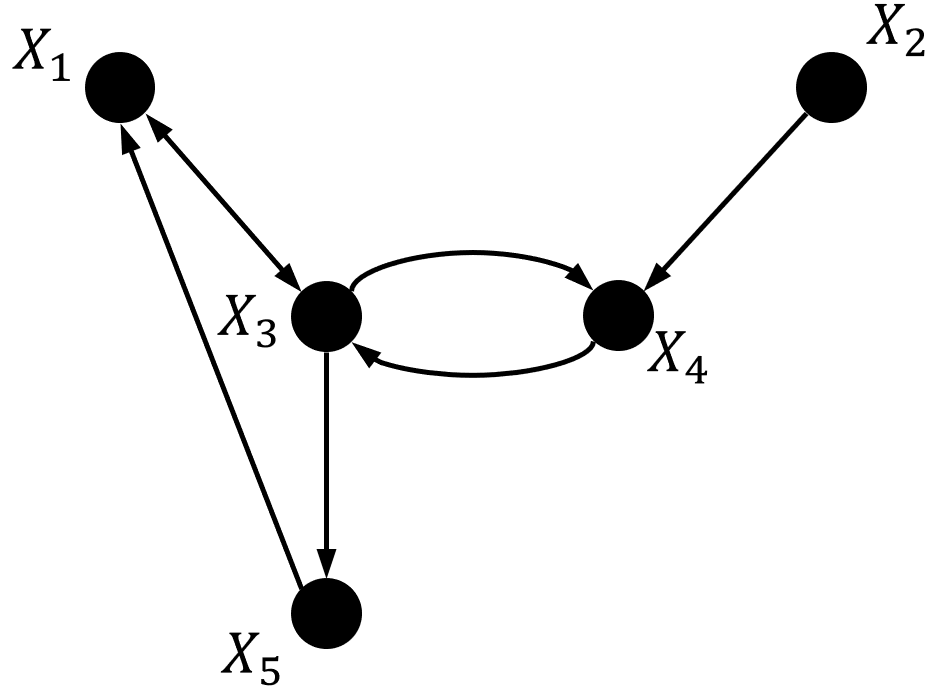}
        \caption{$\cH$}
        \label{fig:example 5.4}
    \end{subfigure}

    \begin{subfigure}[b]{0.25\textwidth}
        \includegraphics[width=\textwidth]{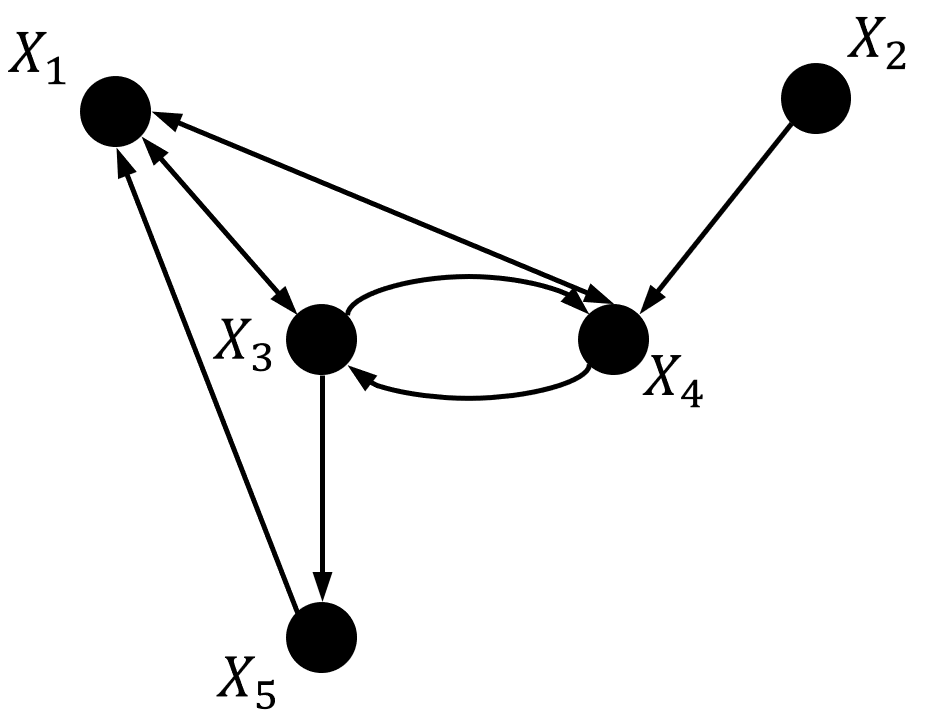}
        \caption{$\Tilde{\cH}$}
        \label{fig:example 5.5}
    \end{subfigure}
    \hfill
    \begin{subfigure}[b]{0.25\textwidth}
        \includegraphics[width=\textwidth]{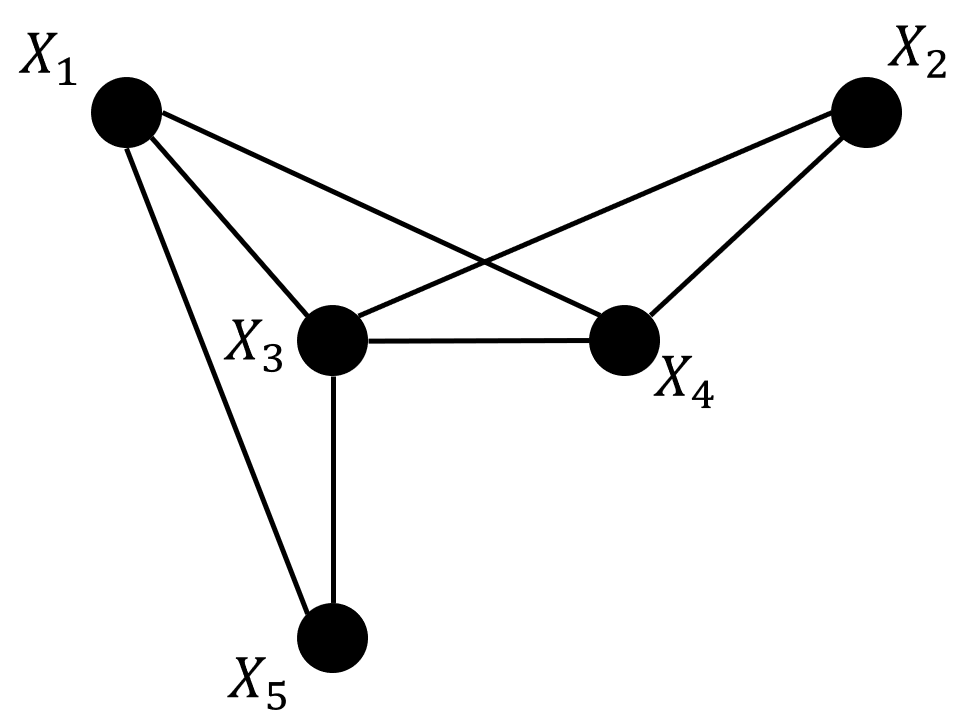}
        \caption{$\cH_d^{obs}$ and $\Tilde{\cH}_d^{obs}$}
        \label{fig:example 5.6}
    \end{subfigure}
    \hfill
    \begin{subfigure}[b]{0.25\textwidth}
        \includegraphics[width=\textwidth]{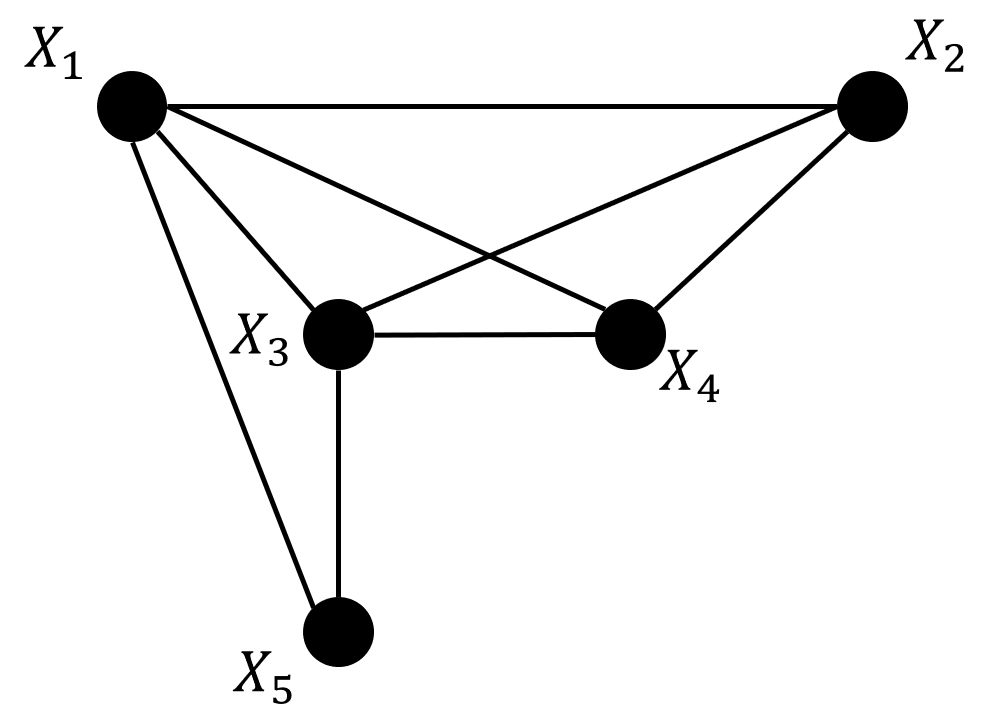}
        \caption{$\cH_\sigma^{obs}$ and $\Tilde{\cH}_\sigma^{obs}$}
        \label{fig:example 5.7}
    \end{subfigure}

    \caption{
DMGs of Examples \ref{exam: Gobs} and \ref{exam: step 1.1}. In (c), different colors denote different colors of nodes in the vertex coloring.
}
    \label{fig: Gobs}
\end{figure}

\begin{example}[Inducing path and $\cG_r^{obs}$]
\label{exam: Gobs}
    Consider the DMG $\cG$ in Figure \ref{fig:example 5.1}. We find that $X_1\rightarrow X_3 \leftrightarrow X_2$ is a $r$-inducing path between $X_1$ and $X_2$, since $X_3$ is the ancestor of $X_2$. Then we have the $\cG_r^{obs}$ is shown in Figure \ref{fig:example 5.3}. We can also find that although $\cG$ and $\Tilde{\cG}$ have different skeletons, they have the same condition-independent relationship according to observation data, i.e., $\cG_r^{obs} = \Tilde{\cG}_r^{obs}$.

    Consider the DMG $\cH$, we can find there are three inducing paths: 1) $X_1\leftrightarrow X_3 \leftarrow X_4$; 2) $X_3\rightarrow X_4 \leftarrow X_5$; 3) $X_1\leftrightarrow X_3 \leftarrow X_4 \leftarrow X_2$. The first two paths are both $d$-inducing paths and $\sigma$-inducing paths, the third path is only $\sigma$-inducing path. Then we have $\cH_d^{obs}$ and $\cH_\sigma^{obs}$ are shown in Figures \ref{fig:example 5.6} and \ref{fig:example 5.7} respectively. We can also find that although $\cH$ and $\Tilde{\cH}$ have different skeletons, they have the same condition-independent relationship according to observation data, i.e., $\cH_r^{obs} = \Tilde{\cH}_r^{obs}$.
\end{example}

Note that if there exists an edge between $X$ and $Y$, whether directed or bidirected, this edge itself forms an inducing path of length one. In this case, the dependence between $X$ and $Y$ arises directly from their adjacency in the skeleton. In the absence of confounders and cycles, all inducing paths correspond exactly to edges in the skeleton of the DMG. However, when confounders and cycles are present, inducing paths between $X$ and $Y$ may exist even without an edge between them. These paths prevent $X$ and $Y$ from being $r$-separated by any conditioning set, making it impossible to recover the true skeleton $\cG$ solely from observational data. As a result, we can at best identify $\cG_r^{obs}$, which reflects the dependencies observable under such constraints.

At step 0,  we learn $\cG_r^{obs}$ from observational data using existing methods such as the one proposed by \cite{ghassami2020characterizing}.

\subsection{Step 1.1: Identify Ancestral Relationship}
\label{sec:unb 1.1}
In this section, we introduce step 1.1 of our algorithm for learning the descendant sets $\{De_{\cG}(X)\}_{X\in\bV}$ and SCCs $\cS = \{\bS_1,\dots,\bS_s\}$ of $\cG$. This part is similar to \cite{mokhtarian2023unified}, but we consider the presence of confounders.

First, we introduce vertex coloring. Note that the terms \textit{vertex} and \textit{node} refer to the same graphical object; in the context of graph coloring, we use \textit{vertex} to follow the standard terminology in the graph theory literature \citep{west2001introduction,diestel2017extremal}.
\begin{definition}[Vertex coloring]
    A vertex coloring for an undirected graph $\cG = (\bV,\bE)$ is an assignment of colors to the vertices, such that no two adjacent vertices are of the same color. Chromatic number of $\cG$, denoted by $\chi(\cG)$, is the smallest number of colors needed for a vertex coloring of $\cG$.
\end{definition}

Based on the vertex coloring, we introduce the colored separating system.

\begin{definition}[Colored separating system]
    Suppose $\bV = \{X_1,\dots,X_n\}$ and let $\cC = \{C_1,\dots,C_n\}$  be an arbitrary coloring for $\bV$. A colored separating system $\cI$ on $(\bV,\cC)$ is a collection of subsets of $\bV$ for every distinct ordered pair of variables $(X_i,X_j)$ in $\bV$, if $C_i\neq C_j$, then there exists $\bI\in\cI$ such that $X_i\in\bI$ and $X_j\notin \bI$.
\end{definition}

\begin{proposition}[\cite{mokhtarian2023unified}]
\label{prop: colored separating}
    There exists a colored separating system on $(\bV,\cC)$ with at most $2\lceil\log_2(\chi)\rceil$ elements, where $\chi$ is the number of colors in $\cC$.
\end{proposition}

\begin{remark}
    The proof of Proposition \ref{prop: colored separating} in \cite{mokhtarian2023unified} is constructive. This allows us to obtain a colored separating system on $(\bV,\cC)$ with at most $2\lceil\log_2(\chi)\rceil$ elements.
\end{remark}

\begin{remark}
    Computing the minimum vertex coloring of an undirected graph $\cG$ is NP-hard. However, a greedy algorithm exists that can find a proper coloring using at most $ d+1$ colors in $O(n)$ time, where $n$ is the number of nodes and $d$ is the maximum degree of $\cG$.
\end{remark}

\begin{example}[Coloring separating system]
\label{exam: step 1.1}
    Consider the DMG $\cG$ in Figure \ref{fig:example 5.1}, Figure \ref{fig:example 5.3} is its vertex coloring. We can get the coloring separating system $\cI$ is $\cI = \{\{X_1,X_4\},\{X_3\},\{X_2\}\}$ and $|\cI| = 3 \leq 2\lceil \log_2(\chi(\cG_r^{obs}))\rceil = 4$.
\end{example}
Based on  Proposition \ref{prop: colored separating}, we present Algorithm \ref{alg: 1.1} for finding the descendant set and the set of SCCs in $\cG$.

\begin{algorithm}
        \caption{Learning descendant sets and strongly connected components}
	\label{alg: 1.1}
    \begin{algorithmic}
        \REQUIRE $\cG_r^{obs}$
        \ENSURE $\{De_{\cG}(X)\}_{X\in\bV}$ and SCCs $\cS = \{\bS_1,\dots,\bS_s\}$
        \STATE  $\cC\leftarrow$ a vertex coloring of $\cG_r^{obs}$\\
        \STATE $\cI \leftarrow$ a colored separating system on $(\bV,\cC)$\\
        \FOR {$X\in\bV$}
            \STATE $\cI_X = \{\bI\in\cI|X\in\bI\}$
            \STATE $\bD_X = \varnothing$
            \FOR{$\bI\in\cI_X$}
                \STATE Add the elements of $\{Y\in Ne_{\cG_r^{obs}}| (X\ncperp Y)_{P_{do(\bI)}}\}$ to $\bD_X$ \\
            \ENDFOR
        \ENDFOR
        \STATE Construct DG $\cH$ by adding directed edges from $X$ to $\bD_X$ for each $X\in\bV$ \\
        \STATE $\{De_{\cG}(X)\}_{X\in\bV}$, $\cS = \{\bS_1,\dots,\bS_s\}\leftarrow$ Compute descendant sets and SCCs of $\cH$\\
    \end{algorithmic}
\end{algorithm}

Then we prove that the output of Algorithm \ref{alg: 1.1} is actually the descendant sets and strongly connected components of $\cG$ by using the following lemmas.

\begin{lemma}
\label{lem:step 1.1 1}
    For each $X\in\bI\subseteq\bV$, $De_{\cG_{\Bar{\bI}}}(X) = \{Y\in\bV|(X\ncperp Y)_{P_{do(\bI)}}\}$.
\end{lemma}
\begin{proof}
    We first prove that $De_{\cG_{\Bar{\bI}}}(X) = \{Y\in\bV|(X\ncperp_r Y)_{\cG_{\Bar{\bI}}}\}$.
    \begin{itemize}
        \item If $Y\in De_{\cG_{\Bar{\bI}}}(X)$, there exist a  directed path from $X$ to $Y$ in $\cG_{\Bar{\bI}}$ and therefore, $(X\ncperp_r Y)_{\cG_{\Bar{\bI}}}$.
        \item Suppose $(X\ncperp_r Y)_{\cG_{\Bar{\bI}}}$. Since $X$ has no parents in $\cG_{\Bar{\bI}}$, any path between $X$ and $Y$ must be oriented outward from $X$. There are two possible cases.  
        
        \textbf{Case I}: There exists a path $\cP$ between $X$ and $Y$ in $\cG_{\Bar{\bI}}$ with no colliders. Since any path between $X$ and $Y$ must be oriented outward from $X$,  $\cP$ must be a directed path from $X$ to $Y$, and it implies $Y\in De_{\cG_{\Bar{\bI}}}(X)$.

        \textbf{Case II}: There exists a path $\cP$ between $X$ and $Y$ in $\cG_{\Bar{\bI}}$ whose every collider is in $Anc_{\cG_{\Bar{\bI}}}(\{X,Y\})$. Let $Z$ be the collider closest to $X$ along the path. Then $Z\in De_{\cG_{\Bar{\bI}}}(X)$. Moreover, since $X$ has no ancestors in $\cG_{\Bar{\bI}}$, $Z \in Anc_{\cG_{\Bar{\bI}}}(Y)$, implying $Y\in De_{\cG_{\Bar{\bI}}}(X)$.
    \end{itemize}
    Since under Scenario 1, $\{Y\in\bV|(X\ncperp_d Y)_{\cG_{\Bar{\bI}}}\} = \{Y\in\bV|(X\ncperp Y)_{P_{do(\bI)}}\}$ and under Scenario 2, $\{Y\in\bV|(X\ncperp_\sigma Y)_{\cG_{\Bar{\bI}}}\} = \{Y\in\bV|(X\ncperp Y)_{P_{do(\bI)}}\}$, we have $De_{\cG_{\Bar{\bI}}}(X) = \{Y\in\bV|(X\ncperp Y)_{P_{do(\bI)}}\}$.
\end{proof}

Lemma \ref{lem:step 1.1 1} shows that  for any set $\bI\subseteq\bV$ and each $X\in\bI$, $De_{\cG_{\Bar{\bI}}}(X)$ is learned by performing an experiment on $\bI$. Therefore, at the end of the loop in Algorithm \ref{alg: 1.1}, we have 
\begin{equation}
\label{eq: DX}
    \bD_X = \left(\cup_{\bI\in\cI_X}De_{\cG_{\Bar{\bI}}}(X) \right) \cap Ne_{\cG_r^{obs}}(X).
\end{equation}

\begin{lemma}
    \label{lem:step 1.1 2}
    For each $X\in\bV$, $Ch_{\cG}(X)\subseteq\bD_X\subseteq De_\cG(X)$ where $\bD_X$ is defined in Equation (\ref{eq: DX}).
\end{lemma}
\begin{proof}
    Since $De_{\cG_{\Bar{\bI}}}(X)\subseteq De_\cG(X)$ for any $\bI\subseteq\bV$, we have $\bD_X\subseteq De_\cG(X)$.

    Suppose $Y\in Ch_{\cG}(X)$, we next show $Y\in \bD_X$. Because $X$ and $Y$ are adjance in $\cG_r^{obs}$, they have different colors in $\cC$. So there exists a $\bI\in\cI$ such that $X\in\bI$ and $Y\notin\bI$. In this case, $\bI \in\cI_X$ since $X\in\bI$. Furthermore, $Y\in De_{\cG_{\Bar{\bI}}}(X)$ since $Y$ is still a children of $X$ in $\cG_{\Bar{\bI}}$. Then we have $Y\in \bD_X$, which implies $Ch_{\cG}(X)\subseteq\bD_X$.
\end{proof}

After learning $\bD_X$ for all $X\in\bV$, a DG $\cH$ is constructed over $\bV$ by adding directed edges
 from $X$ to the variables in $\bD_X$ for each $X\in\bV$ in Algorithm \ref{alg: 1.1}. We can find that $\cH$ is a supergraph of $\cRB(\cG)$, where the extra edges in $\cH$ appear only from
 the variables to some of their descendants in $\cRB(\cG)$. We have the following corollary of Lemma \ref{lem:step 1.1 1} and Lemma \ref{lem:step 1.1 2}. 

 \begin{corollary}
     \label{coro: step 1.1}
     In Algorithm \ref{alg: 1.1}, DMG $\cG$ and DG $\cH$ have the same descendant sets, i.e., for each $X\in\bV$, $De_{\cG}(X) = De_\cH(X)$.Furthermore, $\cG$ and $\cH$ have the same SCCs.
 \end{corollary}

Note that in Corollary \ref{coro: step 1.1}, $\cG$ and $\cH$ have the same SCCs because that by definition, two variables
 $X$ and $Y$ are in the same SCC in $\cG$ if and only if $X\in De_\cG(Y)$ and  $Y\in De_\cG(X)$.

 So far, we have proved that we learn the descendant sets and SCCs of a DMG $\cG$ using Algorithm \ref{alg: 1.1}. 

\subsection{Step 1.2: Identify Directed Edges}
\label{sec:unb 1.2}
As we discussed in the last section, we learn the descendant sets and SCCs of a DMG $\cG$ via performing $2\lceil\log_2(\chi(\cG_r^{obs})\rceil$ experiments. Then we can get the SCC-Anc partition of $\cG$, $\bbT^{\cG} = \{\cT_1^\cG,\dots,\cT_{l+1}^\cG\}$ according to Definition \ref{def: scc-anc partition}. In this section, we design $\sum_{k=1}^{l+1}\zeta_{max}^{k,\cG}$ new experiments to learn $\cRB(\cG)$ where $\zeta_{max}^{k,\cG}$ is the size of largest SCC in $\cT_k^\cG$. In this step,  we perform experiments on certain subsets of $\bV$ that form an SCC-Anc separating system, defined as follows.

\begin{definition}[SCC-Anc separating system]
\label{def: scc-anc separating system}
Let $\bbT^{\cG} = \{\cT_1^\cG, \dots, \cT_{l+1}^\cG\}$ be the SCC-Anc partition of a DMG $\cG = (\bV, \bD, \bB)$, where each $\cT_k^\cG = \{\bS_{k,1}, \dots, \bS_{k,n_k}\}$ is the set of SCCs at level $k$.
An SCC-Anc separating system $\cI$ on $(\bV, \bbT^{\cG})$ is a collection of subsets of $\bV$ satisfying the following condition:

For every level $k \in \{1, \dots, l+1\}$, every SCC $\bS_{k,j} \in \cT_k^\cG$ (with $j \in {1, \dots, n_k}$), and every node $X \in \bS_{k,j}$, there exists an intervention set $\bI \in \cI$ such that:
\begin{equation*}
    \rmT_k^{\cG}\cup\bS_{k,j}\backslash\{X\} \subseteq\bI \text{  and  } X\notin\bI, 
\end{equation*}
where $\rmT_k^\cG = \cup_{j=1}^{k-1}\cT_j^\cG$ for $2\leq k \leq l+1$ and $\rmT_1^\cG = \varnothing$.
%Suppose  $\bbT^{\cG} = \{\cT_1^\cG,\dots,\cT_{l+1}^\cG\}$ is the  SCC-Anc partition of DMG $\cG = (\bV,\bD,\bB)$ where $\cT_k^\cG = \{\bS_{k,1},\dots,\bS_{k,n_k}\}$. An SCC-Anc separating system $\cI$ on $(\bV,\bbT^{\cG})$ is a collection of subsets of $\bV$ such that for each $k\in\{1,2,\dots,l+1\}$, $j\in\{1,2,\dots,n_k\}$, $\bS_{k,j}\in\cT_k^\cG$ and $X\in\bS_{k,j}$, there exist $\bI\in\cI$ such that $\rmT_k^{\cG}\cup\bS_{k,j}\backslash\{X\} \subseteq\bI$ and $X\notin\bI$ where $\rmT_k^\cG = \cup_{j=1}^{k-1}\cT_j^\cG$ for $2\leq k \leq l+1$ and $\rmT_1 = \varnothing$.
\end{definition}

Then we provide a design method to construct the SCC-Anc separating system with at most $\sum_{k=1}^{l+1}\zeta_{max}^{k,\cG}$ elements.

\begin{proposition}
    \label{prop: step 1.2 SCC-Anc separating}
     Suppose  $\bbT^{\cG} = \{\cT_1^\cG,\dots,\cT_{l+1}^\cG\}$ is the  SCC-Anc partition of DMG $\cG = (\bV,\bD,\bB)$. There exists an SCC-Anc separating system with at most $\sum_{k=1}^{l+1}\zeta_{max}^{k,\cG}$ elements where $\zeta_{max}^{k,\cG}$ is the size of largest SCC in $\cT_k^\cG$.
\end{proposition}
\begin{proof}
    For each $1\leq k\leq l+1$, suppose $\cT_k^\cG = \{\bS_{k,1},\dots,\bS_{k,n_k}\}$ and $\bS_{k,j} = \{X_{k,j}^{1},\dots,X_{k,j}^{m_{k,j}}\}$ for $1\leq j\leq n_k$, where $m_{k,j} = |\bS_{k,j}|$. Let $m_{k,max} = \max\{m_{k,1},\dots,m_{k,n_k}\} = \zeta_{max}^{k,\cG}$. For each $1\leq k\leq l+1$ and $1\leq i\leq m_{k,max}$, we construct $\bI_{k,i}\subseteq\bV$ as follows,
    \begin{equation}
    \label{eq: SCC-Anc separating constract}
        \bI_{k,i} = \rmT^\cG_k \cup\left(\cup_{\textstyle\substack{1\leq j\leq n_k \\ \text{s.t.  }i\leq m_{k,j}}}\bS_{k,j}\backslash\{X_{k,j}^i\} \right).
    \end{equation}
    Then we show that $\cI = \{\bI_{1,1},\dots,\bI_{1,m_{1,\max}}, \dots,\bI_{l+1,1},\dots,\bI_{l+1,m_{l+1,max}}\}$ is an SCC-Anc separating system on $(\bV,\bbT^\cG)$. Note that $|\cI| = \sum_{k=1}^{l+1}\zeta_{max}^{k,\cG}$. Suppose $k\in\{1,2,\dots,l+1\}$, $j\in\{1,2,\dots,n_k\}$, $X^i_{k,j}\in\bS_{k,j}$ where $1 \leq i\leq m_{k,j}$. We need to show that there exists $\bI\in\cI$ such that $\rmT_k^{\cG}\cup\bS_{k,j}\backslash\{X^i_{k,j}\} \subseteq\bI$ and $X^i_{k,j}\notin\bI$. $\bI = \bI_{k,i}$  satisfies this property because $\bI_{k,i}\cap \left(\rmT_k^{\cG}\cup\bS_{k,j} \right) =\rmT_k^{\cG}\cup\bS_{k,j}\backslash\{X^i_{k,j}\} $. Hence $\cI$ is an SCC-Anc separating system on $(\bV,\bbT^\cG)$ with size $\sum_{k=1}^{l+1}\zeta_{max}^{k,\cG}$. 
\end{proof}

\begin{remark}
    The proof of Proposition \ref{prop: step 1.2 SCC-Anc separating} is constructive. Given the SCC-Anc partition $\bbT^\cG$, we can use Equation (\ref{eq: SCC-Anc separating constract}) to design an SCC-Anc separating system with as most $\sum_{k=1}^{l+1}\zeta_{max}^{k,\cG}$ elements.
\end{remark}

\begin{example}[SCC-Anc separating system]
\label{exam: SCC-Anc separating system}
    Consider the DMG $\cG$ in Figure \ref{fig:example SCC-Anc patition}, a corresponding SCC-Anc separating system $\cI$ is $\cI = \{\{X_1\},\{X_2\}, \{X_1,X_2\}, \{X_1,X_2,X_3,X_6,X_7\},$ $\{X_1,X_2,X_3,X_5,X_7\},\{X_1,X_2,X_3,X_5,X_6\}\}$. In this case, we have $|\cI| = \sum_{k=1}^{l+1}\zeta_{max}^{k,\cG}= 6$ where $l=2$ is  the SCC-Anc length of $\cG$.
\end{example}

We present Algorithm \ref{alg: 1.2} to learn the directed part of DMG $\cG$ that takes the SCC-Anc partition $\bbT^\cG$ as input. The algorithm constructs an SCC-Anc separating system and initializes a DG $\cRB(\hat{\cG})$ with no edge at first.

\begin{algorithm}
        \caption{Learning the directed edges $\cRB(\cG)$}
	\label{alg: 1.2}
    \begin{algorithmic}
        \REQUIRE $\bbT^\cG =\{\cT_1^\cG,\dots,\cT_{l+1}^\cG\}$ where $\cT_k^\cG = \{\bS_{k,1},\dots,\bS_{k,n_k}\}$ for $1\leq k\leq l+1$ 
        \ENSURE $\cRB(\hat{\cG})$
        \STATE $\cI \leftarrow$ an SCC-Anc separating system on $(\bV,\bbT^\cG)$\\
        \STATE Initialization: $\cRB(\hat{\cG}) = (\bV, \hat{\bD} = \varnothing)$
        \FOR{$k$ from $1$ to $l+1$}
            \FOR{$j$ from $1$ to $n_k$}
                \FOR{$X\in\bS_{k,j}$}
                    \STATE $\bI\leftarrow$ an element of $\cI$ that contains $\rmT_k^{\cG}\cup\bS_{k,j}\setminus \{X\}$ but does not contain $X$ \\
                    \FOR{$Y\in \rmT_k^{\cG}\cup\bS_{k,j}\backslash\{X\}$}
                        \STATE Add $(Y,X)$ to $\hat{\bD}$ if $(X\ncperp Y)_{P_{do(\bI)}}$\\
                    \ENDFOR
                \ENDFOR
            \ENDFOR
        \ENDFOR
    \end{algorithmic}
\end{algorithm}

Suppose $X$ is an arbitrary variable in $\bS_{k,j}$ and $\bS_{k,j}$ is an arbitrary SCC in $\cT^\cG_k$.  Since $\cI$ is an  SCC-Anc separating system on $(\bV, \bbT^\cG)$, there exists $\bI\in\cI$ that contain $\rmT_k^{\cG}\cup\bS_{k,j}\setminus \{X\}$ but does not contain $X$. According to the following lemma, we find that performing an experiment on $\bI$ can learn the parents of $X$.

\begin{lemma}
    \label{lem: step 1.2}
    For each $X \in \bS_{k,j} \in \cT_k^{\cG}$ for $1\leq k\leq l+1, 1\leq j\leq n_k$, if $\rmT_k^{\cG}\cup\bS_{k,j}\backslash\{X\} \subseteq\bI$ and $X\notin\bI$, we have $Pa_\cG(X) = \{Y\in\bI|(X\ncperp Y)_{P_{do(\bI)}} \}$.
\end{lemma}
\begin{proof}
     We first prove that $Pa_\cG(X) = \{Y\in\bI|(X\ncperp_r Y)_{\cG_{\Bar{\bI}}}\}$.
    \begin{itemize}
        \item Suppose $Y\in Pa_{\cG}(X)$. Since $X\notin \bI$, there is a directed edge from $Y$ to $X$ in $\cG_{\Bar{\bI}}$, which implies $(X\ncperp_r Y)_{\cG_{\Bar{\bI}}}$. Since $Y\in Pa_{\cG}(X) \subseteq \rmT_k^{\cG}\cup\bS_{k,j}\backslash\{X\}$, we also have $Y \in \bI$.
        \item Suppose $Y\in\{Y\in\bI|(X\ncperp Y)_{P_{do(\bI)}} \}$. Since $Y\in\bI$, according to Lemma \ref{lem:step 1.1 1}, we have $X\in De_{\cG_{\Bar{\bI}}}(Y)$. To prove $Y \in Pa_{\cG_{\Bar{\bI}}}(X)$, we only need to show that every path between $X$ and $Y$ with length larger than 1, i.e., $\cP = (X,E_1,Z_1,E_2,Z_2,\dots,Z_t,E_{t+1},Y)$, $t\geq 1$, is $r$-blocked. We next prove this in two cases.

        \textbf{Case I}: Suppose $E_1 = (X,Z_1)$ or $[X,Z_1]$. Since $Y\in \bI$, we have $E_{t+1} = (Y,Z_t)$, then there exist colliders in path $\cP$. We denote the closest collider to $X$ as $Z_j$. Since $Y\in\bI$, we have $Anc_{\cG_{\Bar{\bI}}}(Y) = \varnothing$ which implies $Z_j\notin Anc_{\cG_{\Bar{\bI}}}(Y)$. Since there are incoming edges to $Z_j$, $Z_j\notin\bI$. According to the Definition of SCC-Anc partition, we have $Pa_{\cG}(X)\subseteq\rmT_k^{\cG}\cup\bS_{k,j}\backslash\{X\}\subseteq\bI$, which implies $Z_j\notin Anc_{\cG_{\Bar{\bI}}}(X)$. Then we have $Z_j$ is a collider on path $\cP$ and $Z_j \notin Anc_{\cG_{\Bar{\bI}}}(\{X,Y\})$, so $\cP$ is $r$-blocked in $\cG_{\Bar{\bI}}$.

        \textbf{Case II}: Suppose $E_1 = (Z_1,X)$. Then $Z_1 \in Pa_{\cG}(X) \subseteq \bI$ according to the Definition of SCC-Anc separating system. Then we have $Z_2 \neq Y$ and $E_2 = (Z_1,Z_2)$. The rest of the proof is similar to Case I. Since $Y\in \bI$, we have $E_{t+1} = (Y,Z_t)$, then there exist colliders in path $\cP$. We denote the closest collider to $X$ as $Z_j$. Same as Case I, we have  $Z_j \notin Anc_{\cG_{\Bar{\bI}}}(\{X,Y\})$, so $\cP$ is $r$-blocked in $\cG_{\Bar{\bI}}$.

        Combine two cases, we have $Y \in Pa_{\cG_{\Bar{\bI}}}(X)$. Since $X\notin \bI$, $ Pa_{\cG_{\Bar{\bI}}}(X) = Pa_{\cG}(X)$ which implies $Y\in  Pa_{\cG}(X)$.
    \end{itemize}
    Under Scenario 1, we have $\{Y\in\bI|(X\ncperp_d Y)_{\cG_{\Bar{\bI}}}\} = \{Y\in\bI|(X\ncperp Y)_{P_{do(\bI)}}\}$ and under Scenario 2, we have $\{Y\in\bI|(X\ncperp_\sigma Y)_{\cG_{\Bar{\bI}}}\} = \{Y\in\bI|(X\ncperp Y)_{P_{do(\bI)}}\}$. So we have $Pa_\cG(X) =\{Y\in\bI|(X\ncperp Y)_{P_{do(\bI)}}\}$.
    
\end{proof}

Lemma \ref{lem: step 1.2} states that by selecting experiments from an SCC-Anc separating system, we can accurately identify the set of parents for each node. As a result, all directed edges can be recovered, and the full directed causal structure can be determined, i.e., $\cRB(\hat{\cG}) = \cRB(\cG)$. 

\begin{remark}
    According to the definition of the SCC-Anc separating system, it essentially allows us to identify all possible parents of a node $X\in\bS_{k,j}\in\cT_k^\cG$, i.e, $Pa_{\cG}(X)\subseteq \rmT_k^{\cG}\cup\bS_{k,j}\backslash\{X\}$; that is, by intervening on these nodes, one can accurately recover the $Pa_{\cG}(X)$. Combined with the result of Lemma \ref{lem:lower D}, it shows that in the worst-case scenario, failing to intervene on all parents of $X$ makes it impossible to fully identify its parent set. In this sense, Lemmas \ref{lem: step 1.2} and \ref{lem:lower D} together establish that, under the presence of cycles and confounders, a necessary and sufficient condition for accurately identifying the parents of $X$ is to intervene on all its possible parents.
\end{remark}

 In Section \ref{sec:unb 1.1}, we showed that the descendant sets and SCCs of a DMG $\cG$ can be learned
 by performing experiments on the elements of a colored separating system. Herein, we
 showed that using the information about the descendant sets and SCCs, $\cRB(\cG)$ can be recovered
 by performing experiments on the elements of an SCC-Anc separating system. Moreover, we
 provided Propositions \ref{prop: colored separating} and \ref{prop: step 1.2 SCC-Anc separating} for constructing separating systems and lifted separating
 systems, respectively, which implies the following.

 \begin{corollary}
 \label{coro: Upper unbounded D cI}
     Algorithm \ref{alg: 1.1} and \ref{alg: 1.2} together learn the directed part of DMG $\cG$, i.e., $\cRB(\cG)$ with at most 
     \begin{equation}
         2\lceil\log_2(\chi(\cG_r^{obs}))\rceil + \sum_{k=1}^{l+1}\zeta_{max}^{k,\cG}
     \end{equation}
     elements. Compared with the lower bound in Theorem \ref{thm:lower D cI}, the proposed approach
 is order-optimal in terms of the number of experiments up to an additive logarithmic term.
 \end{corollary}

\subsection{Step 2.1: Identify Non-adjacent Bidirected Edges}
\label{sec:unb 2.1}
We have identified the directed part of DMG $\cG = (\bV,\bD,\bB^N\cup\bB^A)$ in the last two sections. In the following two steps, we will identify the bidirected edges of $\cG$ using the $\cRB(\cG)$. In this section, we focus on the non-adjacent bidirected edges, $\bB^N$. We perform experiments on certain subsets of $\bV$ that form a non-adjacent separating system defined as follows. 

\begin{definition} [Non-adjacent separating system]
\label{def: non-adjacent separating system}
Suppose $\cRB(\cG) = (\bV,\bD)$ is the directed part of DMG $\cG = (\bV,\bD,\bB^N\cup\bB^A)$. 
A non-adjacent separating system $\cI$ on $(\bV, \bD)$ is a collection of subsets of $\bV$ satisfying the following condition:

For every pair of nodes $X, Y \in \bV$ such that $(X, Y) \notin \bD$ and $(Y, X) \notin \bD$, there exists an intervention set $\bI \in \cI$ such that:
\begin{equation*}
    Pa_{\cG}(\{X,Y\})\subseteq \bI \text{  and  }X\notin\bI,Y\notin\bI.
\end{equation*}
%A non-adjacent separating system $\cI$ on $(\bV,\bD)$ is a collection of subsets of $\bV$ such that for each pair of nodes $X,Y\in\bV$ which satisfies $(X,Y)\notin\bD$ and $(Y,X)\notin \bD$, there exist $\bI\in\cI$ such that $Pa_{\cG}(\{X,Y\})\subseteq \bI$ and $X\notin\bI,Y\notin\bI$.
\end{definition}

Subsequently, we present a theoretically grounded approach for constructing a non-adjacent separating system with at most $cc(\cG^{uc})$ elements where $cc(\cG^{uc})$ is the minimum edge clique covering number of the undirected component graph of $\cG$. 

\begin{proposition}
    \label{prop: non-ad separating cc}
    Suppose $\cRB(\cG) = (\bV,\bD)$ is the directed part of DMG $\cG = (\bV,\bD,\bB^N\cup\bB^A)$. There exists a non-adjacent separating system with at most $cc(\cG^{uc})$ elements.
\end{proposition}
\begin{proof}
    Suppose $\cG^{uc}$ is the undirected component graph of $\cG$. According to Definition \ref{def: Guc}, $\cG^{uc}$ is determined by the directed part of $\cG$. Suppose $\{\bC_1,\dots,\bC_{cc(\cG^{uc})}\}$ is the minimum edge clique covering of $\cG^{uc}$. For $1\leq k\leq cc(\cG^{uc})$ ,we construct $\bI_{k}\subseteq\bV$ as follows, 
    \begin{equation*}
        \bI_k = Pa_{\cG}(\bC_k).
    \end{equation*}
    Then we show that $\cI = \{\bI_1,\dots,\bI_{cc(\cG^{uc})}\}$ is a non-adjacent separating system on $(\bV,\bD)$. Note that $|\cI| = cc(\cG^{uc})$. Suppose $X,Y\in\bV$ and $(X,Y)\notin\bD,(Y,X)\notin\bD$. Then $X,Y$ belong to the same clique, which we denote by $\bC_k$. We need to show that there exists $\bI\in\cI$ such that $Pa_{\cG}(\{X,Y\})\subseteq \bI$ and $X\notin\bI,Y\notin\bI$. $\bI = \bI_k$ satisfies this property because $\bI_k\cap Pa_{\cG}(\{X,Y\}) = Pa_{\cG}(\{X,Y\})$ and $X,Y\notin\bI_k$ since $X,Y$ are not parents of node in $\bC_k$.
\end{proof}

\begin{example}[Non-adjacent separating system]
    Consider the DMG $\cG$ shown in Figure \ref{fig:example 4.1}, we have shown in Example \ref{exam: edge clique covering} that Figure \ref{fig:example 4.2} is the minimum edge clique covering of $\cG^{uc}$. Then we can use it to construct the non-adjacent separating system $\cI$ similar to the Proposition \ref{prop: non-ad separating cc}, i.e., $\cI = \{\{X_3,X_4\},\{X_5\},\{X_1,X_3,X_4\},\{X_1,X_5\}\}$ and $|\cI| = cc(\cG^{uc}) = 4$.
\end{example}

Although the proof of Proposition \ref{prop: non-ad separating cc} is constructive, constructing a non-adjacent separating system requires finding a minimum edge clique covering, which is NP-hard. Therefore, in practice, we recommend using an approximate method proposed in \cite{NIPS2017_291d43c6}, which adopts a randomized approach.

\begin{proposition}[\cite{NIPS2017_291d43c6}]
    \label{prop: non-ad separating random}
    For a DMG $\cG = (\bV,\bD,\bB)$,  there exists a randomized algorithm that returns a non-adjacent separating system on $(\bV,\bD)$ with $\lfloor 4e^2(d+1)^2\ln(n) \rfloor$ elements with probability at least $1-\frac{1}{n^2}$. $n$ is the number of nodes and $d$ is the max degree of $\cG^{u}$.
\end{proposition}

\begin{remark}
The randomized construction underlying Proposition~\ref{prop: non-ad separating random} follows the proof of Lemma~6 in \citet{NIPS2017_291d43c6}. Let $d$ be the maximum degree of $\cG^u$ and set $m=\lfloor 4e^2(d+1)^2\ln(n)\rfloor$. For each $q\in\{1,\dots,m\}$, independently include every variable in a temporary set $\bS_q$ with probability $1/(d+1)$. We then prune $\bS_q$ by removing every variable that has a non-neighbor within $\bS_q$ in $\cG^{uc}$, producing a clique $\bC_q$ of $\cG^{uc}$. Equivalently, $\bC_q$ is an independent set in the directed skeleton graph $\cG^u$, so no two variables in $\bC_q$ are connected by a directed edge. For each resulting clique, construct an intervention set $\bI_q=Pa_{\cG}(\bC_q)$, and let $\cI=\{\bI_1,\dots,\bI_m\}$. With probability at least $1-1/n^2$, the cliques $\{\bC_q\}_{q=1}^m$ cover all edges of $\cG^{uc}$, in which case $\cI$ is a valid non-adjacent separating system. Even when this covering event fails, the resulting experiments can still identify the non-adjacent bidirected edges between pairs of variables covered by at least one sampled clique, although recovery of all such edges is no longer guaranteed.
\end{remark}

We present Algorithm \ref{alg: 2.1} to learn the non-adjacent bidirected edges of DMG $\cG$ that takes the $\cRB(\cG)$ as input. The algorithm constructs a non-adjacent separating system and initializes a set of non-adjacent bidirected edges as an empty set at first.

\begin{algorithm}
        \caption{Learning the non-adjacnet bidirected edges $\bB^N$}
	\label{alg: 2.1}
    \begin{algorithmic}
        \REQUIRE $\cRB(\cG) = (\bV,\bD)$
        \ENSURE $\hat{\bB^N}$
        \STATE $\cI \leftarrow$ an non-adjacent separating system on $(\bV,\bD)$\\
        \STATE Initialization: $\hat{\bB^N} = \varnothing$
        \FOR{ $[X,Y]\in \{[S,R]| (S,R)\notin\bD, (R,S)\notin\bD\}$}
        \STATE $\bI\leftarrow$ an element of $\cI$ that contain $Pa_{\cG}(\{X,Y\})$ but does not contain $X,Y$\\
        \STATE Add $[X,Y]$ to $\hat{\bB^N}$ if $(X\ncperp Y|Pa_{\cG}(\{X,Y\}))_{P_{do(\bI)}}$
        \ENDFOR
    \end{algorithmic}
\end{algorithm}

Suppose $X,Y$ are two arbitrary variables in $\bV$ which satisfy $(X,Y)\notin\bD, (Y,X)\notin\bD$. Since $\cI$ is a non-adjacent separating system on $(\bV,\bD)$, there exist $\bI\in\cI$ that contain $Pa_{\cG}(\{X,Y\})$ but does not contain $X,Y$. According to the following lemma, we find that performing the experiment on $\bI$ can justify whether $[X,Y]\in\bB^N$.

\begin{lemma}
    \label{lem: step 2.1}
    Suppose $\cG = (\bV,\bD,\bB^N\cup\bB^A)$ is a DMG. For each two variables $X,Y\in\bV$ satisfy $(X,Y)\notin\bD, (Y,X)\notin\bD$, if $Pa_\cG(\{X,Y\}) \subseteq\bI$, $X\notin\bI,Y\notin\bI$, then $[X,Y]\in\bB^N$ if and only if $(X\ncperp Y|Pa_{\cG}(\{X,Y\}))_{P_{do(\bI)}}$.
\end{lemma}
\begin{proof}
    We first prove $[X,Y]\in\bB^N$ if and only if $(X\ncperp_r Y|Pa_{\cG}(\{X,Y\}))_{\cG_{\Bar{\bI}}}$.
    \begin{itemize}
        \item If $[X,Y]\in\bB^N$, there exist a path $\cP = (X,[X,Y],Y)$ in $\cG_{\Bar{\bI}}$ since $X,Y\notin \bI$. And this path is not $r$-blocked by $Pa_{\cG}(\{X,Y\})$ in $\cG_{\Bar{\bI}}$, so $(X\ncperp_r Y|Pa_{\cG}(\{X,Y\}))_{\cG_{\Bar{\bI}}}$.
        \item Suppose $(X\ncperp Y|Pa_{\cG}(\{X,Y\}))_{P_{do(\bI)}}$. To prove $[X,Y]\in\bB^N$, we only need to prove that if $[X,Y]\notin \bB^n$, $(X\cperp Y|Pa_{\cG}(\{X,Y\}))_{P_{do(\bI)}}$, which is equivalent to proving that for any path $\cP = (X,E_1,Z_1,\dots,Z_t,E_{t+1},Y)$ between $X$ and $Y$ is $r$-bloced by $Pa_{\cG}(\{X,Y\})$ in $\cG_{\Bar{\bI}}$. Since $(X,Y)\notin \bD,(Y,X)\notin\bD$ and $[X,Y]\notin\bB$, we have $t \geq 1$. We next prove this in two cases.

        \textbf{Case I} : Suppose $E_1 = (X,Z_1)$ or $E_1 = [X,Z_1]$. Since $Pa_G(Y)\subseteq\bI$, $Z_1 \notin Pa_G(Y)$. Furthermore, either $E_{t+1} = [Z_t,Y]$, $E_{t+1} = (Y,Z_t)$ or $E_{t+1} = (Z_t,Y)$ and $E_t = (Z_t, Z_{t-1})$, for all the three scenarios, there exist a collider in $\cP$ and we denote the cloest one to $X$ as $Z_j$. Since there are incoming edges to $Z_j$, we have $Z_j \notin \bI$ and $Z_j \notin Pa_{\cG}(\{X,Y\})$ because $Pa_{\cG}(\{X,Y\})\subseteq\bI$. Besides, $Z_j\notin Anc_\cG(\{X,Y\})$ since $Pa_{\cG}(\{X,Y\})\subseteq\bI$. Thus $Z_j$ is a collider on $\cP$ in $\cG_{\Bar{\bI}}$ which is not the ancestor of $\{X,Y,Pa_{\cG_{\Bar{\bI}}}(\{X,Y\})\}$, we have $\cP$ is $r$-bloced in $\cG_{\Bar{\bI}}$

        \textbf{Case II}: Suppose $E_1 = (Z_1,X)$. We have $Z_1$ is the parent of $X$, so $Z_i\in\bI$ which implies $E_2 = (Z_1,Z_2)$. If $Z_2 = Y$, we have $\cP = (X,(Z_1,X),Z_1,(Z_1,Y),Y)$. In this scenario, $Z_1\in Pa_{\cG}(\{X,Y\})$ and $Z_1,X,Y$ are not in the same SCC. These implies that $\cP$ if $r$-blocked by $Z_1$. If $Z_2 \neq Y$, the rest proof is similar to Case I. There exists a collider in $\cP$ and we denote the closest one to $X$ as $Z_j$. Since there are incoming edges to $Z_j$, we have $Z_j \notin \bI$ and $Z_j \notin Pa_{\cG}(\{X,Y\})$ because $Pa_{\cG}(\{X,Y\})\subseteq\bI$. Besides, $Z_j\notin Anc_\cG(\{X,Y\})$ since $Pa_{\cG}(\{X,Y\})\subseteq\bI$. Thus $Z_j$ is a collider on $\cP$ in $\cG_{\Bar{\bI}}$ which is not the ancestor of $\{X,Y,Pa_{\cG_{\Bar{\bI}}}(\{X,Y\})\}$, we have $\cP$ is $r$-bloced  in $\cG_{\Bar{\bI}}$
        \end{itemize}

        Under Scenario 1, we have $(X\ncperp_d Y|Pa_{\cG}(\{X,Y\}))_{\cG_{\Bar{\bI}}}$ is equivalent to $(X\ncperp Y|Pa_{\cG}(\{X,Y\}))_{P(do(\bI))}$ and under Scenario 2, we have $(X\ncperp_\sigma Y|Pa_{\cG}(\{X,Y\}))_{\cG_{\Bar{\bI}}}$ is equivalent to $(X\ncperp Y|Pa_{\cG}(\{X,Y\}))_{P(do(\bI))}$. This completes the proof.
\end{proof}

Lemma \ref{lem: step 2.1} states that by selecting experiments from a non-adjacent separating system, we can accurately identify the non-adjacent bidirected edge, i.e., $\hat{\bB^N} = \bB^N$.

\begin{remark}
    Compared to Theorem \ref{thm:lower B cI}, Proposition \ref{prop: non-ad separating cc} shows that all non-adjacent bidirected edges can be identified using $cc(\cG^{uc})$ experiments, which matches the worst-case lower bound. However, since the construction in Proposition \ref{prop: non-ad separating cc} is NP-hard, we instead adopt the randomized construction proposed in Proposition \ref{prop: non-ad separating random} as a practical alternative. This method also requires $\lfloor4e^2(d+1)^2\ln(n)\rfloor$ experiments and identifies all bidirected edges with probability at least $1-\frac{1}{n^2}$.
\end{remark}

\subsection{Step 2.2: Identify Adjacent Bidirected Edges}
\label{sec:unb 2.2}
In this section, we focus on the adjacent bidirected edges, $\bB^A$. According to \cite{NIPS2017_291d43c6}, adjacent bidirected edges cannot be identified solely through CI tests; instead, it is necessary to employ the \textit{do-see} test proposed by \cite{NIPS2017_291d43c6}.
We perform experiments on certain subsets of $\bV$ that form an adjacent separating system defined as follows.

\begin{definition}[Adjacent separating system]
\label{def: adjacent separating system}
Suppose $\cRB(\cG) = (\bV,\bD)$ is the directed part of DMG $\cG = (\bV,\bD,\bB^N\cup\bB^A)$. 
An adjacent separating system $\cI$ on $(\bV, \bD)$ is a collection of subsets of $\bV$ satisfying the following condition:

For every pair of nodes $X, Y \in \bV$ such that $(X, Y) \in \bD$ and $(Y, X) \notin \bD$, the following two intervention sets exist:
\begin{itemize}
    \item There exists $\bI \in \cI$ such that:
    \begin{equation*}
        Pa_\cG(\{X,Y\})\backslash\{X,Y\} \subseteq\bI \text{  and  }X\notin\bI,Y\notin\bI;
    \end{equation*}
    \item There exists $\bI' \in \cI$ such that:
    \begin{equation*}
        Pa_\cG(\{X,Y\})\backslash\{Y\} \subseteq\bI^\prime\text{  and  }X\in\bI^\prime,Y\notin\bI^\prime.
    \end{equation*}
\end{itemize}
%A adjacent separating system $\cI$ on $(\bV,\bD)$ is a collection of subsets of $\bV$ such that for each pair of nodes $X,Y$ which satisfies $(X,Y)\in\bD$ and $(Y,X) \notin\bD$, there exists $\bI\in\cI$ such that $Pa_\cG(\{X,Y\})\backslash\{X,Y\} \subseteq\bI$ and $X\notin\bI,Y\notin\bI$; and there also exists $\bI^\prime\in\cI$ such that $Pa_\cG(\{X,Y\})\backslash\{Y\} \subseteq\bI^\prime$ and $X\in\bI^\prime,Y\notin\bI^\prime$. 
\end{definition}

Then we present a theoretically grounded approach for constructing an adjacent separating system by using a strong edge coloring defined as follows.

\begin{definition}
    Given an undirected graph $\cG = (\bV,\bE)$,  a strong edge coloring is a function $f: \bE\rightarrow\cC$ that assigns a color to each edge such that: any two edges that are adjacent or connected by another edge must receive different colors. We use $\chi_s(\cG)$ to denote the minimum number of colors needed for a strong edge coloring of $\cG$.
\end{definition}

%Before we introduce the construct of the adjacent separating system, we first present some definitions. For a DMG $\cG = (\bV,\bD,\bB^N\cup\bB^A)$, we further partition the set $\bB^A$ into two disjoint subsets: the double adjacent bidirected edge set $\bB^{AD}\{[X,Y]\in\bB^A|(X,Y)\in\bD,(Y,X)\in\bD\}$ and the single adjacent bidirected edge set $\bB^{AS} = \bB^A\backslash\bB^{AD}$. Then we define a removing double adjacent bidirected edge operator $\cRD$ as $\cRD(\cG) = (\bV,\bD,\bB^N\cup\bB^{AS})$. 

\begin{proposition}
    \label{prop: step 2.2 min strong EC}
    Suppose $\cRB(\cG) = (\bV,\bD)$ is the directed part of DMG $\cG = (\bV,\bD,\bB^N\cup\bB^{AS}\cup\bB^{AD})$. There exists an adjacent separating system with at most $2\chi_s(\cG^u)$.
\end{proposition}
\begin{proof}
    Suppose $\cG^u$ is the directed skeleton of $\cG$ and it is determined by $\cRB(\cG)$. Define $\bE_k$ as the set of edges colored with the $k$th color in the minimum strong edge coloring of $\cG^u$. Let $\bN_k\subseteq\bV$ be the set of all endpoints of the edges in $\bE_k$, and let $\bT_k$ be the set of tail nodes of the directed edges in $\bE_k$, where the direction of each edge follows that in the original directed graph $\cRB(\cG)$. If an edge in $\cG^u$ corresponds to two directed edges in the original graph $\cRB(\cG)$, then either of its endpoints can be chosen as the tail node.
    For $1\leq k\leq \chi_s(\cG^u)$, we construct $\bI_k,\bI_k^\prime \subseteq\bV$ as follows,
    \begin{equation*}
        \bI_k = Pa_{\cG}(\bN_k)\backslash\bN_k, \qquad \bI_k^\prime = \left(Pa_\cG(\bN_k)\backslash\bN_k\right)\cup\bT_k
    \end{equation*}
    Then we show that $\cI = \{\bI_1,\bI_1^\prime,\dots,\bI_{\chi_s(\cG^u)},\bI_{\chi_s(\cG^u)}^\prime\}$ is a adjacent separating system on $(\bV,\bD)$. Note that $|\cI| = 2\chi_s(\cG^u)$. Suppose $X,Y\in\bV$ and $(X,Y)\in\bD$, $(Y,X)\notin\bD$ and corresponding undirected edge $[X,Y]^u$ in the directed skeleton graph $\cG^u$ belongs to the color class $\bE_k$. Then we have $X\in\bT_k\subseteq\bN_k$ and $Y\in\bN_k\backslash\bT_k$.  We want to show that there exist $\bI,\bI^\prime \in \cI$ such that $Pa_\cG(\{X,Y\})\backslash\{X,Y\} \subseteq\bI$ and $X\notin\bI,Y\notin\bI$; $Pa_\cG(\{X,Y\})\backslash\{Y\} \subseteq\bI^\prime$ and $X\in\bI^\prime,Y\notin\bI^\prime$. $\bI =\bI_k,\bI^\prime = \bI_k^\prime$ satisfies this propoty. According to the definition of strong edge coloring, we have $(Pa_\cG(\{X,Y\}\backslash\{X,Y\}) \cap \bN_k = \varnothing $ and $X,Y\in\bN_k$, $X\in \bT_k$, $Y\notin \bT_k$. Then we have $\bI_k\cap Pa_\cG(\{X,Y\}) = Pa_\cG(\{X,Y\})\backslash\{X,Y\}$, $X\notin\bI_k,Y\notin\bI_k$ and $\bI_k^\prime\cap Pa_\cG(\{X,Y\}) =Pa_\cG(\{X,Y\})\backslash\{Y\}$, $X\in\bI_k^\prime,Y\notin\bI_k^\prime$.
\end{proof}

\begin{example}[Strong edge coloring and adjacent separating system]
    Consider the DMG $\cG$ shown in Figure \ref{fig:example 4.1},  Figure \ref{fig:example 4.3} is a strong edge coloring of $\cG^u$ where different colors of edges mean different colors in the strong edge coloring. We can verify that this is the minimum strong edge coloring.
    
    Then we can use it to construct the adjacent separating system $\cI$ similar to Proposition \ref{prop: step 2.2 strong EC}, i.e., $\cI = \{\varnothing,\{X_1\},\{X_5\},\{X_1,X_5\},\{X_1,X_4\},\{X_1,X_3,X_4\},\{X_1,X_3\},\{X_1,X_3,X_5\},$ \\ $\{X_3,X_4\},\{X_1,X_3,X_4,X_5\}\}$ and $|\cI| = 2\chi_s(\cG^{u}) = 10$.
\end{example}

Although the proof of Proposition \ref{prop: step 2.2 min strong EC} is constructive, constructing an adjacent separating system requires finding a minimum strong edge covering, which is NP-hard. Therefore, in practice, we recommend using a simple greedy algorithm to find an alternative adjacent separating system.

\begin{proposition}
    \label{prop: step 2.2 strong EC}
    Suppose $\cRB(\cG) = (\bV,\bD)$ is the directed part of DMG $\cG = (\bV,\bD,\bB^N\cup\bB^{AS}\cup\bB^{AD})$ and $\cG^u$ is the directed skeleton of $\cRB(\cG)$, $d$ is the max degree of $\cG^u$. We can use a simple greedy algorithm to find an adjacent separating system on $(\bV,\bD)$ with at most $4d^2$ elements.
\end{proposition}
\begin{proof}
    The proof is similar to that of Proposition \ref{prop: step 2.2 min strong EC}. The key difference lies in replacing the minimum strong edge coloring used in Proposition \ref{prop: step 2.2 min strong EC} with a general strong edge coloring, which can be efficiently obtained using a greedy algorithm \citep{Choi_2018}. Although not optimal, this general coloring uses at most $2d^2$ colors. The remainder of the proof follows identically to that of Proposition \ref{prop: step 2.2 min strong EC}.
\end{proof}

We present Algorithm \ref{alg: 2.2} to learn the single adjacent bidirected edges $\bB^{AS}$ of DMG $\cG$ that takes $\cRB(\cG)$ as input. The algorithm constructs an adjacent separating system and initializes a set of single adjacent bidirected edges as empty at first.

\begin{algorithm}
        \caption{Learning the single adjacent bidirected edges $\bB^{AS}$}
	\label{alg: 2.2}
    \begin{algorithmic}
        \REQUIRE $\cRB(\cG) = (\bV,\bD)$
        \ENSURE $\hat{\bB^{AS}}$
        \STATE $\cI \leftarrow$ an adjacent separating system on $(\bV,\bD)$\\
        \STATE Initialization: $\hat{\bB^{AS}} = \varnothing$
        \FOR{ $[X,Y]\in \{[S,R]| (S,R)\in\bD, (R,S)\notin\bD\}$}
        \STATE $\bI\leftarrow$ an element of $\cI$ that contains $Pa_{\cG}(\{X,Y\})\backslash\{X,Y\}$ but does not contain $X,Y$\\
        \STATE $\bI^\prime\leftarrow$ an element of $\cI$ that contains $Pa_{\cG}(\{X, Y\})\backslash\{Y\}$ and $X$ but does not contain $Y$\\
        \STATE Add $[X,Y]$ to $\hat{\bB^{AS}}$ if $P_{do(\bI)}(Y|X,do(\bI))\neq P_{do(\bI^\prime)}(Y|do(\bI^\prime))$\\
        \ENDFOR
    \end{algorithmic}
\end{algorithm}

Suppose $X,Y$ are two arbitrary variables in $\bV$ which satisfy $(X,Y)\in\bD, (Y,X)\notin\bD$. Since $\cI$ is an adjacent separating system on $(\bV,\bD)$, there exist $\bI\in\cI$ that contains $Pa_{\cG}(\{X,Y\})\backslash\{X,Y\}$ but does not contain $X,Y$ and exist $\bI^\prime\in\cI$ that contain $Pa_{\cG}(\{XY\})\backslash\{Y\}$ and $X$ but does not contain $Y$. According to the following lemma, we find that performing the experiment on $\bI$ and $\bI^\prime$ can justify whether $[X,Y]\in\bB^{AS}$ through the do-see test.

\begin{lemma}[\cite{NIPS2017_291d43c6}]
    \label{lem:step2.2}
    Suppose $\cG = (\bV,\bD,\bB^N\cup\bB^{AS}\cup\bB^{AD})$ is a DMG. For each two variables $X,Y\in\bV$ satisfy $(X,Y)\in\bD$ and $(Y,X)\notin\bD$, if $Pa_\cG(\{X,Y\})\backslash\{X,Y\} \subseteq\bI$, $X\notin\bI,Y\notin\bI$ and $Pa_\cG(\{X,Y\})\backslash\{Y\} \subseteq\bI^\prime$, $X\in\bI^\prime,Y\notin\bI^\prime$, then $[X,Y]\in\bB^{AS}$ if and only if $P_{do(\bI)}(Y|X,do(\bI))\neq P_{do(\bI^\prime)}(Y|do(\bI^\prime))$.
\end{lemma}

\begin{proof}
Let $\mathbf{P}=Pa_{\cG}(\{X,Y\})\backslash\{X,Y\}$. Since $(X,Y)\in\bD$ and $(Y,X)\notin\bD$, all observable parents of $X$, as well as all observable parents of $Y$ other than $X$, belong to $\mathbf{P}$. By assumption, these variables are intervened upon in both experiments, while $X$ is additionally intervened upon under $do(\bI^\prime)$.

Consequently, all directed paths entering $X$ or $Y$ through their observable parents are cut. Therefore, although $\cG$ may contain cycles elsewhere, the local causal model relevant to the do-see test between $X$ and $Y$ is acyclic. Moreover, conditional on the realized values of the intervened variables, interventions outside $\{X,Y\}\cup\mathbf{P}$ do not affect the corresponding local distributions.

Hence, the proof of the interventional do-see test in
\citet[Theorem~6]{NIPS2017_291d43c6} applies directly, yielding
\[
P_{do(\bI)}(Y\mid X,do(\bI))
=
P_{do(\bI^\prime)}(Y\mid do(\bI^\prime))
\]
if and only if $X$ and $Y$ do not share an exogenous variable, i.e.,
$\bU^X\cap\bU^Y=\varnothing$. By the definition of the induced DMG,
$\bU^X\cap\bU^Y\neq\varnothing$ is equivalent to $[X,Y]\in\bB$.
Since $(X,Y)\in\bD$ and $(Y,X)\notin\bD$, this is further equivalent to
$[X,Y]\in\bB^{AS}$. Therefore,
\[
[X,Y]\in\bB^{AS}
\quad\Longleftrightarrow\quad
P_{do(\bI)}(Y\mid X,do(\bI))
\neq
P_{do(\bI^\prime)}(Y\mid do(\bI^\prime)),
\]
which completes the proof.
\end{proof}

\begin{remark}
    Lemma \ref{lem:step2.2} only holds for single adjacent bidirected edges, which is $\{[X,Y]\in\bB|(X,Y)\in\bD, (Y,X)\notin\bD\}$. If $(X,Y)\in\bD$ and $(Y,X)\in\bD$, $P_{do(\bI)}(Y|X,do(\bI))\neq P_{do(\bI^\prime)}(Y|do(\bI^\prime))$ always holds, thus we can not justify whether $[X,Y]\in\bB^{AD}$.
\end{remark}
Lemma \ref{lem:step2.2} states that by selecting experiments from an adjacent separating system, we can accurately identify the single adjacent bidirected edge, i.e., $\hat{\bB^{AS}} = \bB^{AS}$.

In Section~\ref{sec:unb 2.1}, we demonstrate that the non-adjacent bidirected edges of a DMG $\cG$ can be identified by conducting experiments based on a non-adjacent separating system. In this section, we further show that the set of single adjacent bidirected edges, denoted $\bB^{AS}$, can be recovered through experiments designed using an adjacent separating system.
To support this, we provide Propositions \ref{prop: non-ad separating cc} and \ref{prop: step 2.2 min strong EC}, which describe constructions for non-adjacent and adjacent separating systems, respectively. Moreover, for components that are computationally intractable due to NP-hardness, we propose alternative constructions based on Propositions \ref{prop: non-ad separating random} and \ref{prop: step 2.2 strong EC}, which offer practical approximations to the ideal separating systems.

%In Section \ref{sec:unb 2.1}, we show that the non-adjacent bidirected edges of DMG $\cG$ can be learned by performing experiments on the elements of a non-adjacent separating system. Herein, we showed that the single adjacent bidirected edges $\bB^{AS}$ can be recovered by performing experiments on the elements of the adjacent separating system. Moreover, we provided Propositions \ref{prop: non-ad separating cc} and \ref{prop: step 2.2 min strong EC} for constructing a non-adjacent separating system and an adjacent separating system, respectively. Furthermore, we replace the components that are difficult to solve due to their NP-hardness with alternative constructions based on Propositions \ref{prop: non-ad separating random} and \ref{prop: step 2.2 strong EC}. 

Thus far, through Algorithms \ref{alg: 1.1} to \ref{alg: 2.2}, we have successfully identified all directed edges and the majority of bidirected edges in the DMG $\cG$, with the exception of double-adjacent bidirected edges. Specifically, when both directed edges $(X, Y)$ and $(Y, X)$ are present, we are unable to determine whether the bidirected edge $[X, Y]$ belongs to $\bB^{AD}$. This difficulty arises because in such cases, not only do CI tests fail, but even do-see tests become ineffective.
Intuitively, the presence of $[X, Y]$ has little impact on the connectivity of the graph, as both $X$ and $Y$ simultaneously possess incoming and outgoing directed edges. As a result, intervening on either node yields post-intervention distributions that remain invariant, regardless of whether $[X, Y]$ is present.
While this limitation poses a challenge for edge identification, we argue that accurately recovering such double-adjacent bidirected edges may be less critical, as their presence or absence has minimal effect on the overall causal structure. A deeper investigation of this limitation is left for future work.

%So far, through Algorithms \ref{alg: 1.1} to \ref{alg: 2.2}, we have successfully identified all directed edges and the majority of bidirected edges in the DMG $\cG$, except for double adjacent bidirected edges. Specifically, when both edges $(X,Y)$ and $(Y,X)$ exist, we are unable to determine whether edge $[X,Y]\in\bB^{AD}$. This is because, in such cases, not only CI tests fail, but even do-see tests become invalid. Intuitively, adding an edge $[X,Y]$ in this setting has little impact on the overall graph connectivity, since both $X$ and $Y$ simultaneously have incoming and outgoing edges. As a result, intervening on either $X$ or $Y$ yields distributions that remain invariant regardless of whether $[X,Y]$ is present.
%While this poses a challenge for edge estimation, we argue that accurately identifying such a double adjacent bidirected edge $[X,Y]$ in such scenarios may be less critical, as its presence or absence does not substantially affect the causal structure. This limitation will be explored in future work.

\begin{corollary}
\label{coro: Upper unbounded DB cI}
    For a DMG $\cG=(\bV,\bD,\bB^{N}\cup\bB^{AS}\cup\bB^{AD})$, if the NP-hardness can be ignored, Algorithms \ref{alg: 1.1} to \ref{alg: 2.2} together learn the $\cRD(\cG) = (\bV,\bD,\bB^{N}\cup\bB^{AS})$ with as most 
    \begin{equation}
        2\lceil\log_2(\chi(\cG_r^{obs}))\rceil + \sum_{k=1}^{l+1}\zeta_{max}^{k,\cG} + cc(\cG^{uc})+2\chi_s(\cG^u)
    \end{equation}
    elements. If the NP-hardness can not be ignored, Algorithms \ref{alg: 1.1} to \ref{alg: 2.2} together learn the $\cRD(\cG) = (\bV,\bD,\bB^{N}\cup\bB^{AS})$ with as most 
    \begin{equation}
        2\lceil\log_2(d_1+1)\rceil+ \sum_{k=1}^{l+1}\zeta_{max}^{k,\cG} + \lfloor4e^2(d_2+1)^2\ln(n)\rfloor + 4d_2^2
    \end{equation}
    elements with probability $1-\frac{1}{n^2}$, where $n$ is the number of nodes, $d_1$ is the max degree of $\cG_r^{obs}$, $d_2$ is the max degree of $\cG^{u}$.
\end{corollary}

\section{Bounded-size Experiment Design}
\label{sec: boun algorithm}
 In the previous sections, we did not place any restrictions on the size of the intervention sets, allowing the algorithm to perform experiments of arbitrary size. However, in practical scenarios, large-scale interventions may be infeasible or prohibitively expensive. In this section, we address the experiment design problem under a constraint on experiment size. Specifically, our objective is to design a collection of intervention sets, denoted by $\cI$, such that $\cRA([\cG]_\cI^r)={\cRA(\cG)}$, and to use the do-see test to recover as much of $\bB^A$ as possible. Each intervention set $\bI\in\cI$ is subject to a size constraint $|\bI|\leq M$, where $M<n$ is a given constant.

 \begin{remark}
 \label{remark: M setting}
     As proved in Theorems \ref{thm:lower D bI} and \ref{thm:lower B cI},  it is necessary to perform some experiments with the max size at least $\max\{\max_{[X,Y]\in\bB^N }|Pa_{\cG}(X\cup Y) |, |\rmT_{l+1}^{\cG}|_n + \zeta^{l+1,\cG}_{max} -1\}$ to learn a DMG $\cG$ in the worst case. Furthermore, to identify adjacent bidirected edges via the do-see test, the largest size of the experiment must have at least $\max_{[X,Y]\in\bB^A }|Pa_{\cG}(X\cup Y) |$ elements. Hence, we set the upper bound $M$ cannot be smaller than $\max\{\max_{\{X,Y\}\in\bV,X\neq Y }|Pa_{\cG}(X\cup Y) |, |\rmT_{l+1}^{\cG}|_n + \zeta^{l+1,\cG}_{max} -1\}$.
 \end{remark}

Since Step 0 does not involve any experiments, we will modify the rest four steps of our proposed method (introduced in Sections \ref{sec:unb 1.1} to \ref{sec:unb 2.2}) in order to accommodate the new constraint that the size of the experiments is bounded by a constant $M \geq \max\{\max_{\{X,Y\}\in\bV,X\neq Y }|Pa_{\cG}(X\cup Y) |, |\rmT_{l+1}^{\cG}|_n + \zeta^{l+1,\cG}_{max} -1\}$.

\subsection{Step 1.1: Identify Ancestral Relationship}
\label{sec:bou 1.1}
 In this step, instead of learning $\cG_r^{obs}$ and constructing a colored separating system, we construct an $(n,M)$-separating system, formally defined by \cite{nMsystem2015} as follows.
\begin{definition}[$(n,M)$-separating system]
    An $(n,M)$-separating system $\cI$ on $\bV$ is a collection of subsets of $\bV$ ($|\bV| = n$) such that $|\bI|\leq M$ for each $\bI\in\cI$, and for every ordered pair of distinct variables $X,Y$ in $\bV$ there exists $\bI\in\cI$ such that $X\in\bI$ and $Y\notin\bI$.
\end{definition}

According to \cite{nMsystem2015}, the cardinality of the $(n,M)$-separating system also has an achievable bound. 

\begin{proposition}[\cite{nMsystem2015}]
\label{prop: boun 1.1}
    There exists an $(n,M)$-separating system on $\bV$ with at most $\lceil\frac{n}{M}\rceil\lceil \log_{\lceil\frac{n}{M}\rceil}n\rceil$ elements.
\end{proposition}

\begin{remark}
    The proof of Proposition \ref{prop: boun 1.1} in \cite{nMsystem2015} is constructive, which allows us to obtain an $(n,M)$-separating system on $\bV$ with at most  $\lceil\frac{n}{M}\rceil\lceil \log_{\lceil\frac{n}{M}\rceil}n\rceil$ elements.
\end{remark}

 It suffices to modify the $\cI$ of Algorithm \ref{alg: 1.1} by setting $\cI$ to be an $(n,M)$-separating system on $\bV$ and leaving the rest of the algorithm unchanged. It is straightforward to verify that the modified algorithm obtains $\{De_{\cG}(X)\}_{X\in\bV}$ and SCCs $\cS = \{\bS_1,\dots,\bS_s\}$ of $\cG$ by performing experiments on the elements of $\cI$.

\subsection{Step 1.2: Identify Directed Edges}
\label{sec:bou 1.2}
Algorithm \ref{alg: 1.2} remains unchanged for this step except the SCC-Anc separating system. We need to construct an SCC-Anc separating system on $(\bV,\bbT^\cG)$ such that the size of the elements of $\bI$ does not exceed $M$.

\begin{theorem}
    \label{thm: boun 1.2}
    Suppose $M \geq \max\{\max_{\{X,Y\}\in\bV,X\neq Y }|Pa_{\cG}(X\cup Y) |, |\rmT_{l+1}^{\cG}|_n + \zeta^{l+1,\cG}_{max} -1\}$. There exist an SCC-Anc separating system $\cI$ on $(\bV,\bbT^\cG)$ such that for each $\bI\in\cI$, $|\bI| \leq M$ and $|\cI|\leq \sum_{k=1}^{l+1}\zeta_{max}^{k,\cG} +\zeta_{max}^{l+1,\cG}\max\{0,\lfloor\frac{n-|\rmT_{l+1}^{\cG}|_n-\zeta_{max}^{l+1,\cG}-1}{M-|\rmT_{l+1}^{\cG}|_n-\zeta_{max}^{l+1,\cG}+2}\rfloor\}$.
\end{theorem}
\begin{proof}
    Recalling the Equation (\ref{eq: SCC-Anc separating constract}) in the proof of Proposition \ref{prop: step 1.2 SCC-Anc separating}, we find that fixed $k\in\{1,\dots,l\}$, $\cI_k=\{\bI_{k,1},\dots,\bI_{k,m_{k,max}}\}$ can guarantee that for $j\in\{1,2,\dots,n_k\}$, $\bS_{k,j}\in\cT_k^\cG$ and $X\in\bS_{k,j}$, there exist $\bI\in\cI_k$ such that $\rmT_k^{\cG}\cup\bS_{k,j}\backslash\{X\} \subseteq\bI$ and $X\notin\bI$. We also find that for the construction in Equation (\ref{eq: SCC-Anc separating constract}), $|\bI_{k,q}| < |\rmT_{l+1}^\cG|_n \leq M$ for $1\leq k\leq l$ and $1\leq q\leq m_{k,max}$. Thus we only need to modify the $\cI_{l+1} = \{\bI_{l+1,1},\cdots,\bI_{l+1,m_{l+1}}\}$ to make them bounded by $M$ and they still can guarantee for $j\in\{1,2,\dots,n_{l+1}\}$, $\bS_{l+1,j}\in\cT_{l+1}^\cG$ and $X\in\bS_{l+1,j}$, there exist $\bI\in\cI_{l+1}$ such that $\rmT_{l+1}^{\cG}\cup\bS_{l+1,j}\backslash\{X\} \subseteq\bI$ and $X\notin\bI$.

    We construct $\cI_{l+1} = \{\rmT_{l+1}^\cG \cup \bI^\prime_q\}_{q = 1}^{m_{l+1}}$. In this way, we only need to guarantee that $|\bI^\prime_q| \leq M- |\rmT_{l+1}^\cG|_n$ for $1\leq q\leq m_{l+1}$ and for $j\in\{1,2,\dots,n_{l+1}\}$, $\bS_{l+1,j}\in\cT_{l+1}^\cG$ and $X\in\bS_{l+1,j}$, there exist $\bI\in \{ \bI^\prime_q\}_{q = 1}^{m_{l+1}}$ such that $\bS_{l+1,j}\backslash\{X\} \subseteq\bI$ and $X\notin\bI$. According to Theorem 46 in \cite{mokhtarian2023unified}, such $\{ \bI^\prime_q\}_{q = 1}^{m_{l+1}}$ exist and has at most  $\zeta_{max}^{l+1,\cG}\left(1+\lfloor\frac{n-|\rmT_{l+1}^{\cG}|_n - \zeta_{max}^{l+1,\cG}-1}{M-|\rmT_{l+1}^{\cG}|_n - \zeta_{max}^{l+1,\cG}+2} \rfloor\right)$  elements. 

    In this way, combine $\cI_k$ for $1\leq k\leq l+1$, we get $\cI=\cup_{k=1}^{l+1}\cI_k$ is an SCC-Anc separating system with $|\cI|\leq \sum_{k=1}^{l+1}\zeta_{max}^{k,\cG} +\zeta_{max}^{l+1,\cG}\lfloor\frac{n-|\rmT_{l+1}^{\cG}|_n - \zeta_{max}^{l+1,\cG}-1}{M-|\rmT_{l+1}^{\cG}|_n - \zeta_{max}^{l+1,\cG}+2} \rfloor$ elements and each element is bounded by $M$.

    However, Theorem 46 of \cite{mokhtarian2023unified} omits the boundary case $n-|\rmT_{l+1}^{\cG}|_n=\zeta_{max}^{l+1,\cG}$, where the last layer consists of a single SCC $\bS$. In this case, the $\zeta_{max}^{l+1,\cG}$ sets $\rmT_{l+1}^{\cG}\cup(\bS\backslash\{X\})$, $X\in\bS$, satisfy the required conditions and have size at most $M$. Thus, to cover this boundary case, the floor term in the stated cardinality bound is replaced by $\max\{0,\lfloor\frac{n-|\rmT_{l+1}^{\cG}|_n-\zeta_{max}^{l+1,\cG}-1}{M-|\rmT_{l+1}^{\cG}|_n-\zeta_{max}^{l+1,\cG}+2}\rfloor\}$.
\end{proof}

If we only consider identifying the directed part of DMG and setting $M = |\rmT_{l+1}^{\cG}|_n + \zeta^{l+1,\cG}_{max} -1$ in Theorem \ref{thm: boun 1.2}, we get the following notable corollary.
\begin{corollary}
    The directed part $\cRB(\cG)$ of a DMG $\cG$ can be learned by performing experiments with maximum size at most
$\max\{1, |\rmT_{l+1}^{\cG}|_n + \zeta^{l+1,\cG}_{max} - 1\}$.
The maximum with $1$ is needed only when $\cG$ has no directed edges.
For the parameter settings covered by Theorem \ref{thm:lower D bI}, this upper bound matches the corresponding worst-case lower bound and is therefore tight.
\end{corollary}

\begin{corollary}
\label{coro: Upper bounded D cI}
    To sum up step 1.1 and step 1.2, the modified Algorithms \ref{alg: 1.1} and \ref{alg: 1.2} can learn a directed part of DMG $\cRB(\cG)$ with $n$ vertices by performing at most
\begin{equation*}
    \lceil\frac{n}{M}\rceil\lceil \log_{\lceil\frac{n}{M}\rceil}n\rceil + \sum_{k=1}^{l+1}\zeta_{max}^{k,\cG} +\zeta_{max}^{l+1,\cG}\max\{0,\lfloor\frac{n-|\rmT_{l+1}^{\cG}|_n-\zeta_{max}^{l+1,\cG}-1}{M-|\rmT_{l+1}^{\cG}|_n-\zeta_{max}^{l+1,\cG}+2}\rfloor\}
\end{equation*}
experiments with size at most $M$, where $M \geq \max\{\max_{\{X,Y\}\in\bV,X\neq Y }|Pa_{\cG}(X\cup Y) |, |\rmT_{l+1}^{\cG}|_n + \zeta^{l+1,\cG}_{max} -1\}$.
\end{corollary}

\subsection{Step 2.1: Identify Non-adjacent Bidirected Edges}
\label{sec:bou 2.1}

Algorithm \ref{alg: 2.1} remains unchanged for this step except for the non-adjacent separating system. We need to construct a non-adjacent separating system on $(\bV,\bD)$ such that the size of the elements of $\bI$ does not exceed $M$.

\begin{theorem}
    \label{thm: boun 2.1}
    Suppose $M \geq \max\{\max_{\{X,Y\}\in\bV,X\neq Y }|Pa_{\cG}(X\cup Y) |, |\rmT_{l+1}^{\cG}|_n + \zeta^{l+1,\cG}_{max} -1\}$. There exist a non-adjacent separating system on $(\bV,\bD)$ such that for each $\bI\in\cI$, $|\bI| \leq M$ and
    \begin{equation*}
        |\cI|\leq\sum_{k=1}^{K}1+\lfloor\frac{(\frac{|\bC_k|(|\bC_k|-1)}{2}-1)(n-|\bC_k|)}{M+1-\max_{X,Y\in\bC_k,X\neq Y}|Pa_{\cG}(\{X,Y\})|}  \rfloor,
    \end{equation*}
    where $\{\bC_1,\dots,\bC_K\}$ is a edge clique covering of $\cG^{uc}$ and $K$ is the corresponding edge clique covering number.
\end{theorem}
\begin{proof}
    Recalling the proof of Propositions \ref{prop: non-ad separating cc} and \ref{prop: non-ad separating random}, to construct a non-adjacent separating system of $\cG$, we essentially aim to find an edge clique covering of $\cG^{uc}$. Once we find $\{\bC_1,\dots,\bC_K\}$ is an edge clique covering of $\cG^{uc}$, we can construct $\cI = \{\bI_k = Pa_{\cG}(\bC_k)\}_{k=1}^K$ is a non-adjacent separating system of $\cG$.  However, this does not guarantee the boundedness of each $\bI_k$. Therefore, for each clique $\bC_k$ in $\cG^{uc}$, we further partition the $|\bC_k|(|\bC_k|-1)/2$ edges (corresponding to the potential locations of non-adjacent bidirected edges among the nodes in $\bC_k$ in $\cG$) into $t_k = 1+\lfloor\frac{(\frac{|\bC_k|(|\bC_k|-1)}{2}-1)(n-|\bC_k|)}{M+1-\max_{X,Y\in\bC_k,X\neq Y}|Pa_{\cG}(\{X,Y\})|}  \rfloor $ subsets using a greedy algorithm, such that the parent set of each subset has size less than $M$. The detailed procedure is as follows.
    We use $C(\bE)$ to denote the node set corresponding to edges in $\bE$.

    Suppose $\{\bC_1,\dots,\bC_K\}$ is a edge clique covering of $\cG^{uc}$. Let us fix a $1\leq k\leq K$. Denote $\bE_k = \{E_{1},\dots,E_{|\bC_k|(|\bC_k|-1)/2}\}$ as the all edges in $\cG^{uc}$ of clique $\bC_k$.
     Next, we will introduce $t_k$ subsets (we call them bins) $\bE_{k,1}, \dots,\bE_{k,t_k}$ of $\bE_k$, and each parent set in $\cG$ at most $M$ elements, i.e., $|Pa_{\cG}(\bE_{k,j})| \leq M, 1\leq j\leq t$. Since $\bE_{k,j}$ is the edge subset of the clique in $\cG^{uc}$, $Pa_{\cG}(\bE_{k,j}) \cap C(\bE_{k,j}) = \varnothing$ which means that they can be used to construct the non-adjacent separating system. 
     We construct these $t_k$ subsets of $\bE_k$ under a stronger condition, in which any two edges in $\bE_k$ have disjoint parent sets in $\cG$. Clearly, if such a partition of $\bE_k$ into $t_k$ subsets with parent set size less than $M$ exists under this stricter assumption, the result still holds in the more general setting.

    We initialize the bins with empty sets.  Then, we add $E_j$ to them in a greedy manner such that the size of the parent set bins in $\cG$ remains less than $M$. That is, we first add the edge $E_1$ into $\bE_{k,1}$. Note this feasible since $|Pa_\cG(E_1)| \leq M$. Then we add $E_2$ to the first feasible bin, i.e., the first bin, such that its parent set size in $\cG$ remains less than $M$ after adding the elements of $E_2$ to it. We subsequently add the elements of $E_j$ to the first feasible bin. It is left to show that there always exists a feasible bin during this process. Suppose $E_1,E_2,\dots,E_{x-1}$ are already placed in the bins, where $1\leq x< |\bC_k|(|\bC_k|-1)/2$, and we want to find a feasible bin for $E_{x}$. Assume by contradiction that there is no feasible bin for $E_{x}$. This shows that adding $E_{x}$ to any bin results in a parent set of bins in $\cG$ that contains at least $M + 1$ elements. Hence,
    \begin{equation}
    \label{eq:bound 2.1 pf1}
        t_k(M+1-|Pa_\cG(E_{x})|) \leq |Pa_\cG(E_1)| +\cdots |Pa_\cG(E_{x-1})|
    \end{equation}
    Since we have $Pa_{\cG}(E_{k,j}) \cap C(\bE_{k}) = \varnothing$ for $1\leq j\leq |\bC_k|(|\bC_k|-1)/2$, we also have 
    \begin{equation}
    \label{eq:bound 2.1 pf2}
        |Pa_\cG(E_1)| +\cdots |Pa_\cG(E_{x-1})| \leq (x-1)(n-|\bC_k|)
    \end{equation}
    Note $x\leq |\bC_k|(|\bC_k|-1)/2$ and $|Pa_\cG(E_{x})|\leq \max_{X,Y\in\bC_k,X\neq Y}Pa_{\cG}(\{X,Y\}) $. Combine Equation (\ref{eq:bound 2.1 pf1}) and (\ref{eq:bound 2.1 pf2}), we have 
    \begin{equation*}
        \begin{aligned}
             1+\lfloor\frac{(\frac{|\bC_k|(|\bC_k|-1)}{2}-1)(n-|\bC_k|)}{M+1-\max_{X,Y\in\bC_k,X\neq Y}|Pa_{\cG}(\{X,Y\})|}  \rfloor = t_k &\leq \frac{(x-1)(n-|\bC_k|)}{M+1-|Pa_\cG(E_{x})|} \\
             &\leq \frac{(\frac{|\bC_k|(|\bC_k|-1)}{2}-1)(n-|\bC_k|)}{M+1-\max_{X,Y\in\bC_k,X\neq Y}|Pa_{\cG}(\{X,Y\})|},
        \end{aligned}
    \end{equation*}
    which is a contradiction. This shows that it is feasible to add all the $E_x$ to the bins in a
 greedy manner under the stronger condition, in which any two edges in $\bE_k$ have disjoint parent sets in $\cG$.

 Finally, similar to the construction of Propostions \ref{prop: non-ad separating cc} and \ref{prop: non-ad separating random}, construct $\cI_k = \{Pa_{\cG}(\bE_{k,j})\}_{j=1}^{t_k}$ and $\cI = \cup_{k=1}^{K}\cI_k$. In this way, $\cI$ is a non-adjacent separating system and bounded by $M$ with $\sum_{k=1}^{K}1+\lfloor\frac{(\frac{|\bC_k|(|\bC_k|-1)}{2}-1)(n-|\bC_k|)}{M+1-\max_{X,Y\in\bC_k,X\neq Y}|Pa_{\cG}(\{X,Y\})|}  \rfloor$ elements.

\end{proof}

\begin{remark}
    The proof of Theorem \ref{thm: boun 2.1} is constructive and associated with the edge clique covering. If we disregard the NP-hardness of the computational problem, $K$ can be set to $cc(\cG^{uc})$, and $\{\bC_1,\dots,\bC_{cc(\cG^{uc})}\}$ corresponds to the minimal edge clique covering, similar to Proposition \ref{prop: non-ad separating cc}. However, if the NP-hardness cannot be ignored, as mentioned in Proposition \ref{prop: non-ad separating random}, there exists an algorithm that returns a collection $\{\bC_1,\dots,\bC_k\}$ which forms an edge clique covering with probability $1-\frac{1}{n^2}$. In this case, $K$ can be set $\lfloor4e^2(d+1)^2\ln(n)\rfloor$. Furthermore, if we let $M\to\infty$, Theorem \ref{thm: boun 2.1} degenerates to Propositions \ref{prop: non-ad separating cc} and \ref{prop: non-ad separating random}.
\end{remark}

According to Theorem \ref{thm: boun 2.1}, we can construct a non-adjacent separating system $\cI$ with each element bounded by $M$.  We only need to modify Algorithm \ref{alg: 2.1} to use the bounded non-adjacent separating system. It is straightforward to verify that the modified algorithm obtains $\bB^N$ of $\cG$ by performing experiments on the elements of $\cI$.

\subsection{Step 2.2: Identify Adjacent Bidirected Edges}
\label{sec:bou 2.2}

Algorithm \ref{alg: 2.2} remains unchanged for this step except for the adjacent separating system. We need to construct an adjacent separating system on $(\bV,\bD)$ such that the size of the elements of $\bI$ does not exceed $M$.
\begin{theorem}
    \label{thm: boun 2.2}
    Suppose $M \geq \max\{\max_{\{X,Y\}\in\bV,X\neq Y }|Pa_{\cG}(X\cup Y) |, |\rmT_{l+1}^{\cG}|_n + \zeta^{l+1,\cG}_{max} -1\}$. There exists an adjacent separating system on $(\bV,\bD)$ such that for each $\bI\in\cI$, $|\bI| \leq M$ and
    \begin{equation*}
        |\cI|\leq2\sum_{k=1}^{K}1+\lfloor\frac{(|\bE_k|-1)(n-2|\bE_k|+1)}{M+1-\max_{[X,Y]^u\in\bE_k}|Pa_{\cG}(\{X,Y\})|}  \rfloor.
    \end{equation*}
    Here $\{\bE_1,\dots,\bE_K\}$ represents a strong edge coloring of $\cG^u$, where each $\bE_k$ denotes the set of edges assigned the same color, and K is the corresponding strong edge coloring number.
\end{theorem}
\begin{proof}
    Recalling the proofs of Propositions
    \ref{prop: step 2.2 min strong EC} and
    \ref{prop: step 2.2 strong EC}, to construct an adjacent separating
    system of $\cG$, we essentially aim to find a strong edge coloring of
    $\cG^{u}$. Once $\{\bE_1,\dots,\bE_K\}$ is a strong edge coloring of
    $\cG^{u}$, we can construct $\cI
        =
        \left\{
        \bI_k=Pa_{\cG}(\bE_k)\setminus\bN_k,\,
        \bI_k^\prime=
        \bigl(Pa_{\cG}(\bE_k)\setminus\bN_k\bigr)\cup\bT_k
        \right\}_{k=1}^K$,
    which is an adjacent separating system of $\cG$, where
    $\bN_k\subseteq\bV$ is the set of all endpoints of the edges in
    $\bE_k$, and $\bT_k$ is the set of selected tail nodes of the
    corresponding directed edges in the original graph $\cG$.
    However, this construction does not guarantee the boundedness of
    each $\bI_k$ or $\bI_k^\prime$.

    Therefore, for each color class $\bE_k$ in $\cG^{u}$, we further
    partition its $|\bE_k|$ edges, which correspond to the potential
    locations of adjacent bidirected edges in $\cG$, into $t_k
        =
        1+
        \left\lfloor
        \frac{
        (|\bE_k|-1)(n-2|\bE_k|+1)
        }{
        M+1-
        \max_{[X,Y]^u\in\bE_k}
        |Pa_{\cG}(\{X,Y\})|
        }
        \right\rfloor$ 
    subsets using a greedy algorithm, such that the intervention set
    associated with each subset has size at most $M$. The detailed
    procedure is as follows. We use $C(\bE)$ to denote the node set
    corresponding to the edges in $\bE$.

    Suppose $\{\bE_1,\dots,\bE_K\}$ is a strong edge coloring of
    $\cG^{u}$. Fix $1\leq k\leq K$ and write $\bE_k=\{E_1,\dots,E_{|\bE_k|}\}$. 
    For each $E_i=[X_i,Y_i]^u\in\bE_k$, let $T_i$ denote its selected
    tail node, and define
    \[
        \mathcal P_i
        :=
        \bigl(
        Pa_{\cG}(\{X_i,Y_i\})\setminus\{X_i,Y_i\}
        \bigr)
        \cup\{T_i\}.
    \]
    For an edge subset $\bE\subseteq\bE_k$, define $\mathcal P(\bE)
        :=
        \bigcup_{E_i\in\bE}\mathcal P_i$. 
    Since every color class in a strong edge coloring is an induced
    matching, for every bin $\bE_{k,j}$ we have
    \[
        \mathcal P(\bE_{k,j})
        =
        \bigl(
        Pa_{\cG}(\bE_{k,j})\setminus\bN_{k,j}
        \bigr)
        \cup\bT_{k,j},
    \]
    where $\bN_{k,j}$ is the set of endpoints of the edges in
    $\bE_{k,j}$ and $\bT_{k,j}$ is the corresponding set of selected
    tail nodes.

    We construct $t_k$ bins
    $\bE_{k,1},\dots,\bE_{k,t_k}$ such that $ |\mathcal P(\bE_{k,j})|\leq M,
        \qquad 1\leq j\leq t_k$.
    We first establish the result under the stronger condition that
    $\mathcal P_1,\dots,\mathcal P_{|\bE_k|}$ are pairwise disjoint.
    Clearly, if such a partition exists under this stricter condition,
    the result continues to hold in the general case, since overlaps
    among these sets can only reduce the size of their union.

    We initialize the bins with empty sets and add the edges greedily.
    First, add $E_1$ to $\bE_{k,1}$, which is feasible because $|\mathcal P_1|
        \leq
        |Pa_{\cG}(\{X_1,Y_1\})|
        \leq M$.
    Subsequently, add each edge to the first bin for which the
    corresponding set $\mathcal P(\bE_{k,j})$ remains of size at most
    $M$.

    Suppose $E_1,\dots,E_{x-1}$ have already been placed in the bins,
    where $1\leq x\leq|\bE_k|$, and consider placing $E_x$. Assume by
    contradiction that no bin is feasible for $E_x$. Then adding $E_x$
    to any bin produces an intervention set containing at least $M+1$
    elements. Under the pairwise-disjointness condition, this implies
    \begin{equation}
        \label{eq:bound 2.2 pf1}
        t_k\bigl(M+1-|\mathcal P_x|\bigr)
        \leq
        |\mathcal P_1|+\cdots+|\mathcal P_{x-1}|.
    \end{equation}

    Since $\bE_k$ is an induced matching, $\mathcal P_i$ contains no
    node in $C(\bE_k)$ other than its selected tail $T_i$. Therefore, $|\mathcal P_i|
        \leq n-2|\bE_k|+1$, and hence
    \begin{equation}
        \label{eq:bound 2.2 pf2}
        |\mathcal P_1|+\cdots+|\mathcal P_{x-1}|
        \leq
        (x-1)(n-2|\bE_k|+1).
    \end{equation}
    Moreover, $x\leq|\bE_k|
        \quad\text{and}\quad
        |\mathcal P_x|
        \leq
        \max_{[X,Y]^u\in\bE_k}
        |Pa_{\cG}(\{X,Y\})|$.  Combining Equations \eqref{eq:bound 2.2 pf1} and
    \eqref{eq:bound 2.2 pf2}, we obtain
    \[
        \begin{aligned}
        t_k
        &\leq
        \frac{
        (x-1)(n-2|\bE_k|+1)
        }{
        M+1-|\mathcal P_x|
        }
        &\leq
        \frac{
        (|\bE_k|-1)(n-2|\bE_k|+1)
        }{
        M+1-
        \max_{[X,Y]^u\in\bE_k}
        |Pa_{\cG}(\{X,Y\})|
        }.
        \end{aligned}
    \]
    This contradicts the definition $t_k
        =
        1+
        \left\lfloor
        \frac{
        (|\bE_k|-1)(n-2|\bE_k|+1)
        }{
        M+1-
        \max_{[X,Y]^u\in\bE_k}
        |Pa_{\cG}(\{X,Y\})|
        }
        \right\rfloor$.
    Thus, all edges can be added to the bins using the greedy
    procedure.

    Finally, similarly to the constructions in Propositions
    \ref{prop: step 2.2 min strong EC} and
    \ref{prop: step 2.2 strong EC}, define
    \[
        \cI_k
        =
        \left\{
        \bI_{k,j}
        =
        Pa_{\cG}(\bE_{k,j})\setminus\bN_{k,j},
        \quad
        \bI_{k,j}^{\prime}
        =
        \bigl(
        Pa_{\cG}(\bE_{k,j})\setminus\bN_{k,j}
        \bigr)
        \cup\bT_{k,j}
        \right\}_{j=1}^{t_k}.
    \]
    Since $\bI_{k,j}\subseteq\bI_{k,j}^{\prime}
        =
        \mathcal P(\bE_{k,j})$,
both intervention sets have size at most $M$. Setting $\cI=\bigcup_{k=1}^{K}\cI_k$, 
    we obtain an adjacent separating system bounded by $M$ with \\ at
    most $2\sum_{k=1}^{K}
        \left(
        1+
        \left\lfloor
        \frac{
        (|\bE_k|-1)(n-2|\bE_k|+1)
        }{
        M+1-
        \max_{[X,Y]^u\in\bE_k}
        |Pa_{\cG}(\{X,Y\})|
        }
        \right\rfloor
        \right)$
    elements.
\end{proof}

\begin{remark}
    The proof of Theorem \ref{thm: boun 2.2} is constructive and associated with the strong edge coloring. If we disregard the NP-hardness of the computational problem, $K$ can be set to $\chi_s(\cG^u)$, and $\{\bE_1,\dots,\bE_{\chi_s(\cG^u)}\}$ corresponds to the minimal strong edge coloring, similar to Proposition \ref{prop: step 2.2 min strong EC}. However, if the NP-hardness cannot be ignored, as mentioned in Proposition \ref{prop: step 2.2 strong EC}, there exists a greedy algorithm that returns a collection $\{\bE_1,\dots,\bE_k\}$ which forms a strong edge coloring and $K \leq 2d^2$ where $d$ is the max degree of $\cG^u$.  Furthermore, if we let $M\to\infty$, Theorem \ref{thm: boun 2.2} degenerates to Propositions \ref{prop: step 2.2 min strong EC} and \ref{prop: step 2.2 strong EC}.
\end{remark}

According to Theorem \ref{thm: boun 2.2}, we can construct an adjacent separating system $\cI$ with each element bounded by $M$.  We only need to modify Algorithm \ref{alg: 2.2} to use the bounded non-adjacent separating system. It is straightforward to verify that the modified algorithm obtains $\bB^{AS}$ of $\cG$ by performing experiments on the elements of $\cI$. Similar to the unbounded algorithm, we do not discuss the double adjacent bidirected edges $\bB^{AD}$ here.

So far, through the modified Algorithms \ref{alg: 1.1} to \ref{alg: 2.2}, we have successfully identified all directed edges and the majority of bidirected edges in the DMG $\cG$, except for double adjacent bidirected edges using the bounded experiments. 

\begin{corollary}
\label{coro: Upper bounded DB cI}
    Suppose $M \geq \max\{\max_{\{X,Y\}\in\bV,X\neq Y }|Pa_{\cG}(X\cup Y) |, |\rmT_{l+1}^{\cG}|_n + \zeta^{l+1,\cG}_{max} -1\}$. For a DMG $\cG=(\bV,\bD,\bB^{N}\cup\bB^{AS}\cup\bB^{AD})$, if the NP-hardness can be ignored, the modified Algorithms \ref{alg: 1.1} to \ref{alg: 2.2} together learn the $\cRD(\cG) = (\bV,\bD,\bB^{N}\cup\bB^{AS})$ with at most 
    \begin{equation}
    \begin{aligned}
        \lceil\frac{n}{M}\rceil\lceil \log_{\lceil\frac{n}{M}\rceil}n\rceil &+ \sum_{k=1}^{l+1}\zeta_{max}^{k,\cG} +\zeta_{max}^{l+1,\cG}\max\{0,\lfloor\frac{n-|\rmT_{l+1}^{\cG}|_n-\zeta_{max}^{l+1,\cG}-1}{M-|\rmT_{l+1}^{\cG}|_n-\zeta_{max}^{l+1,\cG}+2}\rfloor\} \\
        &+\sum_{k=1}^{cc(\cG^{uc})}1+\lfloor\frac{(\frac{|\bC_k|(|\bC_k|-1)}{2}-1)(n-|\bC_k|)}{M+1-\max_{X,Y\in\bC_k,X\neq Y}|Pa_{\cG}(\{X,Y\})|}  \rfloor \\
        &+2\sum_{k=1}^{\chi_s(\cG^u)}1+\lfloor\frac{(|\bE_k|-1)(n-2|\bE_k|+1)}{M+1-\max_{[X,Y]^u\in\bE_k}|Pa_{\cG}(\{X,Y\})|}  \rfloor
        \end{aligned}
    \end{equation}
    elements and each experiment is bounded by $M$. Here $\{\bC_1,\dots,\bC_{cc(\cG^{uc})}\}$ is the minimum edge clique covering of $\cG^{uc}$ and $\{\bE_1,\dots,\bE_{\chi_s(\cG^u)}\}$ is the minimum strong edge coloring of $\cG^u$. If the NP-hardness can not be ignored, the modified Algorithms \ref{alg: 1.1} to \ref{alg: 2.2} together learn the $\cRD(\cG) = (\bV,\bD,\bB^{N}\cup\bB^{AS})$ with as most 
    \begin{equation}
        \begin{aligned}
        \lceil\frac{n}{M}\rceil\lceil \log_{\lceil\frac{n}{M}\rceil}n\rceil &+ \sum_{k=1}^{l+1}\zeta_{max}^{k,\cG} +\zeta_{max}^{l+1,\cG}\max\{0,\lfloor\frac{n-|\rmT_{l+1}^{\cG}|_n - \zeta_{max}^{l+1,\cG}-1}{M-|\rmT_{l+1}^{\cG}|_n - \zeta_{max}^{l+1,\cG}+2} \rfloor\} \\
        &+\sum_{k=1}^{\lfloor 4e^2(d+1)^2\ln(n) \rfloor}1+\lfloor\frac{(\frac{|\bC_k|(|\bC_k|-1)}{2}-1)(n-|\bC_k|)}{M+1-\max_{X,Y\in\bC_k,X\neq Y}|Pa_{\cG}(\{X,Y\})|}  \rfloor \\
        &+2\sum_{k=1}^{K}1+\lfloor\frac{(|\bE_k|-1)(n-2|\bE_k|+1)}{M+1-\max_{[X,Y]^u\in\bE_k}|Pa_{\cG}(\{X,Y\})|}  \rfloor
        \end{aligned}
    \end{equation}
    elements with probability $1-\frac{1}{n^2}$, where $n$ is the number of nodes and $d$ is the max degree of $\cG^{u}$. Here  $\{\bC_1,\dots,\bC_{\lfloor 4e^2(d+1)^2\ln(n) \rfloor}\}$ is the an edge clique covering of $\cG^{uc}$ with probability $1-\frac{1}{n^2}$ and $\{\bE_1,\dots,\bE_{K}\}$, $K \leq 2d^2$ is a strong edge coloring of $\cG^u$.
\end{corollary}

\section{Conclusion}
\label{sec:conclusion}

In this paper, we propose an experimental design framework for learning the structure of a DMG that may contain both cycles and bidirected edges. The DMG corresponds to a simple SCM, where cycles represent feedback loops and bidirected edges capture unobserved confounding.
We theoretically established worst-case lower bounds on both the maximum experiment size per round and the total number of experiments required to identify all directed edges and non-adjacent bidirected edges in the DMG. To approach these limits, we developed both unbounded and bounded versions of our experimental design framework. Both versions integrate CI tests and do-see tests to accurately recover all directed edges and the majority of bidirected edges. Moreover, the number and size of experiments required by our algorithms are provably close to the derived lower bounds.

In addition, several directions remain for future work:
\begin{itemize}
    \item \textbf{Soft interventions:} In our current framework, all interventions are hard interventions, which completely sever the relationship between the intervened node and its parents. An important extension is to consider soft interventions, which instead modify the conditional distribution between a node and its parents without entirely removing the dependency. Designing effective experiments under soft interventions introduces new methodological challenges.
    %In our current framework, all interventions are hard interventions, which sever the relationship between the intervened node and its parents. An important extension is to consider soft interventions, where interventions only modify the conditional distribution between the intervened node and its parents. Designing experiments under such a setting presents new challenges.
    \item \textbf{Adaptive experimental design:} Although our proposed framework proceeds in three major stages, it is fundamentally non-adaptive, meaning that all interventions are predetermined in advance. A natural extension is to develop adaptive intervention strategies, where each experiment is designed based on the outcomes of previous ones. How to design such adaptive procedures for DMGs remains an open and promising research direction.
    %Although the current framework proceeds in three major stages, it is fundamentally non-adaptive, meaning that all interventions are predetermined. A natural extension is to consider adaptive intervention strategies, where each subsequent experiment is designed based on the outcomes of previous ones. How to design such adaptive experiments for DMGs is an open question.
    \item \textbf{Instance-specific lower bounds:} The current lower bounds for identifying directed edges and non-adjacent bidirected edges are established in the worst-case setting. A valuable direction for future work is to derive instance-specific or MEC-dependent lower bounds that more precisely characterize the identification complexity of a given DMG instance.
    %The current lower bounds for identifying directed edges and non-adjacent bidirected edges are derived in a worst-case sense. It would be valuable to extend these results to obtain instance-specific or MEC–dependent lower bounds that capture the complexity of individual DMGs.
    \item \textbf{Double-adjacent bidirected edges:} Our current methods are unable to recover double-adjacent bidirected edges, i.e., bidirected edges between nodes that are also connected by directed edges in both directions. Although we argue that these edges may have limited impact on downstream causal estimation tasks, developing principled methods to identify them remains an important open problem.
    
    %Our current methods are unable to identify double-adjacent bidirected edges. Although we argue that the presence or absence of such edges may not significantly impact causal effect estimation, developing methods to identify them remains an important direction for future work.
\end{itemize}

\bibliography{ref}
\end{document}